\documentclass[11pt]{article}
\usepackage[margin=1in]{geometry}
\usepackage{natbib}
\usepackage{ifthen}
\usepackage{amsmath,amsfonts,bm}

\usepackage{hyperref}
\usepackage{url}

\usepackage{nicefrac}       
\usepackage{microtype}
\usepackage{comment}
\usepackage{adjustbox}
\usepackage{libertinust1math}
\usepackage{xspace}
\usepackage{chemfig}
\usepackage{enumitem}
\usepackage{multirow}
\usepackage{booktabs}
\usepackage{tikz}
\usepackage{wrapfig}
\usepackage{pgf}

\usepackage[dvipsnames]{xcolor}

\usepackage{graphicx}
\usepackage{amssymb}
\usepackage{mathtools}
\usepackage{amsthm}
\usepackage{subcaption}
\usepackage[capitalize,noabbrev]{cleveref}
\usepackage{subfiles}

\theoremstyle{plain}
\newtheorem{theorem}{Theorem}
\newtheorem*{theorem*}{Theorem}
\newtheorem{proposition}[theorem]{Proposition}
\newtheorem*{proposition*}{Proposition}

\newtheorem{corollary}[theorem]{Corollary}
\newtheorem*{corollary*}{Corollary}
\theoremstyle{definition}

\theoremstyle{remark}

\crefname{theorem}{theorem}{theorems}
\Crefname{theorem}{Theorem}{Theorems}

\crefname{lemma}{lemma}{lemmas}
\Crefname{lemma}{Lemma}{Lemmas}

\crefname{proposition}{proposition}{propositions}
\Crefname{proposition}{Proposition}{Propositions}

\crefname{corollary}{corollary}{corollaries}
\Crefname{corollary}{Corollary}{Corollaries}

\crefname{definition}{definition}{definitions}
\Crefname{definition}{Definition}{Definitions}

\usetikzlibrary{decorations.pathmorphing, quotes, patterns,arrows.meta,backgrounds,fit,positioning}

\usepackage{pdfpages}

\tikzset{snake it/.style={decorate, decoration=snake}}
\tikzset{%
	graph vertex/.style={
		circle,draw=black, fill=white, inner sep=0pt,minimum size=10pt
	}
}

\newcommand{\MyGNN}{{Share\-GNN}\xspace}
\newcommand{\MyGNNs}{{Share\-GNN}s\xspace}

\newcommand{\Rule}{\mathbf{R}}

\newcommand{\RuleMol}{\Rule_{\operatorname{Mol}}}
\newcommand{\RuleFinal}{\Rule_{\operatorname{Aggr}}}

\newcommand{\weightset}{\mathbf{\Omega}}
\newcommand{\weightmatrix}{\mathbf{W}}
\newcommand{\bias}{\mathbf{b}}
\newcommand{\ruleset}{\mathbf{\mathcal{T}}}
\newcommand{\graphlabeling}{l}
\newcommand{\labelset}{\mathcal{L}}

\newcommand{\naturalZero}{\mathbb{N}_0}
\newcommand{\Neighbor}[1]{\mathcal{N}_{#1}}
\newcommand{\MessagePassingFunction}{\operatorname{m}}
\newcommand{\BiasFunction}{\operatorname{b}}
\newcommand{\MPLabel}{l_1}
\newcommand{\BiasLabel}{l_2}

	\definecolor{tab_blue}{HTML}{1f77b4}
	\definecolor{tab_orange}{HTML}{ff7f0e}
	\definecolor{tab_green}{HTML}{2ca02c}
	\definecolor{tab_red}{HTML}{d62728}
	\definecolor{tab_purple}{HTML}{9467bd}
	\definecolor{tab_brown}{HTML}{8c564b}
	\definecolor{tab_pink}{HTML}{e377c2}
	\definecolor{tab_gray}{HTML}{7f7f7f}
	\definecolor{tab_olive}{HTML}{bcbd22}
	\definecolor{tab_cyan}{HTML}{17becf}
	
	\definecolor{darknavy}{RGB}{12,18,43}
	\definecolor{aqua}{RGB}{0,158,227}
	\definecolor{skyblue}{RGB}{89,189,247}
	\definecolor{fuchsia}{RGB}{232,46,130}
	\definecolor{violet}{RGB}{152,48,130}
	
	\newcommand{\FirstColor}[1]{\textcolor{fuchsia}{$\mathbf{#1}$}}
	\newcommand{\SecondColor}[1]{\textcolor{aqua}{$\mathbf{#1}$}}
	\newcommand{\ThirdColor}[1]{\textcolor{black}{$\mathbf{#1}$}}

	\newif\iflong
	\newif\ifshort
	\longfalse 
	\ifthenelse{\boolean{long}}{
		\shortfalse
	}{
		\shorttrue
	}

\title{Invariant-Based Weight Sharing for Message Passing}
\author{Florian Seiffarth$^{1,2}$\\
$^1$\small University of Bonn, Bonn, Germany\\
$^2$\small Lamarr Institute for Machine Learning and Artificial Intelligence, Bonn, Germany\\
\texttt{seiffarth@cs.uni-bonn.de}}
\date{}

\begin{document}
\maketitle

\begin{abstract}
Message-passing neural networks (MPNNs) are a powerful framework for learning representations of graph-structured domains.
However, weights in MPNNs act on features only, limiting their ability to capture structural patterns.
We introduce a novel \textit{structure-aware} weight sharing principle that explicitly incorporates information inherent to the graph structure.
Weights are indexed directly by user-chosen graph invariants, i.e., functions preserved under node permutations,
enabling systematic reuse across structurally equivalent subgraphs.
We present ShareGNNs, which instantiate this principle within a simple encoder-decoder architecture, resulting in an MPNN with \textit{learnable adjacency} and \textit{transformer-like connectivity}.
We show that their expressivity is at least as strong as the discriminative power of the chosen invariants, providing explicit control over the model complexity.
Experiments on synthetic and real-world data, as well as subgraph counting tasks,
demonstrate consistent improvements over standard MPNNs, competitive expressivity beyond the 1-WL test, and scalability to large datasets.
\end{abstract}

\section{Introduction}\label{sec:introduction}
	    \newcommand{\GraphNode}[2]{
	\scalebox{#2}{
		\begin{tikzpicture}
			\node[graph vertex, fill=#1](1) at (0,0){};
		\end{tikzpicture}
	}
}

\newcommand{\ScaledGraphNode}[2]{
	\begin{tikzpicture}
		\node[circle,draw=black, fill=#1, inner sep=0pt,minimum size=#2] at (0,0) {};
	\end{tikzpicture}
}

\newcommand{\targetNode}{
	\begin{tikzpicture}
		\node[circle,draw=red,fill=aqua, inner sep=0pt, minimum size=5pt] at (0, 0) (1) {};
\end{tikzpicture}}

\newcommand{\hydrogenAtom}{\begin{tikzpicture}\node [circle,draw=black, fill=fuchsia, inner sep=0pt, minimum size=5pt](1) at (0,0){};\end{tikzpicture}}

\newcommand{\carbonAtom}{\begin{tikzpicture}\node [circle,draw=black, fill=aqua, inner sep=0pt, minimum size=5pt](1) at (0,0){};\end{tikzpicture}}

	    \newcommand{\AttentionCoeff}[5]{
    \scalebox{#4}{
    \begin{tikzpicture}
    \node at (0,0) {$\omega_{(\GraphNode{#1}{#5}, \GraphNode{#2}{#5}, #3)}$};
    \end{tikzpicture}
    }
    }

	    \newcommand{\ScaledAttentionCoeff}[5]{
    \scalebox{#4}{
    \begin{tikzpicture}
    \node (O) at (0,0) {$\omega_{(\ScaledGraphNode{#1}{#5}, \ScaledGraphNode{#2}{#5}, #3)}$};
    \end{tikzpicture}
    }
    }
\textit{Structure-aware} weight sharing for graphs is challenging because the data is inherently irregular.
Regular domains such as images or sequences, provide a canonical notion of weight sharing, where a single kernel can be reused across spatial locations.
Extending this principle to graphs is not straightforward: graphs lack a canonical node ordering, and their connectivity can vary greatly across instances.
Existing graph neural networks (GNNs) such as message-passing networks~\citep{DBLP:conf/icml/GilmerSRVD17,DBLP:conf/iclr/KipfW17,Hamilton2017InductiveRL,Velickovic2017GraphAN,DBLP:conf/iclr/XuHLJ19} and graph transformers~\citep{DBLP:conf/nips/YunJKKK19,DBLP:conf/ijcai/ShiHFZWS21}
keep weight sharing simple by applying the same transformation to all node or edge representations, regardless of their structural context.

In this work, we extend the classical feature-based approach with a \textit{structure-aware invariant-based weight sharing} principle for graphs, leading to several advantages discussed throughout the paper.
Our principle is based on graph invariants, i.e., functions that remain unchanged under node permutations, and we therefore refer to it as \textit{invariant-based weight sharing}.
By indexing, i.e., tying, weights directly with invariant properties (e.g., atomic numbers or node degrees), we obtain a dynamic weight-sharing scheme that applies to any node or node pair with the same invariant signature, independent of the graph size, node ordering or the node and edge feature dimensions.
Moreover, the principle naturally extends to any domain representable via graphs, including images (using neighborhoods of pixels) or texts (using the sentence structure).

		We present \MyGNNs, an instantiation of this principle, and show that they extend classical message-passing networks by
(i) using invariant-based weight sharing,
(ii) enabling long-range interactions in a single message-passing layer, and
(iii) offering principled control of expressivity through the choice of invariants.
Conceptually, \MyGNNs can be interpreted as MPNNs with \textit{learnable adjacency matrix} and \textit{transformer-like connectivity}.
Classical message-passing networks~\citep{DBLP:conf/icml/GilmerSRVD17,DBLP:conf/iclr/KipfW17} iteratively update node embeddings by aggregating information from neighbors, while graph transformers~\citep{DBLP:conf/nips/YunJKKK19,DBLP:conf/ijcai/ShiHFZWS21} extend this idea by allowing every pair of nodes to exchange information via learned attention.

		ShareGNNs differ fundamentally from graph transformers in \textit{where} parameters are defined and \textit{how} interactions are modeled.
		Graph transformers learn pair-specific attention coefficients as functions of node features and positional encodings,
		resulting in parameters that are \textit{instance-specific} and not explicitly tied across structurally equivalent node pairs.
		In contrast, ShareGNNs index learnable parameters directly by graph invariants,
		yielding a \textit{structure-aware} weight sharing mechanism that enforces systematic parameter reuse across all node pairs with the same invariant signature, within and across graphs.
		As a result, ShareGNNs replace feature-conditioned attention with invariant-indexed message passing,
		providing explicit control over expressivity, improved interpretability, and transferability of learned parameters across graphs.

\begin{figure}[t]
		\centering
        \resizebox{0.49\linewidth}{!}{
		\begin{subfigure}{0.49\linewidth}\centering
	\begin{tikzpicture}[text=black,auto=left]
	\newcommand{\moleculeHX}{-0.6}
	\newcommand{\moleculeCX}{1.75}
	\newcommand{\moleculeHY}{0.985}
	\node [graph vertex](1) at (0,0){};
	\node [graph vertex, draw=red, thick](2) at (\moleculeCX,0){};
	\node [graph vertex](3) at (\moleculeHX,\moleculeHY){};
	\node [graph vertex](4) at (\moleculeHX,-\moleculeHY){};
	\node [graph vertex](5) at (\moleculeCX-\moleculeHX,\moleculeHY){};
	\node [graph vertex](6) at (\moleculeCX-\moleculeHX,-\moleculeHY){};

    \foreach \i in {3,4,5,6}
        \fill[fuchsia] (\i) circle (0.1);
    \foreach \i in {1,2}
        \fill[aqua] (\i) circle (0.1);

	\draw[thick] (1) -- (2);
	\draw[thick] (1) -- (3);
	\draw[thick] (1) -- (4);
	\draw[thick] (5) -- (2);
	\draw[thick] (6) -- (2);
	\draw (2) edge[in=-40,out=40,loop ,gray, thick, -latex, dashed] node[sloped, transform shape, above=-3pt]{\textcolor{black}{$\ScaledAttentionCoeff{aqua}{aqua}{0}{0.8}{6}$}} (2);
	\draw[gray, thick, -latex,decorate,dashed](1) to[bend left] node[pos=0.3,sloped, transform shape,  above=-3pt]{\textcolor{black}{$\ScaledAttentionCoeff{aqua}{aqua}{1}{0.8}{6}$}}  (2);
	\draw[gray, thick, -latex,decorate,dashed](5)  to[bend right] node[pos=0.3, transform shape,sloped, above=-3pt]{\textcolor{black}{$\ScaledAttentionCoeff{fuchsia}{aqua}{1}{0.8}{6}$}}   (2);
	\draw[gray, thick, -latex,decorate,dashed](6)  to[bend left] node[pos=0.3, transform shape,sloped, swap, below=-3pt]{\textcolor{black}{$\ScaledAttentionCoeff{fuchsia}{aqua}{1}{0.8}{6}$}}   (2);
	\draw[gray,thick, -latex,decorate,dashed](3) to[bend left] node[sloped, transform shape, above=-3pt]{\textcolor{black}{$\ScaledAttentionCoeff{fuchsia}{aqua}{2}{0.8}{6}$}} (2);
	\draw[gray,thick, -latex,decorate,dashed](4) to[bend right] node[sloped, swap, transform shape, below=-3pt]{\textcolor{black}{$\ScaledAttentionCoeff{fuchsia}{aqua}{2}{0.8}{6}$}} (2);
\end{tikzpicture}
	\end{subfigure}}
                \resizebox{0.49\linewidth}{!}{
				\begin{subfigure}{0.49\linewidth}\centering
			\begin{tikzpicture}[text=black,auto=left]
	\newcommand{\moleculeHX}{-0.6}
	\newcommand{\moleculeCX}{1.75}
	\newcommand{\moleculeHY}{0.985}
	\node [graph vertex](1) at (0,0){};
	\node [graph vertex, draw=red, thick](2) at (\moleculeCX,0){};
	\node [graph vertex](3) at (\moleculeHX,\moleculeHY){};
	\node [graph vertex](4) at (\moleculeHX,-\moleculeHY){};
	\node [graph vertex](5) at (\moleculeCX-\moleculeHX,\moleculeHY){};
	\node [graph vertex](6) at (\moleculeCX-\moleculeHX,-\moleculeHY){};

    \foreach \i in {3,4,5,6}
        \fill[fuchsia] (\i) circle (0.1);
    \foreach \i in {1,2}
        \fill[aqua] (\i) circle (0.1);

	\draw[thick] (1) -- (2);
	\draw[thick] (1) -- (3);
	\draw[thick] (1) -- (4);
	\draw[thick] (5) -- (2);
	\draw[thick] (6) -- (2);
	\draw (2) edge[in=-40,out=40,loop,gray, thick, -latex, dashed] node[sloped]{\textcolor{black}{\small$a_{11}$}}(2);
	\draw[gray, thick, -latex,decorate,dashed](1) to[bend left] node[above]{\textcolor{black}{\small$a_{13}$}} (2);
	\draw[gray, thick, -latex,decorate,dashed](5) to[bend right] node[midway,above left]{\textcolor{black}{\small$a_{12}$}} (2);
	\draw[gray, thick, -latex,decorate,dashed](6) to[bend left] node[midway, below left]{\textcolor{black}{\small$a_{14}$}} (2);

	\phantom{
		\draw (2) edge[in=-40,out=40,loop ,gray, thick, -latex, dashed] node[sloped, transform shape, above=-3pt]{\textcolor{black}{$\ScaledAttentionCoeff{aqua}{aqua}{0}{0.8}{6}$}} (2);
		\draw[gray, thick, -latex,decorate,dashed](1) to[bend left] node[pos=0.3,sloped, transform shape,  above=-3pt]{\textcolor{black}{$\ScaledAttentionCoeff{aqua}{aqua}{1}{0.8}{6}$}}  (2);
		\draw[gray, thick, -latex,decorate,dashed](5)  to[bend right] node[pos=0.3, transform shape,sloped, above=-3pt]{\textcolor{black}{$\ScaledAttentionCoeff{fuchsia}{aqua}{1}{0.8}{6}$}}   (2);
		\draw[gray, thick, -latex,decorate,dashed](6)  to[bend left] node[pos=0.3, transform shape,sloped, swap, below=-3pt]{\textcolor{black}{$\ScaledAttentionCoeff{fuchsia}{aqua}{1}{0.8}{6}$}}   (2);
		\draw[gray,thick, -latex,decorate,dashed](3) to[bend left] node[sloped, transform shape, above=-3pt]{\textcolor{black}{$\ScaledAttentionCoeff{fuchsia}{aqua}{2}{0.8}{6}$}} (2);
		\draw[gray,thick, -latex,decorate,dashed](4) to[bend right] node[sloped, swap, transform shape, below=-3pt]{\textcolor{black}{$\ScaledAttentionCoeff{fuchsia}{aqua}{2}{0.8}{6}$}} (2);
	}
	\end{tikzpicture}
\end{subfigure}}
	\caption{\label{fig:shared-attention}Invariant-based message passing (left) and graph attention (right) for the molecular graph of ethylene (\textcolor{fuchsia}{hydrogen}, \textcolor{aqua}{carbon}).
		The arrows denote the shared weights (left) and attention coefficients (right)
		for some target marked with red border.}
    \end{figure}
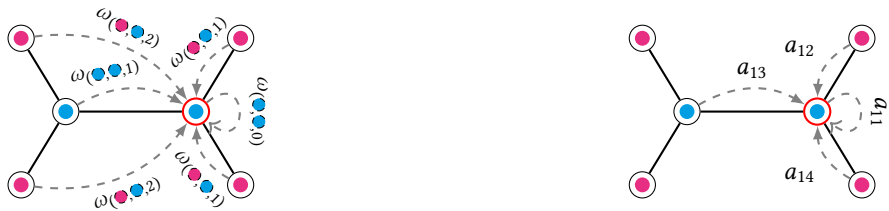

The following toy example on molecular graphs illustrates our idea.
Two node pairs are considered equivalent if:
(i) their atomic numbers and
(ii) their distances match.
In \Cref{fig:shared-attention} (left) there are four node pairs $(\hydrogenAtom, \carbonAtom)$
that are two hops apart.
All such pairs share a single weight, indexed by the tuple $(\hydrogenAtom, \carbonAtom, 2)$.
In this way, every node pair in any molecule can be assigned a weight based on its atomic numbers and pairwise distance.

In \MyGNNs, this concept is generalized: messages are computed for every source-target pair of nodes.
For a given pair, the source node’s embedding is multiplied by the weight associated with that pair’s invariant signature, and the result is added to the target node’s embedding.
Since these weights are shared using graph invariants, the same parameter is reused for all equivalent pairs, both within a graph and across different graphs.
Thus, information can flow between arbitrary nodes in a single layer, rather than being restricted to immediate neighbors
		This makes the learned weights naturally transferable to new graphs that contain similar structural patterns.
Again, we emphasize that graph transformers and ordinary message passing networks infer weights from learned node features rather than from structural invariants.
\Cref{fig:shared-attention} (left) illustrates this process: invariant-based shared weights allow non-adjacent node pairs to send messages based on structural similarity.
In contrast, GAT~\citep{Velickovic2017GraphAN} (right) assigns distinct attention coefficients, based on node features, to each edge and aggregates information only from direct neighbors.

We conclude with the main contributions of this work.
\begin{enumerate}[label=\arabic*)]
	\item \textit{Invariant-Based Weight Sharing:}
	A novel weight sharing principle using graph invariants, enabling systematic reuse of weights across graphs.

	\item \textit{\MyGNN:}
	A novel encoder–decoder architecture using invariant-based weight sharing that combines message passing with transformer-like connectivity.
	\item \textit{Theoretical and Empirical evaluation:}
    We prove that \MyGNNs are at least as expressive as the chosen invariants, giving a straightforward
	way to control model complexity and show that
	invariant-based weight sharing enhances model performance.
\end{enumerate}

\section{Main Approach}\label{sec:RuleGNN}
At the heart of our approach lies a \textit{new principle for learning weights in neural networks}.
Instead of relying on fixed-size weight matrices, we treat all learnable parameters as elements of an unordered collection.
For a given input graph, the relevant weights are \textit{dynamically} selected and assembled into matrices or bias terms according to the chosen graph invariants.
This on‑the‑fly construction of weight matrices allows the model to adaptively reuse parameters across structurally equivalent parts of a graph and even across graphs,
while maintaining full flexibility.
Building on this idea, we instantiate two types of invariant-based weight sharing:

\begin{enumerate}[label=(\alph*)]

\item \label{item:encoder-message-passing}
\textit{Encoder message passing}: weights are indexed by invariant signatures of node pairs (e.g., labels and distance), enabling information flow over arbitrary pairs.

\item \label{item:decoder-pooling-bias}
\textit{Decoder pooling and bias}: weights depend on invariants of single nodes, allowing invariant-based pooling and bias terms.

\end{enumerate}

In both cases, the model constructs the relevant weight structures dynamically from a shared parameter pool,
guaranteeing permutation equivariance in the encoder and permutation invariance in the decoder.
Note that these properties are not automatically guaranteed as in standard MPNNs, since
\MyGNNs learn weights of the adjacency structure itself rather than relying on a fixed adjacency matrix.

\subsection{Preliminaries}\label{subsec:MyGNNpreliminaries}
We denote by $\naturalZero$ the set of natural numbers including zero.
A \textit{graph} is a pair $G=(V, E)$ where $V$ is the set of nodes and $E\subseteq\{\{i,j\}\mid i,j\in V\}$ is the set of edges.
Each graph $G=(V, E)$ is associated with an initial node feature matrix $X^{(0)}\in\mathbb{R}^{|V|\times k}$, where $k\in\mathbb{N}$ is the node feature dimension.
This matrix also fixes an ordering of nodes in $V$.
For a node $i\in V$, its \textit{neighborhood} is denoted by $\Neighbor{i}\coloneqq\{j\in V : \{i,j\} \in E\}$, and its \textit{degree}  is $\operatorname{deg}(i)\coloneqq|\Neighbor{i}|$.
The \textit{distance} between two nodes $i,j\in V$, i.e., the length of a shortest path between them, is written as $d(i,j)$.
A \textit{labeling function} $l:V\rightarrow \naturalZero$ assigns an integer label to each node.

\subsection{Invariant-Based Weight Sharing}\label{subsec:definingrules}
We now formalize the two types (a) and (b) of invariant-based weight sharing from above.
In both cases, graph invariants provide indices into a shared parameter pool from which the corresponding weights are constructed.

\paragraph{Weight Sharing for Message Passing}
For type (a), we associate a learnable weight with each valid invariant triple that characterizes a pair of nodes.
A triple consists of: (i) the label of the first node, (ii) the label of the second node, and (iii) the shortest-path distance between them.
Thus, given a labeling function $l:V\rightarrow\naturalZero$, the triple for a node pair $(v, w)\in V\times V$ is defined as $\tau_{vw}\coloneqq(l(v), l(w), d(v,w))\in \naturalZero^3$.
For molecular graphs, $l$ can assign to each node its atomic number, while for general graphs $l$ may, for example, assign the node's degree or any other structural identifier.
Moreover, the labeling function for the source and target nodes could in principle be chosen differently, but we omit this distinction in our experiments for simplicity.
To control the parameter budget, we restrict ourselves to a finite set of valid triples $\ruleset\subset\naturalZero^3$.
For each valid triple $\tau\in\ruleset$, we assign a learnable scalar weight $\omega_\tau\in\mathbb{R}$.
Node pairs with triples not in $\ruleset$ get weight zero, ensuring sparsity and scalability.
For example, if $\mathcal{T}$ includes triples with distances up to $D$,
then any pair of nodes whose shortest-path distance exceeds $D$ will be assigned zero weight and will not exchange messages.
Let $\weightset_{\ruleset}$ be the set of weights associated with all valid triples.
This defines the invariant mapping
\begin{align}\label{eq:pairs-to-attention-coefficients}
	\MessagePassingFunction^l_{\ruleset}: V\times V\rightarrow\ruleset\cup\{0\}\rightarrow\weightset_{\ruleset}\cup\{0\}
\end{align}
that assigns to each pair of nodes $(v,w)$ a weight $\omega_{\tau_{vw}}$ if $\tau_{vw}\in\ruleset$ and zero otherwise.

While we focus on label-label-distance triples for simplicity, this association is only one example: other invariants (such as edge labels, motif counts, or higher-order structural descriptors) can be incorporated in exactly the same way, see~\Cref{sec:example-molecule-graphs} for an example with edge labels.

\paragraph{Weight Sharing for Bias and Pooling}
For type (b), we assign shared weights to individual nodes rather than node pairs.
Given a labeling function $l:V\rightarrow\naturalZero$ (not necessarily the same as before) and a set of valid labels $\labelset\subset\naturalZero$,
each label $\lambda\in\labelset$ is associated with a learnable weight vector $\omega_\lambda\in\mathbb{R}^k$, where $k$ is the node feature dimension.
The set of all such vectors is denoted by $\weightset_{\labelset,k}$.
This yields
\begin{align}\label{eq:nodes-to-label-weights}
	\BiasFunction^l_{\labelset,k}: V\rightarrow\labelset\cup\{0\}\rightarrow\weightset_{\labelset,k}\cup\{0\}
\end{align}
which assigns each node $v\in V$ a learnable weight vector $\omega_{l(v)}\in\mathbb{R}^k$ if $l(v)\in\labelset$ and zero otherwise.
During encoding, these vectors act as bias terms, while during decoding they are used as pooling weights to aggregate node embeddings.
As with type (a), the choice of labeling function is flexible: for molecules, $l$ may correspond to atomic numbers, while for arbitrary graphs it may use node degree, Weisfeiler–Leman labels, or other structural identifiers.
We will now describe several options for the labeling functions.

\subsection{Labeling Functions}\label{subsec:labeling-nodes}
The choice of labeling function $l$ determines how nodes (or node pairs) are grouped for parameter sharing.
Depending on the application, these labels encode properties, such as atomic numbers in molecular graphs, or richer structural information, such as degrees, Weisfeiler–Leman (WL) labels, or the existence of small patterns.
In our experiments, we consider the following labeling strategies:

\emph{Molecular Labeling.}
For molecular graphs, $l$ assigns to each node its atomic number.

\emph{Weisfeiler-Leman (WL) Labeling.}
The WL algorithm iteratively refines node labels by hashing together a node’s current label with the multiset of labels of its neighbors.
After a fixed number of iterations, this procedure yields labels that capture information about the $i$-hop neighborhood around each node
~\citep{DBLP:journals/jmlr/ShervashidzeSLMB11}.
In our experiments, we use $1$-WL labels with varying depths.

\emph{Pattern Labeling.}
In addition to WL labels, we consider labeling nodes by their participation in small subgraph patterns.
For a given set of patterns (e.g., triangles, cliques, cycles), we count for each node how many embeddings of each pattern include that node.
This count vector is then hashed into a unique integer label, providing additional structural distinctions beyond WL.

These strategies can also be combined, e.g., by concatenating the atomic number, WL label, and pattern labels into a single label.
This concatenation allows the model to capture multiple levels of structural information simultaneously.

\subsection{Invariant-Based Layers}\label{subsec:learningrules}
We now describe how invariant-based weight sharing is integrated into the architecture of  \MyGNNs.
Each encoder and decoder layer dynamically assembles its weight matrices
from a shared parameter pool according to the chosen invariants.
Given an input graph, the corresponding invariant mappings are applied to its nodes (or node pairs) to get the appropriate weights.
\paragraph{Invariant-Based Encoder}
An invariant-based encoder layer consists of two components:  ($1$) a message-passing matrix whose entries are determined by the type (a) invariant mapping $\MessagePassingFunction^{\MPLabel}_{\ruleset}(v,w)$, and
($2$) a bias matrix whose rows are determined by the type (b) mapping  $\BiasFunction^{\BiasLabel}_{\labelset, k}(v)$.
Unlike fixed weight matrices, these matrices are constructed dynamically for each input graph from the layer’s parameter pool.
Concretely, let $G=(V,E)$, with $|V|=n$ be the input graph and $X^{(h)}\in\mathbb{R}^{n \times k}$ the node embeddings after $h$ layers.
Assume that the nodes $v,w\in V$ correspond to the positions $i,j$ in $X^{(h)}$, i.e., they are nodes $i,j$ of the fixed
node ordering of $G$.
Then, for each node pair $(v,w)$, the entry $W_{ij}$ of the message-passing matrix corresponds to the weight $\MessagePassingFunction^{\MPLabel}_{\ruleset}(v,w)$,
where $\MPLabel$ is the labeling function and $\ruleset$ the valid set of triples.
Similarly, the $i$-th row of the bias matrix is given by $\BiasFunction^{\BiasLabel}_{\labelset, k}(v)$.
The encoder updates node embeddings as
\begin{align}\label{eq:rulebasedlayer}
    \hspace{-0.25em}\underbrace{\vphantom{\weightmatrix\left[G, \MessagePassingFunction^{\MPLabel}_{\ruleset},\weightset_{\ruleset}\right]}X^{(h+1)}}_{n\times k}
		\hspace{-0.1em}=\hspace{-0.1em} \sigma\left(\vphantom{\weightmatrix\left[G, \MessagePassingFunction^{\MPLabel}_{\ruleset},\weightset_{\ruleset}\right]}\right.
		\underbrace{\weightmatrix\left[G, \MessagePassingFunction^{\MPLabel}_{\ruleset},\weightset_{\ruleset}\right]}_{\mathclap{n\times n}} \cdot
		\underbrace{\vphantom{\weightmatrix\left[G, \MessagePassingFunction^{\MPLabel}_{\ruleset},\weightset_{\ruleset}\right]}X^{(h)}}_{\mathclap{n\times k}}\hspace{-0.05em} +\hspace{-0.05em}
		\underbrace{\bias\left[G,\BiasFunction^{\BiasLabel}_{\labelset, k}, \weightset_{\mathcal{L},k}\right]}_{\mathclap{n\times k}}
	\left.\vphantom{\weightmatrix\left[G, \MessagePassingFunction^{\MPLabel}_{\ruleset},\weightset_{\ruleset}\right]}\right)
\end{align}
where $\sigma$ is a non-linear activation function.
This formulation defines a form of message passing that:
is not limited to adjacent nodes, since any valid invariant triple can link two nodes,
allows directional effects, since $\MessagePassingFunction^{\MPLabel}_{\ruleset}(v,w)$ can differ from $\MessagePassingFunction^{\MPLabel}_{\ruleset}(w,v)$, and
reuses the same weights for all pairs that share the same triple.
As a result, a single encoder layer can propagate information over arbitrary distances while preserving permutation equivariance.
\Cref{fig:shared-attention-coefficients} (left) illustrates this mechanism for node $v_i$.

\paragraph{Invariant-Based Decoder}
The decoder aggregates the final node embeddings into a fixed-size graph representation using the type (b) mapping.
Instead of mean or max pooling, it performs a weighted pooling where the contribution of each node is determined by its invariant label.
Given the representation $X^{(h)}\in\mathbb{R}^{n \times k}$, we construct a pooling matrix $\weightmatrix$ whose $i$-th column corresponds to the weight vector $\BiasFunction^{l}_{\labelset, m}(v_i)\in\mathbb{R}^m$
retrieved from the decoder's parameter pool $\weightset_{\labelset,m}$.
The aggregated representation is then computed as
\begin{align}\label{eq:agregationlayer}
    \underbrace{\vphantom{\weightmatrix\left[G, \BiasFunction^{\MPLabel}_{\ruleset}, \weightset_{\ruleset}\right]}\mathbf{X}^{(h+1)}}_{m\times k} = \sigma
	\left(\vphantom{\weightmatrix\left[G,\BiasFunction^{l}_{\labelset, m}, \weightset_{\labelset}\right]}\right.
	\frac{1}{n}\underbrace{\weightmatrix\left[G,\BiasFunction^{l}_{\labelset, m}, \weightset_{\labelset}\right]}_{\mathclap{m\times n}} \cdot
	\underbrace{\vphantom{\weightmatrix\left[G,\BiasFunction^{l}_{\labelset, m}, \weightset_{\labelset}\right]}
		\mathbf{X}^{(h)}}_{\mathclap{n\times k}} +
	\underbrace{\vphantom{\weightmatrix\left[G,\BiasFunction^{l}_{\labelset, m}, \weightset_{\labelset}\right]}\bias}_{\mathclap{m\times k}}
	\left.\vphantom{\weightmatrix\left[G,\BiasFunction^{l}_{\labelset, m}, \weightset_{\labelset}\right]}\right)
\end{align}
where $m$ is the desired output dimension and $\bias$ is a standard bias term of fixed size $m\times k$.
The $1/n$ normalization mitigates the effect of graph size.
This pooling strategy preserves permutation invariance and allows the decoder to emphasize structurally important nodes, as determined by the chosen invariant labels, see~\Cref{fig:shared-attention-coefficients} (right).

\paragraph{Multi-Heads}
While encoder and decoder operate with a single set of invariants, \MyGNNs naturally extend to a multi-head architecture.
Each encoder or decoder head is instantiated in parallel, with its own parameter pool and potentially different labeling functions.
The heads are concatenated and combined by a feed-forward layer before proceeding to the next stage.
This multi-head design allows the model to capture complementary structural patterns within the same layer, see~\Cref{fig:substructure-counting-heads}.
\paragraph{Backpropagation} \MyGNNs use standard backpropagation.
Once the dynamic weight matrices are constructed, gradient computation proceeds automatically through the computational graph.

\begin{figure}[t]\centering
    \begin{tikzpicture}[scale=0.75]\centering
        \foreach \i in {1, ..., 6}
             \node[graph vertex, fill=white] (\i) at ({cos(\i*60)}, {sin(\i*60)}) {};
        \foreach \i in {1,...,6}{
            \pgfmathtruncatemacro{\j}{mod(\i,6) + 1}
            \draw (\i) -- (\j);
        }
        \node at (0,-1.5) {$G_1$};
        \foreach \n in {1,...,6}
            {
                \fill[aqua] (\n) circle (0.1);
            }

        	\draw[gray, thick, -latex,decorate,dashed](1) to[bend left] node[sloped, transform shape,  above=-3pt]{\textcolor{black}{$\ScaledAttentionCoeff{aqua}{aqua}{3}{1.2}{6}$}}  (4);

        \end{tikzpicture}
    \hspace{3cm}
    \begin{tikzpicture}[scale=0.75]
        \foreach \i in {1, ..., 6}
             \node[graph vertex, fill=white] (C1\i) at ({cos(\i*60+30)}, {sin(\i*60+30)}) {};
        \foreach \i in {1,...,6}{
            \pgfmathtruncatemacro{\j}{mod(\i,6) + 1}
            \draw (C1\i) -- (C1\j);
        }
        \foreach \i in {1, ..., 6}
             \node[graph vertex, fill=white] (C2\i) at ({2*cos(30) + cos(\i*60+30)}, {sin(\i*60+30)}) {};
        \foreach \i in {1,...,6}{
            \pgfmathtruncatemacro{\j}{mod(\i,6) + 1}
            \draw (C2\i) -- (C2\j);
        }
        \foreach \n in {5,6}
            {
                \fill[fuchsia] (C1\n) circle (0.1);
            }
        \foreach \n in {1,4}
            {
                \fill[violet] (C1\n) circle (0.1);
                \fill[violet] (C2\n) circle (0.1);
            }
        \foreach \n in {2,3}
            {
                \fill[tab_orange] (C1\n) circle (0.1);
            }
        \foreach \n in {5,6}
            {
                \fill[tab_orange] (C2\n) circle (0.1);
            }
        	\draw[gray, thick, -latex,decorate,dashed](C12) to[bend left=20] node[ sloped, transform shape,  above=-3pt]{\textcolor{black}{$\ScaledAttentionCoeff{tab_orange}{tab_orange}{4}{1.2}{6}$}}  (C26);
        	\draw[gray, thick, -latex,decorate,dashed](C13) to[bend left=-10] node[pos=0.2, sloped, transform shape,  above=-3pt]{\textcolor{black}{$\ScaledAttentionCoeff{tab_orange}{tab_orange}{5}{1.2}{6}$}}  (C26);

        \node at (1,-1.5) {$G_3$};
        \node at (3.4,0) {};
        \end{tikzpicture}
\\
        \begin{tikzpicture}[scale=0.75]
        \foreach \i in {1, ..., 6}
             \node[graph vertex, fill=white] (\i) at ({cos(\i*60)}, {sin(\i*60)}) {};
        \draw (2) -- (3);
        \draw (3) -- (4);
        \draw (4) -- (2);

        \draw (5) -- (6);
        \draw (6) -- (1);
        \draw (1) -- (5);

        \node at (0,-1.5) {$G_2$};
        \foreach \n in {1,...,6}
            {
                \fill[aqua] (\n) circle (0.1);
            }
        \end{tikzpicture}
\hspace{3cm}
        \begin{tikzpicture}[scale=0.75]
            \foreach \i in {1, ..., 5}
                 \node[graph vertex, fill=white] (C1\i) at ({cos(\i*72)}, {sin(\i*72)}) {};
            \foreach \i in {1,...,5}{
                \pgfmathtruncatemacro{\j}{mod(\i,5) + 1}
                \draw (C1\i) -- (C1\j);
            }
            \foreach \i in {1, ..., 5}
                 \node[graph vertex, fill=white] (C2\i) at ({3 + cos(\i*72+36)}, {sin(\i*72+36)}) {};
            \foreach \i in {1,...,5}{
                \pgfmathtruncatemacro{\j}{mod(\i,5) + 1}
                \draw (C2\i) -- (C2\j);
            }
            \draw (C15) -- (C22);
            \node at (1.5,-1.5) {$G_4$};
            \foreach \n in {2,3}
                {
                    \fill[tab_orange] (C1\n) circle (0.1);
                }
            \foreach \n in {4,5}
                {
                    \fill[tab_orange] (C2\n) circle (0.1);
                }
            \foreach \n in {1,4}
                {
                    \fill[violet] (C1\n) circle (0.1);
                }
            \foreach \n in {1,3}
                {
                    \fill[violet] (C2\n) circle (0.1);
                }
            \foreach \n in {5}
                {
                    \fill[fuchsia] (C1\n) circle (0.1);
                }
            \foreach \n in {2}
                {
                    \fill[fuchsia] (C2\n) circle (0.1);
                }
                    	\draw[gray, thick, -latex,decorate,dashed](C12) to[bend left=10] node[sloped, transform shape,  above=-3pt]{\textcolor{black}{$\ScaledAttentionCoeff{tab_orange}{tab_orange}{5}{1.2}{6}$}}  (C25);
                    	\draw[gray, thick, -latex,decorate,dashed](C13) to[bend right=10] node[pos=0.45, sloped, transform shape,  below=-3pt]{\textcolor{black}{$\ScaledAttentionCoeff{tab_orange}{tab_orange}{5}{1.2}{6}$}}  (C25);
        \end{tikzpicture}
    \caption{\label{fig:expressivity-MyGNN} Non-isomorphic graph pairs $G_1, G_2$ and $G_3, G_4$ that are indistinguishable by the 1-WL test, and therefore by standard message-passing GNNs, but distinguishable by~\MyGNN. Node colors indicate $1$-WL labels.
        Arrows mark messages in~\MyGNN that enable distinguishing the graphs.
    }
\end{figure}
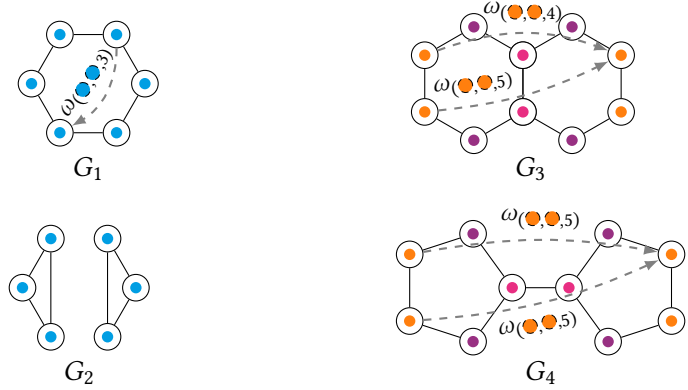

\subsection{\MyGNNs}\label{subsec:MyGNNs}
A \MyGNN is built by stacking one or more invari\-ant-based encoder layers followed by a decoder that aggregates the final node embeddings into a graph-level representation.
This design offers several advantages:
\textit{(i) Adaptivity:}
The model is not constrained to fixed-size weight matrices; it dynamically adapts to graphs of varying sizes.
\textit{(ii) Expressiveness:}
Its ability to capture structural patterns is determined by the choice of invariants and is not limited by standard neighborhood-based schemes.
\textit{(iii) Long-range interactions:}
It propagates information across arbitrary node pairs within a single layer, avoiding over-smoothing~\cite{DBLP:journals/corr/abs-2006-13318} and over-squashing~\cite{DBLP:conf/iclr/0002Y21}.
\textit{(iv) Permutation awareness:}
The encoder is permutation equivariant, and the decoder is permutation invariant by construction.
\textit{(v) Transferability and interpretability:}
Since parameters are indexed by structural invariants, the learned weights can be transferred across different datasets, and each parameter has a clear structural meaning.
We now turn to a formal analysis of the properties of \MyGNNs (see the proofs in~\Cref{sec:proofs}).
\begin{proposition}\label{prop:permutation-equivariance}
    The encoder layer as defined in~\Cref{eq:rulebasedlayer} is permutation equivariant, i.e.,
    for each permutation $\pi$ of the node order of the input graph it holds
    $\pi(\mathbf{X}^{(h+1)}) = \sigma(\weightmatrix\cdot \pi(\mathbf{X}^{(h)}) + \bias)$.
\end{proposition}
\begin{proposition}\label{prop:permutation-invariance}
	The decoder as defined in~\Cref{eq:agregationlayer} is permutation invariant, i.e.,
    for each permutation $\pi$ of the node order of the input graph it holds
    $\mathbf{X}^{(h+1)} = \sigma(\frac{1}{n}\weightmatrix\cdot \pi(\mathbf{X}^{(h)}) + \bias)$.
\end{proposition}
\begin{corollary}\label{cor:permutation-invariance}
	\MyGNNs are permutation invariant.
\end{corollary}
\begin{proposition}\label{prop:expressivity-MyGNN}
	Each~\MyGNN is at least as expressive as the labeling function $l$ of the decoder, i.e.,
	if two graphs $G$ and $G'$ are distinguishable by $l$ then they are also distinguishable by the~\MyGNN.
\end{proposition}
\begin{corollary}\label{cor:expressivity-MyGNN}
    For non-isomorphic graphs $G$ and $G'$ it exists a~\MyGNN distinguishing $G$ and $G'$.
\end{corollary}
Beyond labels, distance-based message passing further strengthens expressivity.
\Cref{fig:expressivity-MyGNN} shows graphs that are indistinguishable by the $1$-WL test, and therefore by standard MPNNs but can be distinguished by~\MyGNN using $1$-WL node labels.

\paragraph{Time and Space Complexity}
The runtime of~\MyGNNs consists of two parts: (i) a preprocessing step, where node labels and pairwise distances are computed, and (ii) the actual training and inference.
The cost of computing node labels depends on the chosen invariants but is in our case negligible in practice compared to model training (see \Cref{tab:preprocessing-times}).
Pairwise distance computations require $O(n(n+m))$ time per graph with $n$ nodes and $m$ edges but need only computed once for the entire dataset.
Let $D$ denote the maximum propagation distance and $N$ the number of distinct node labels in the dataset.
This yields at most $N^2\cdot D$ parameters for the message-passing weights.
In practice (cf.~\Cref{tab:details-best-runs}) the number is significantly smaller: not all label pairs occur at every distance, and weights unused in the dataset are discarded.
If the label space is very large, $N$ can also be restricted explicitly.
Assigning weights to node pairs requires $O(n^2)$ time and $O(n^2)$ space to store the mapping from node pairs to their corresponding weights.
Overall, the asymptotic complexity is comparable to graph transformers, but with the benefit of
flexibility in reducing the number of parameters through the choice of invariants.

\subsection{\MyGNN is a MPNN with Learnable Adjacency}
The key difference of the message passing in \MyGNNs and ordinary MPNNs lies in the dynamic construction of weight matrices based on invariant mappings.
Thus, message passing in  \MyGNNs can be interpreted as message passing in GCNs with learnable adjacency matrices.
Note that the learned adjacency is beyond the fixed graph structure, i.e., beyond the edges defined in the input graph.
The following comparison of the message passing in GCN~\ref{eq:MPNN} and \MyGNN~\ref{eq:ShareGNN}
illustrates this:
\begin{align}
			\label{eq:MPNN}\mathbf{X}^{(h+1)} &=& \sigma(&&\underbrace{\mathbf{D}^{-\frac{1}{2}}\mathbf{\tilde{A}}\mathbf{D}^{-\frac{1}{2}}}_{n\times n}&\cdot&\underbrace{\mathbf{X}^{(h)}}_{n\times k}&\cdot&\underbrace{\weightmatrix^{(h)}}_{k\times k}&&)\\
			\label{eq:ShareGNN}\mathbf{X}^{(h+1)} &=& \sigma(&&\underbrace{\Theta^{(h)}}_{n\times n}&\cdot&\underbrace{\mathbf{X}^{(h)}}_{n\times k}&\cdot&\underbrace{\weightmatrix^{(h)}}_{k\times k}&&)
\end{align}
Importantly, in contrast to GCNs where the size of $\weightmatrix$ depends only on the feature dimension $k$,
the size of $\Theta$ is \textit{not fixed} and
depends on the input graph size.
Note that this novel approach is only
possible due to the dynamic weight sharing mechanism of \MyGNNs.
In fact, $\Theta$ is \textit{not} learned directly but assembled from invariant-indexed scalar weights.

\section{Related Work}\label{sec:related-work}
We position our approach in the context of prior work on weight sharing and graph neural networks.

\paragraph{Weight sharing}
Weight sharing was first introduced in neural networks by
~\cite{10.5555/104279.104293} and became a cornerstone of convolutional networks
~\citep{DBLP:conf/nips/CunBDHHHJ89}, where shared kernels provide translation invariance and parameter efficiency.
Extending these ideas to irregular structures such as graphs has been challenging because graphs lack a regular grid structure.

\paragraph{MPNNs}
In message passing networks
~\citep{DBLP:conf/iclr/KipfW17,DBLP:conf/icml/GilmerSRVD17,Hamilton2017InductiveRL}
node embeddings are updated by aggregating information from neighbors.
Variants such as GAT~\citep{Velickovic2017GraphAN} and GATv2~\citep{DBLP:conf/iclr/Brody0Y22}
implement attention mechanisms to learn the importance of neighboring nodes.
These models typically learn weights and importance between neighboring nodes based on node features.

\paragraph{Beyond MPNNs}
Several approaches overcome the locality and limitations of standard message passing
by introducing global interactions and richer structural information.
These include graph transformers
~\citep{DBLP:conf/nips/YunJKKK19,DBLP:conf/ijcai/ShiHFZWS21,DBLP:conf/iclr/Brody0Y22,DBLP:conf/kdd/ZhuWS0Z23,DBLP:journals/corr/abs-2402-10793},
which enable all node pairs to interact in a single layer through learned attention, as well as approaches based on subgraph-level reasoning
\citep{DBLP:conf/icml/ZhangFDHW23,DBLP:conf/icml/PunyLKML23,DBLP:conf/nips/PaolinoMWK24,DBLP:conf/kdd/YanZG0Z24},
higher-order representations using simplicial and cellular complexes
~\citep{DBLP:conf/nips/BodnarFOWLMB21,DBLP:conf/icml/BodnarF0OMLB21},
provably powerful architectures~\citep{DBLP:conf/nips/MaronBSL19},
or positional/distance encodings~\citep{DBLP:conf/nips/LiWWL20,DBLP:conf/nips/KreuzerBHLT21,DBLP:conf/nips/RampasekGDLWB22,DBLP:journals/tmlr/FranksECWSFK25}.
However, the works still parameterize their operations in terms of edges, attention scores, or node features, without a principled mechanism for structure-aware weight sharing.
\paragraph{R-GCN}
We would like to particularly highlight the related work on Relational Graph Convolutional Networks (R-GCN) by~\citet{10.1007/978-3-319-93417-4_38} as
weights are shared across edges of the same relation type in multi-relational graphs, effectively tying parameters to a fixed set of relation labels.
    In contrast, \MyGNN ties weights by structural invariants computed from the graph itself,
which allows sharing across structurally equivalent node pairs even when no explicit relation types are given.
    This decouples the sharing mechanism from a predefined relation vocabulary and enables transfer across graphs with
    different structures but similar invariant patterns.
    Moreover, their approach is based on learnable normalizing factors while we directly consider learnable relations (based on graph-invariants).
    Besides an encoder, we also provide a decoder to handle graph-level tasks.
    \citet{10.1007/978-3-319-93417-4_38} consider only node classification and link prediction tasks as their weight matrices are of fixed size in contrast to our dynamic sized weight matrices.
    Notably, the dynamic-sized weight matrices are the main practical challenges (on the implementation side).

\section{Experiments}\label{sec:Experiments}

Our experiments address three main questions:
(1) \textit{Does the invariant-based weight sharing extension improve over ordinary MPNNs?}
(2) \textit{Do expressive graph invariants also lead to expressive models in practice?}
(3) \textit{Is \MyGNN scalable to real-world datasets with many graphs?}
Additionally, we conduct an ablation study to assess the impact of different model modifications on performance.
We provide the experimental details (\Cref{sec:evaluation-details}) including the dataset details (\Cref{subsec:details-on-the-datasets}), the full experimental setup including hyperparameter configurations~(\Cref{sec:evaluation-details}), as well as extended results and more ablation studies~(\Cref{sec:additional-results}) in the
supplementary material.
To reproduce the experiments see \href{https://github.com/fseiffarth/SimpleGNN}{https://github.com/fseiffarth/SimpleGNN}.

\subsection{Invariant-Based Weight Sharing vs. MPNNs}\label{subsec:results}
We show the improvement due to structure-aware weight sharing compared to standard MPNN weight sharing on diverse graph classification tasks, including molecular graphs, social graphs, and synthetic datasets.
We set up \MyGNN according to the following configuration the other baselines are configured as described in~\cite{DBLP:conf/iclr/ErricaPBM20}.

The comparison is performed using \textit{two} different setups to emphasize that the results in GNN evaluation highly depend on the evaluation setup:
(i) \textit{Fair} setup, using same train/validation/test splits as in~\citep{DBLP:conf/iclr/ErricaPBM20} but we report the performance of the single hyperparameter configuration that performs best (on validation) across all folds instead of using the best hyperparameter configuration per validation fold and (ii) \textit{Standard} evaluation~\citep{DBLP:conf/iclr/XuHLJ19} only with train/test splits.

\paragraph{\MyGNN Configuration}
The experiments show that structure-aware weight sharing can replace weights based on node and edge features and even the explicit encoding of node attributes completely.
Thus, the initial graph representation is a vector of ones, and all structural information enters the model solely through the labeling functions, i.e., the weight sharing mechanism.
For robustness, we additionally consider a variant where Gaussian noise with standard deviation $0.5$ is added to the constant inputs.
The decoder output dimension is set to the number of classes, producing a graph-level representation of the correct size without any additional projection.
Here, we use a single encoder layer with one head.
The only exception is RT3 with two encoder layers.
Our hyperparameters are the labeling functions $l_{\operatorname{e}}$ and $l_{\operatorname{d}}$ for
the encoder and decoder and the set of valid triples $\ruleset$, controlling the range of message passing.
For $l_{\operatorname{e}}$ and $l_{\operatorname{d}}$, we perform a grid search over: atomic numbers, $1$‑WL labels with depths from $0$ to $3$, and pattern labels (e.g., cycles and cliques)
which is standard practice for these datasets (cf.~\citep{DBLP:journals/pami/BouritsasFZB23}).
The valid triples $\mathcal{T}$ are restricted by the maximum hop distance $D$, which we set to $6$ for molecular datasets and to the maximum graph diameter for social datasets.
We use \textit{tanh} for activation in all layers, cross-entropy loss
and the Adam optimizer~\citep{DBLP:journals/corr/KingmaB14} with a learning rate of $0.01$
(real-world) and $0.1$ (synthetic).
All biases are initialized with $0$ and all other weights with a constant value of $0.001$ (real-world) or uniformly distributed $[-0.1,0.1]$ (synthetic).
Interestingly, therefore, the only random factor in the real-world experiments is the order of the input samples.
The models are trained for $200$ epochs with a batch size of $64$ and early stopping if the validation accuracy does not improve for $25$ epochs.

\begin{table*}[t]
	\centering
    \begin{adjustbox}{width=\textwidth}
\begin{tabular}{lcccc|cc|ccc|cc}
 Method     & NCI1             & NCI109                     & Muta.                            & DHFR                        & IM$-$B           & IM$-$M                          & RT1                     & RT2                       & RT3          & CSL          & SF\\
\toprule
GraphSAGE & $79.0^\bullet$     & $78.5^\circ$               & $80.0^\circ$                     & $79.8^\bullet$              & $70.2^\bullet$ & $47.8^\bullet$                & $31.5^\star$              & $50.1^\bullet$          & $26.7^\bullet$ & $10.0^\circ$ & $23.3^\bullet$\\
GIN & $80.3^\circ$             & \ThirdColor{79.1^\circ}               & \ThirdColor{81.7^\circ}                     & \ThirdColor{80.0^\bullet}              & $67.3^\star$ & $44.9^\star$              & $31.7^\bullet$              & $51.2^\bullet$              & $26.2^\bullet$ & $8.7^\bullet$ & $24.2^\bullet$\\
GAT & $76.1^\bullet$           & $75.4^\circ$               & $78.7^\circ$                     & $78.4^\bullet$              & $68.0^\star$ & $47.8^\bullet$                & $33.0^\bullet$              & $59.4^\star$            & $36.9^\star$ & $10.0^\circ$ & $25.4^\bullet$\\
GATv2  & $80.4^\circ$          & $78.6^\circ$               & $78.8^\circ$                     & $78.3^\bullet$              & $69.8^\star$ & $48.0^\bullet$                & $33.4^\bullet$              & $63.9^\bullet$          & $35.4^\star$ & $10.0^\circ$ & $24.0^\star$\\
 \hline
 GT & $80.5^\circ$             & $-$ & $-$ & $-$                                                                          & $73.1^\circ$   & $49.0\circ$ & $-$ & $-$ & $-$ & $-$ & $-$\\
GT + R$-$Cov & \ThirdColor{83.1^\circ}      & $-$ & $-$ & $-$                                                                          & \FirstColor{76.1^\circ} & \SecondColor{51.1\bullet} & $-$ & $-$ & $-$ & $-$ & $-$\\
	\midrule
\textbf{ours} & \FirstColor{85.4^\circ} & \SecondColor{85.3^\circ} & \SecondColor{83.2^\circ} & \SecondColor{80.9^\star} & \SecondColor{75.9^\bullet} & \SecondColor{51.1^\star} & \FirstColor{100.0^\circ} & \FirstColor{100.0^\circ} & \FirstColor{99.9^\circ} & \FirstColor{100.0^\circ} & \FirstColor{98.0^\circ}\\
\textbf{ours (Random)} & \SecondColor{85.3^\bullet} & \FirstColor{85.4^\circ} & \FirstColor{83.3^\circ} & \FirstColor{81.5^\bullet} & \ThirdColor{75.6^\bullet} & \FirstColor{51.6^\star} & \FirstColor{100.0^\circ} & \FirstColor{100.0^\circ} & \SecondColor{96.9^\circ} & \FirstColor{100.0^\circ} & \FirstColor{98.0^\bullet}\\
\bottomrule
\end{tabular}
\end{adjustbox}
    \caption{\label{tab:real-world}
	\textit{Fair} evaluation (Accuracy in \%).
		The best results are highlighted by \FirstColor{\textbf{First}}, \SecondColor{\textbf{Second}} and \ThirdColor{\textbf{Third}}.
The standard deviation is denoted by $^\circ$ for small deviation ($0.0-2.0$), $^\bullet$ for medium deviation ($2.1-4.0$) and $^\star$ for large deviation ($>4.0$).
}

\end{table*}

\begin{table}[t]
	\small
	\centering
\begin{tabular}{lcc|cc}
	\toprule
 & NCI1 & DHFR & IMDB-B & IMDB-M\\
\midrule
    GraphSAGE & \FirstColor{\phantom{-}0.4} & \FirstColor{\phantom{-}2.3} & \SecondColor{-1.1} & \SecondColor{-1.9} \\
	GIN & \SecondColor{-2.8} & \FirstColor{\phantom{-}0.4} & \FirstColor{\phantom{-}3.4} & \FirstColor{\phantom{-}3.0} \\
	GAT  & \SecondColor{-1.1} & \SecondColor{-0.4} & \FirstColor{\phantom{-}0.9} & \SecondColor{-0.7} \\
	GATv2 & \FirstColor{\phantom{-}0.2} & \SecondColor{-0.9} & \FirstColor{\phantom{-}0.0} & \SecondColor{-0.5} \\

	\bottomrule
	\end{tabular}
		\caption{\label{tab:results-with-features}
			GNNs with additional node features (Mean change in accuracy in \%), \FirstColor{positive} and \SecondColor{negative} changes.
			The GNNs use the same label information of the best run of~\MyGNNs as additional input features.
		}
	\end{table}

\paragraph{Results}
The experiments (see~\Cref{tab:real-world,tab:real-world-sota}) demonstrate that  \MyGNNs outperform ordinary MPNNs and are competitive to transformer-based models.
Thus, invariant-based weight sharing enables a single-layer model to capture complex structural dependencies that typically require deep architectures.
Notably, on synthetic benchmarks such as RT and SF, we achieve accuracies exceeding $98\%$,
even though these datasets are designed to challenge standard message-passing methods.
On datasets not solvable by 1-WL (e.g., CSL and Snowflakes (SF)), \MyGNN consistently outperforms message-passing baselines.
This demonstrates that weight-sharing using graph invariants improves \MyGNNs expressivity compared to ordinary message passing.
The comparison to ordinary MPNNs might not seem entirely fair, as our model uses additional structural information encoded in the weight sharing.
Due to their design, the only way to provide this information to ordinary MPNNs is via input features.
Thus, we also compare against
ordinary MPNNs that receive the same structural information (of our best performing model) as input features (see \Cref{tab:details-best-runs}).
Surprisingly, even in this setting, \MyGNN outperforms the ordinary MPNNs consistently (see~\Cref{tab:real-world-sota}) showing that
structure-aware weight sharing is more effective than standard weight sharing based on node and edge features.

\begin{figure*}[t]
		\begin{subfigure}{0.32\linewidth}
		\includegraphics[width=1\linewidth]{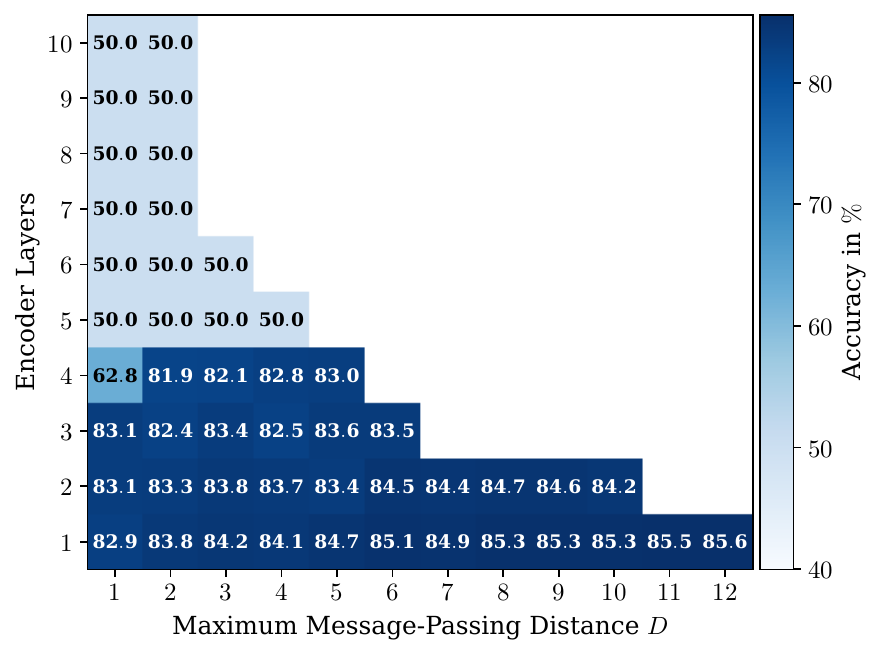}
		\caption[]{NCI1\label{fig:ablation-c}}

	\end{subfigure}
	\begin{subfigure}{0.32\linewidth}
		\includegraphics[width=1\linewidth]{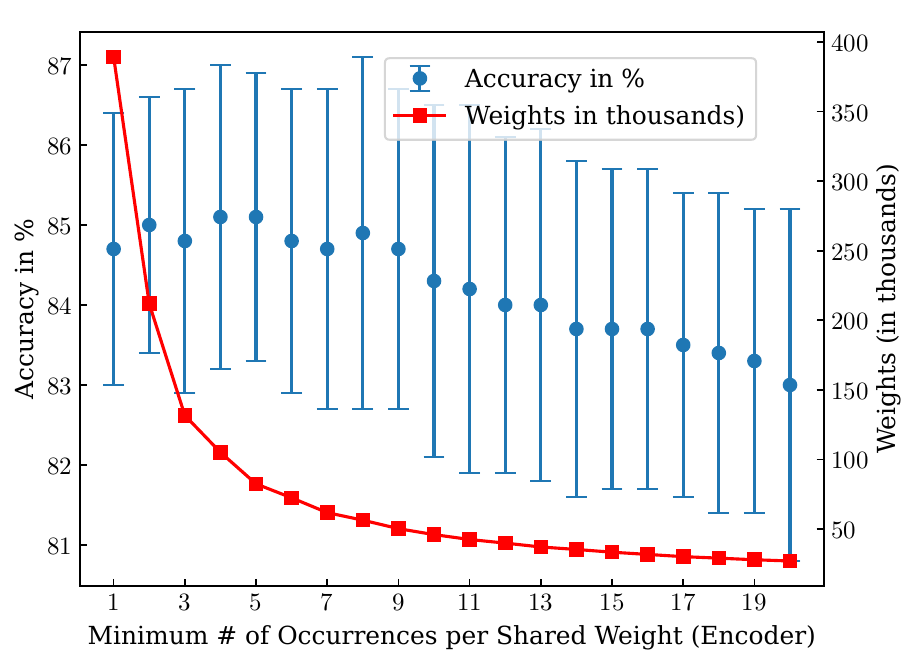}
		\caption[]{NCI1\label{fig:ablation-a}}

	\end{subfigure}
	\begin{subfigure}{0.32\linewidth}
		\includegraphics[width=1\linewidth]{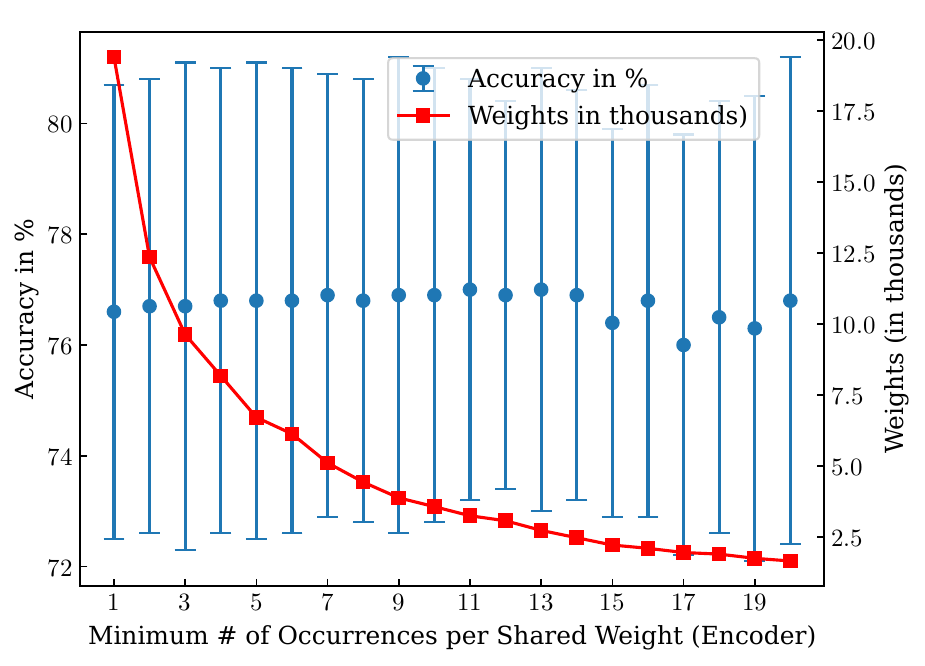}
		\caption[]{IMDB-BINARY\label{fig:ablation-b}}

	\end{subfigure}
	\caption{\label{fig:ablation-parameters} Ablation Study.
		(\ref{fig:ablation-c})~\MyGNN performance for different maximum message-passing distances $D$ ($x$-axis) and different numbers of encoder layers ($y$-axis).
		(\ref{fig:ablation-a}) and (\ref{fig:ablation-b})~\MyGNN performance for different number of weights.
	The model corresponding to the $x$-axis value $i$ contains only those shared message passing weights that occur at least $i$-times in the respective graph dataset.
	}
	\end{figure*}

\paragraph{Ablation Study}
We ablate the number of encoder layers (i), the maximum hop distance $D$ (ii) both in \Cref{fig:ablation-c},
and the size and pruning of $\ruleset$ (iii) in \Cref{fig:ablation-a,fig:ablation-b},
to quantify how invariant-based sharing trades off accuracy and parameter count.
The results for (i) and (ii) show that increasing the maximum hop distance $D$ has more impact on the performance than adding more encoder layers.
This confirms that long-range message passing is crucial for performance and can be captured in shallow architectures.
As for classical MPNNs, adding more layers (here more than four) adds instability and decreases performance.
The results for (iii) show that pruning the set of valid triples $\ruleset$ based on their occurrence in the dataset
allows to reduce the number of weights significantly without sacrificing accuracy.
The same behavior is observed across different datasets (see \Cref{fig:ablation_threshold_lower,fig:ablation_distance} in the supplementary material).

\subsection{Expressive Invariants lead to Expressive Models}\label{subsec:substructure-counting}
Using the Substructure Counting Benchmark (cf.~\citep{DBLP:conf/iclr/ZhaoJAS22,DBLP:conf/nips/FrascaBBM22}) that contains graphs labeled with counts of six motifs (triangles, tailed triangles, stars, and 4-6 cycles)
and measures the ability to distinguish structural patterns we can show that expressive invariants lead to expressive models in practice.
We train a \MyGNN with a single encoder and decoder layer, using five heads:
four heads dedicated to specific cycle lengths and one capturing the global structure.
We consider two settings: one separate model for each motif (one) and a single model jointly for all six motifs (all).
Baselines include strong models like
PPGN~\citep{DBLP:conf/nips/MaronBSL19},
GNN-AK+~\citep{DBLP:conf/iclr/ZhaoJAS22},
SUN (EGO+)~\citep{DBLP:conf/nips/FrascaBBM22}, and GNN-SSWL+~\citep{DBLP:conf/icml/ZhangFDHW23}.
\Cref{tab:graph-counting-comparison} shows that \MyGNN achieves state-of-the-art results in the single-task setting and remains competitive in the multi-task setting, confirming the theoretical properties of invariant-based weight sharing even with shallow architectures.
Moreover, our method allows to visualize the activations of the different heads (\Cref{fig:substructure-counting-heads}), highlighting that \MyGNNs' predictions can be directly interpreted in terms of the structural motifs it uses.

\begin{table*}[t]
\begin{adjustbox}{width=\textwidth}
\begin{tabular}{l|cccccc}
    \toprule
ZINC & GIN  & GSN & GNN-SSWL+ & PPGN++ (6) & ESC-GNN & \textbf{ours} \\
& \citep{DBLP:conf/iclr/XuHLJ19}
  & \citep{DBLP:journals/pami/BouritsasFZB23}
  & \citep{DBLP:conf/icml/PunyLKML23}
  & \citep{DBLP:conf/icml/ZhangFDHW23}
  & \citep{DBLP:conf/kdd/YanZG0Z24}
   & \\  \midrule
    ($12k$) & $0.163 \pm 0.004$  & $0.101 \pm 0.010$ & \FirstColor{0.070 \pm 0.005} & \SecondColor{0.071 \pm 0.001} & $0.075 \pm 0.002$ & $0.107 \pm 0.003$ \\
    ($250k$) & $0.088 \pm 0.002$  & $-$ & $0.022 \pm 0.002$ & \FirstColor{0.020 \pm 0.001} & \SecondColor{0.021\pm 0.003} & $0.024 \pm 0.001$ \\
    \bottomrule
\end{tabular}
\end{adjustbox}
    \caption{\label{tab:zinc-regression}
		ZINC Benchmark (MAE $\downarrow$).
		The best two results are highlighted by \FirstColor{\textbf{First}} and \SecondColor{\textbf{Second}}.
	}
\end{table*}

\begin{table}[t]\centering
\begin{tabular}{lcccccc}
\toprule
 & triangle & tailed triangle & star & $4$-cycles & $5$-cycles & $6$-cycles \\
\midrule
PPGN & $8.9$ & $9.6$ & $14.8$ & $9.0$ & $13.7$ & $16.7$ \\
GNN-AK+ & $12.3$ & $11.2$ & $15.0$ & $12.6$ & $26.8$ & $58.4$ \\
SUN (EGO+) & $7.9$ & \ThirdColor{8.0} & \ThirdColor{6.4} & $10.5$ & $17.0$ & $55.0$ \\
GNN-SSWL+ & \ThirdColor{6.4} & \SecondColor{6.7} & $7.8$ & \ThirdColor{7.9} & \ThirdColor{10.8} & \ThirdColor{15.4} \\
\midrule
\textbf{ours (one)}  & \FirstColor{1.9} & \FirstColor{3.2} & \FirstColor{2.4} & \FirstColor{3.6} & \FirstColor{5.3} & \FirstColor{5.1} \\
\textbf{ours (all)} & \SecondColor{3.3} & $9.1$ & \SecondColor{6.3} & \SecondColor{5.6} & \FirstColor{5.3} & \SecondColor{7.8} \\
\bottomrule
\end{tabular}
\caption{
 Substructure Counting Benchmark MAE ($\downarrow$) in $10^{-3}$ on different tasks.
	The best results are highlighted by \FirstColor{\textbf{First}}, \SecondColor{\textbf{Second}} and \ThirdColor{\textbf{Third}}.
}
\label{tab:graph-counting-comparison}
\end{table}

\begin{figure}[t]
    \centering
	\begin{subfigure}{0.2\linewidth}
		        \includegraphics[width=\linewidth]{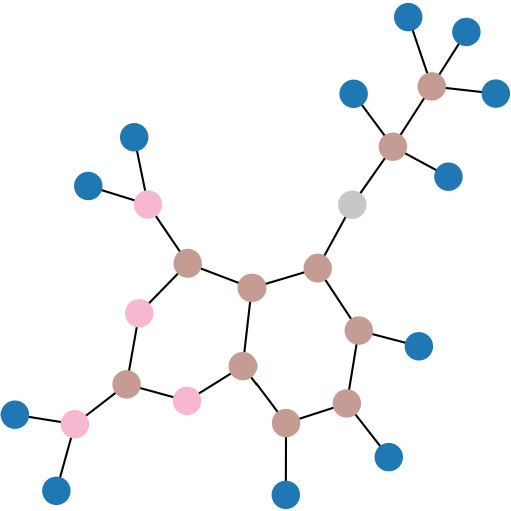}
	\end{subfigure}\hfill
		\begin{subfigure}{0.2\linewidth}
		        \includegraphics[width=\linewidth]{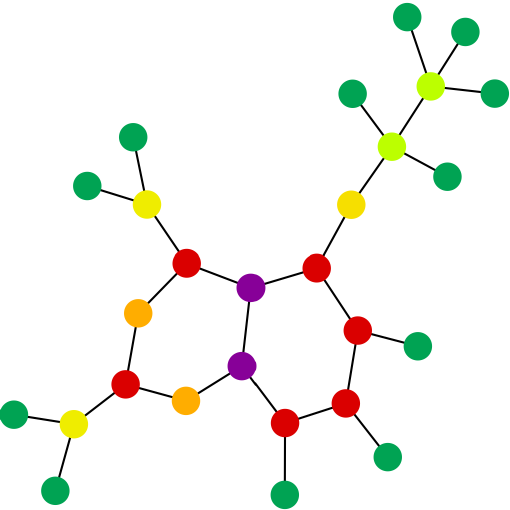}
	\end{subfigure}\hfill
		\begin{subfigure}{0.2\linewidth}
		        \includegraphics[width=\linewidth]{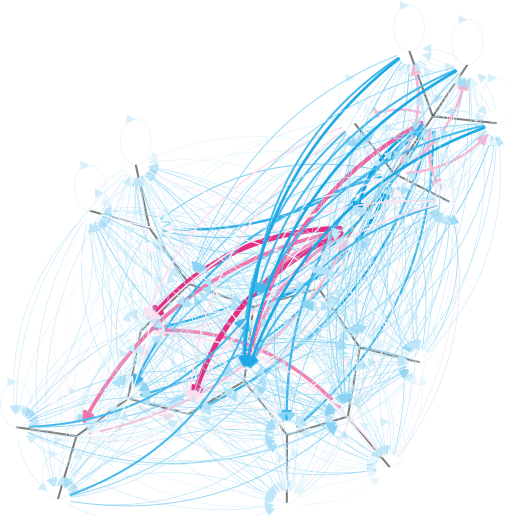}
	\end{subfigure}\hfill
		\begin{subfigure}{0.2\linewidth}
		        \includegraphics[width=\linewidth]{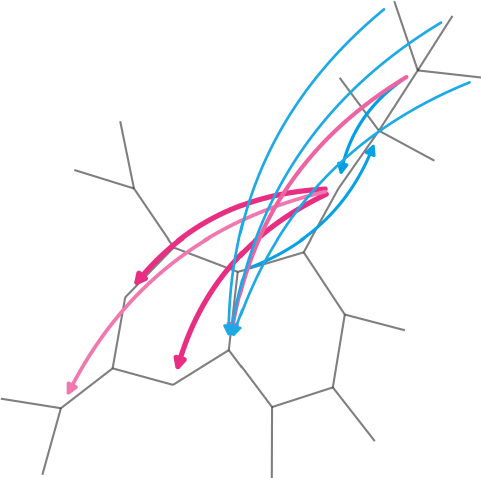}
	\end{subfigure}
    \caption{\label{fig:weight-evaluation} Invariant-based message passing (\MyGNN encoder layer) for a molecule from DHFR.
    	Left to right:
    	atomic numbers, graph invariant labels, learned weights, and the three largest positive (red), negative (blue) ones.
    	The thickness and color of the edges scales with their absolute value.
	}
\end{figure}

\subsection{Scalability}
\label{subsec:graph-regression}
To assess scalability, we evaluate \MyGNN on the ZINC benchmark dataset, considering both the $12k$ and $250k$ graph versions.
For the experiments, we use a \textit{single-layer} encoder and decoder with multiple heads.
Each head focuses on a distinct type of structural invariant, see~\Cref{subsec:regression}.
\MyGNNs attain competitive performance (MAE) on both datasets (\Cref{tab:zinc-regression})
while maintaining a shallow architecture.
Thus, our invariant-based design scales to large datasets while retaining high accuracy, \textit{without} requiring deep or complex networks.
While ShareGNNs achieve competitive performance on ZINC, they do not surpass the best-performing models.
We attribute this gap primarily to two factors.
First, several competing methods incorporate richer chemistry-specific inductive biases, such as detailed bond encodings or higher-order substructure enumerations beyond simple cycle information.
Second, our design intentionally prioritizes \textit{structural parameter sharing} over task-specific feature engineering, favoring interpretability and transferability.
Notably, models such as GSN~\citep{DBLP:journals/pami/BouritsasFZB23},
which also rely heavily on cycle information, achieve comparable performance, supporting the view that the observed gap reflects
differences in inductive bias
rather than a limitation of invariant-based weight sharing itself.

In summary, we show that \textit{structure-aware} weight sharing surpasses the ordinary \textit{feature-based} weight sharing.
While relying on shallow architectures, we
achieve competitive performance to sota models on diverse benchmarks.
Finally, we see that expressive graph invariants lead to expressive models in practice,
confirming that our method provides a
promising novel direction for capturing structural information.

\section{Concluding Remarks and Outlook}\label{sec:Conclusion}

We introduced \textit{structure-aware invariant-based} weight sharing, and demonstrated its
advantages over traditional, fixed weight-sharing schemes for graph learning.
The approach yields permutation-equivariant encoders and permutation-invariant decoders by construction, enables long-range message passing in a single layer, and mitigates over-smoothing effects associated with deep architectures.
Our analysis establishes a direct connection between \MyGNN's expressivity and the discriminative power of the chosen invariants.
Since invariants can be incorporated without modifying the underlying architecture,~\MyGNNs provide a flexible testbed for studying how different invariants affect graph representation learning.
Extensive experiments show that our weight-sharing scheme achieves competitive or superior performance on substructure counting, synthetic, and real-world benchmarks, despite relying on remarkably shallow models.
A key benefit of invariant-based weight sharing is that weights are explicitly tied to graph structure, which improves interpretability by revealing which structural patterns drive predictions (cf. \Cref{fig:weight-evaluation,fig:substructure-counting-heads,fig:interpretability-1,fig:interpretability-2}).

\paragraph{Limitations and Future Work}
Weights are naturally transferable across datasets which makes the invariant-based approach especially well suited for the emerging field of foundation models on graphs
\citep{DBLP:journals/corr/abs-2310-11829,DBLP:journals/tmlr/FranksECWSFK25}, as it naturally supports pretraining on large (synthetic) graph datasets
followed by a transfer to new tasks.
While our focus here was to establish the core formulation, analyze its theoretical properties, and evaluate its empirical performance, future work will scale this approach to larger and more diverse datasets, and investigate its role in transfer across domains and tasks.
ShareGNNs incur additional preprocessing costs for computing node invariants and pairwise distances, and their worst-case memory complexity scales quadratically with the number of nodes.
While in practice this cost is comparable to transformer-based architectures and can be
substantially reduced by restricting the set of valid invariant signatures and pruning rarely occurring relations, scalability to extremely large or temporal graphs remains a challenge.
Furthermore, the choice of invariants currently requires user input and domain knowledge.
Future work will explore automated invariant selection, incremental updates for dynamic graphs, and applications to
other domains such as images and text (cf. \Cref{sec:appendix-text-image}).

\newpage

\section*{Impact Statement}
This paper presents work whose goal is
 to advance the field of machine learning.
There are many potential
societal consequences of our work, none of which we feel must be
specifically highlighted here.

\section*{Ethics statement}
The authors declare that they have no potential conflicts of interest with respect to the research, authorship, and/or publication of this article.

\section*{Reproducibility statement}
Our results can be reproduced using the code and instructions provided under the repository found at
\href{https://github.com/fseiffarth/SimpleGNN}{https://github.com/fseiffarth/SimpleGNN}.
We provide the hardware specifications in~\Cref{subsec:resources-and-implementations}.
All datasets used in our experiments are publicly available.
We provide additional details in~\Cref{tab:real-world-datasets,tab:synthetic-datasets} as well as illustrative examples of graphs from the synthetic datasets~(\Cref{fig:ring-transfer-1-plot,fig:ring-transfer-2-plot,fig:ring-transfer-3-plot,fig:csl-plot,fig:snowflakes-plot}).
Moreover, in the appendix we provide detailed information on the experimental setup, including hyperparameters~(\Cref{tab:invariants-social-encoder,tab:details-best-runs,tab:zinc-hparams}), preprocessing~(\Cref{tab:preprocessing-times}) and architecture details~(\Cref{tab:zinc-arch}).

\bibliography{bibliography}
\bibliographystyle{plainnat}

\newpage
\section*{The Use of Large Language Models (LLMs)}
We used large language models for minor editorial polishing (e.g., grammar, wording).
All the ideas, analyses, and conclusions remain entirely our own.

\appendix

\section{Proofs}\label{sec:proofs}
\textbf{Proposition~\ref{prop:permutation-equivariance}}
\textit{
    The encoder layer is permutation equivariant, i.e.,
    for each permutation $\pi$ of the node order of the input graph it holds
    $\pi(\mathbf{X}^{(h+1)}) = \sigma(\weightmatrix\cdot \pi(\mathbf{X}^{(h)}) + \bias)$.
}
\begin{proof}
	We prove that the encoder layers of~\MyGNNs are permutation equivariant, by showing that the following holds for all permutations $\pi$ of the nodes of the graph:
	\[\pi\left(X^{(h+1)}\right) = \sigma\left(\weightmatrix\cdot \pi\left(X^{(h)}\right) + \bias\right)\]
	Note that a permutation of the nodes of a graph corresponds to a permutation of the rows of its corresponding node representation $X^{(h+1)}$.
	The entries of the weight matrix and the bias term in our definition of the forward propagation depend on the fixed order of the nodes of the input graph.
	Thus, considering the permuted input $\pi(X^{(h)})$, the corresponding
	weight matrix and the bias term are by definition,
	permuted in the same way compared to the original input signal $X^{(h+1)}$.
	Hence, it follows that $\weightmatrix\cdot \pi(X^{(h)}) = \pi(\weightmatrix\cdot X^{(h)})$ and
	$\pi(\weightmatrix\cdot X^{(h)}) + \bias = \pi(\weightmatrix\cdot X^{(h)} + \bias)$ which completes the proof.
\end{proof}
\textbf{Proposition~\ref{prop:permutation-invariance}}
\textit{
	The decoder layer is permutation invariant, i.e.,
    for each permutation $\pi$ of the node order of the input graph it holds
    $\mathbf{X}^{(h+1)} = \sigma(\frac{1}{n}\weightmatrix\cdot \pi(\mathbf{X}^{(h)}) + \bias)$.
}
\begin{proof}
	We prove that the decoder of~\MyGNNs is permutation invariant, by showing that the following holds for all permutations $\pi$ of the nodes of the graph:
	\[
		\mathbf{X}^{(h+1)} = \sigma\left(\frac{1}{n}\weightmatrix\cdot \pi\left(\mathbf{X}^{(h)}\right) + \bias\right)
	\]
	Again, the columns of the weight matrix $\weightmatrix$ depend on the fixed order of the nodes of the input graph.
	Thus, if the rows in $X^{(h)}$ are permuted, the corresponding columns of $\weightmatrix$ are permuted in the same way.
	Hence, the result of the multiplication remains unchanged, i.e., we have that
	$\weightmatrix\cdot \pi(\mathbf{X}^{(h)}) = \weightmatrix\cdot \mathbf{X}^{(h)}$, which completes the proof.
\end{proof}
\textbf{Corollary~\ref{cor:permutation-invariance}}
\textit{
	\MyGNNs are permutation invariant.
}
\begin{proof}
	Permutation invariance of~\MyGNNs follows directly from the permutation equivariance the encoder layers and the permutation invariance of the decoder.
	\end{proof}
\textbf{Proposition~\ref{prop:expressivity-MyGNN}}
\textit{
	Each~\MyGNN is at least as expressive as the labeling function $l$ of the decoder, i.e.,
	if two graphs $G$ and $G'$ are distinguishable by $l$ then they are also distinguishable by the~\MyGNN.
}
\begin{proof}
	We show that the expressive power of the~\MyGNNs is based on the expressive power of the underlying labeling functions.
	Indeed, we show that~\MyGNNs are at least as powerful as the labeling function $\graphlabeling$
	of the decoder.
	Assume, $l:V,V'\rightarrow\mathbb{N}_0$ is a labeling function that distinguishes the non-isomorphic graphs $G=(V,E)$ and $G'=(V',E')$ with $|V|=|V'|=n$ by comparing the histograms of the label counts.
	More precisely, if $G$ and $G'$ are non-isomorphic
	there exists some $a\in\mathbb{N}_0$ such that $A \coloneqq \sum_{v\in V} \mathbf{1}_{l(v)=a} \neq \sum_{v\in V'} \mathbf{1}_{l(v)=a}\eqqcolon A'$ with
	$\mathbf{1}$ being the indicator function.
	For example, we can define $l$ to be the $(k+1)$-Weisfeiler-Leman labels with $k$ being the maximum of the treewidths of $G$ and $G'$~\citep{DBLP:journals/jgt/Dvorak10}.
	Let~\MyGNN consist of only a decoder layer based on the labeling function $l$ with a single output neuron, i.e., $m=1$.
	The corresponding set of weights is denoted by $\weightset_{\labelset}$ where $|\labelset|$ is the number of different ${(k+1)}$-Weisfeiler-Leman labels
	that occur for the nodes of $G$ and $G'$.
	Without loss of generality we assume that the graph representations of $G$ and $G'$
	denoted by $X^G, X^{G'}\in\mathbb{R}^{n}$ are equal to vectors of ones.
	Let $\omega_a = 1$ and $\omega_b=0$ for all $b\in\labelset\setminus\{a\}$.
	It follows that
	$\weightmatrix\cdot X^G=A\neq A'=\weightmatrix\cdot X^{G'}$ showing that the above defined~\MyGNN is able to distinguish $G$ and $G'$.
	Moreover, we have shown that \textit{every}~\MyGNN is at least as powerful as the labeling function $l$ of the decoder.
\end{proof}
\textbf{Corollary~\ref{cor:expressivity-MyGNN}}
\textit{
    For each two non-isomorphic graphs $G$ and $G'$ there exists a ~\MyGNN that distinguishes $G$ and $G'$.
}
\begin{proof}
    The proof follows directly from Proposition~\ref{prop:expressivity-MyGNN} and the fact that there exists a labeling function $l$ that distinguishes $G$ and $G'$.
    For example, we can define $l$ to be the $(k+1)$-Weisfeiler-Leman labels with $k$ being the maximum of the treewidths of $G$ and $G'$~\citep{DBLP:journals/jgt/Dvorak10}.
\end{proof}

\section{Architecture Vizualizations}\label{sec:architecture-visualizations}
\Cref{fig:shared-attention-coefficients} visualizes the update of a single node $v_i$ in the encoder (left) and the aggregation of the final node embeddings in the decoder (right).

\begin{figure*}[t]\centering
	\begin{tikzpicture}\centering
		\node[graph vertex, fill=aqua!60] (v1) at (1,0) {\small $x_1^{(0)}$};
		\node at (1.75,0) {\ldots};
		\node[graph vertex, fill=aqua!60] (v4) at (2.5,0) {\small$x_n^{(0)}$};
		\node[graph vertex, fill=aqua!60] (bias) at (1.75, -1.2) {$+$};
		\node[graph vertex, fill=aqua!60] (v5) at (1.75,-2.2) {\small$x_i^{(1)}$};
		\path (v1) edge[-latex] node[left] {\small$\cdot~{\MessagePassingFunction^{\MPLabel}_{\ruleset}(v_1,v_i)}$} (bias);
		\path (v4) edge[-latex] node[right] {\small$\cdot~{\MessagePassingFunction^{\MPLabel}_{\ruleset}(v_n,v_i)}$} (bias);
		\path (bias) edge[-latex] node[right] {\small$+~{\BiasFunction^{\BiasLabel}_{\labelset, k}(v_i)}$} (v5);

	\end{tikzpicture}
	\begin{tikzpicture}\centering
		\node[graph vertex, fill=aqua!60] (x1) at (0,0.4) {\small $x_1^{(h)}$};
		\node at (1.7,0.4) {\ldots};
		\node[graph vertex, fill=aqua!60] (xn) at (3.4,0.4) {\small$x_n^{(h)}$};
		\node[graph vertex, fill=aqua!60] (p1) at (0.9,-1.2) {$+$};
		\node at (1.7,-1.2) {\ldots};
		\node[graph vertex, fill=aqua!60] (pm) at (2.5,-1.2) {$+$};

		\node[graph vertex, fill=aqua!60] (y1) at (0.9,-2) {\small$y_1$};
		\node at (1.7,-2) {\ldots};
		\node[graph vertex, fill=aqua!60] (ym) at (2.5,-2) {\small$y_m$};

		\draw[-latex] (x1) -- (p1) node[pos=0.45, below, sloped] {\small$\cdot~{\BiasFunction^{l}_{\labelset, m}(v_1)}_1$};
		\draw[-latex] (xn) -- (pm) node[pos=0.4, below, sloped] {\small$\cdot~{\BiasFunction^{l}_{\labelset, m}(v_n)}_m$};
		\draw[-latex] (x1) -- (pm) node[pos=0.3, above, sloped] {\small$\cdot~{\BiasFunction^{l}_{\labelset, m}(v_1)}_m$};
		\draw[-latex] (xn) -- (p1) node[pos=0.3, above, sloped] {\small$\cdot~{\BiasFunction^{l}_{\labelset, m}(v_n)}_1$};

		\path (p1) edge[-latex] node[left] {\small$+\bias_1$} (y1);
		\path (pm) edge[-latex] node[right] {\small$+\bias_m$} (ym);

	\end{tikzpicture}
	\caption{\label{fig:shared-attention-coefficients}
	Update ($v_i$) in the encoder (left)  and
	aggregation of the final node embeddings in the decoder (right)
	}
\end{figure*}
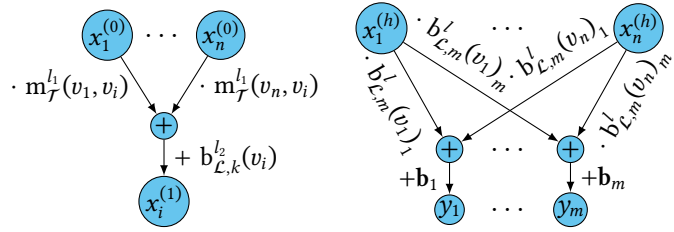

\section{Experimental Details}\label{sec:evaluation-details}
We provide additional details about our experiments, including details on the benchmark datasets,
the hyperparameters used to train the models, and the implementation of the~\MyGNNs.

Since standardized evaluation protocols are often missing in graph learning, direct comparison across works is difficult~\citep{DBLP:conf/iclr/ErricaPBM20}.
To ensure clarity and reproducibility, we follow \textit{two} widely used setups:
\textit{fair}~\citep{DBLP:conf/iclr/ErricaPBM20} and~\textit{standard}~\citep{DBLP:conf/iclr/XuHLJ19}.
\paragraph{Dataset Choice}
As noted by, e.g.,~\cite{Schulz2019OnTN,DBLP:conf/iclr/ErricaPBM20,DBLP:conf/icml/Bechler-Speicher24}, it difficult to assess the true benefits of advanced models, as even simple baselines often perform competitively.
For very small datasets (MUTAG, PTC), we observed that results vary dramatically with the data split, making meaningful comparisons unreliable.
To ensure a fair and informative evaluation, we therefore focus on datasets where structural information is essential by also including challenging synthetic benchmarks RingTransfer (RT1-RT3) and Snowflakes (SF).
These are designed to test long-range dependencies and structural reasoning.
Moreover, we consider molecular datasets (DHFR, Mutagenicity (Muta), NCI1, NCI109), social network datasets (IMDB-BINARY (IM-B), IMDB-MULTI (IM-M)) collected by~\cite{TUDortmund}.
\paragraph{\textit{Fair} Evaluation \citep{DBLP:conf/iclr/ErricaPBM20}}
Each dataset is split into $10$ predefined folds, with separate training, validation, and test sets.
Models are trained on the training folds, and the hyperparameter configuration that performs best on the validation sets is used to evaluate the corresponding test sets.
We report the average test accuracy over 10 folds, selecting the epoch with the best validation performance for each fold and repeating the procedure with three random seeds
 (\Cref{tab:real-world}).
\paragraph{Competitor Choice}
We consider several baselines, including, non-neural methods such as simple label histograms (NoG)~\citep{Schulz2019OnTN} and Weisfeiler–Leman subtree kernel (WL)~\citep{DBLP:journals/jmlr/ShervashidzeSLMB11},
classical GNNs, GraphSAGE~\citep{Hamilton2017InductiveRL}, GIN~\citep{DBLP:conf/iclr/XuHLJ19}, GAT~\citep{Velickovic2017GraphAN} and GATv2~\citep{DBLP:conf/iclr/Brody0Y22},
and modern transformer-based models, GraphTransformer (GT)~\citep{DBLP:conf/ijcai/ShiHFZWS21} and GT+R-Cov~\citep{DBLP:conf/icml/Bechler-Speicher24} that
are evaluated in the \textit{fair} setup.
All baselines use the same train/validation/test splits, if available by~\citet{DBLP:conf/iclr/ErricaPBM20} and ~\citet{DBLP:journals/jmlr/DwivediJL0BB23} (CSL).

\subsection{Implementation}\label{subsec:implementation}
To reproduce the results we provide the implementation under the following link\\~\href{https://github.com/fseiffarth/SimpleGNN}{https://github.com/fseiffarth/SimpleGNN}.
Moreover, we give a detailed description on how to run the experiments, and how to add custom datasets, and custom labeling functions.
Labeling functions and other hyperparameters are easily adaptable by changing our config files.
In fact, all precomputation of already implemented labeling functions and pairwise distances between nodes is
automated and also executed when custom datasets are added.
Our implementation is fully compatible with the PyTorch Geometric library~\citep{Fey/Lenssen/2019} and the PyTorch library~\citep{DBLP:conf/nips/PaszkeGMLBCKLGA19}.

\subsection{Resources}\label{subsec:resources-and-implementations}
All experiments were conducted on a machine with an AMD Ryzen 9 7950X 16-core processor with $128$ GB of RAM.

\subsection{Dataset Details}\label{subsec:details-on-the-datasets}
In this section, we provide additional details about the datasets used in the experiments.
First, we provide a detailed description of our new synthetic benchmark datasets.

\paragraph{\textit{RingTransfer1}} The dataset consists of $1200$ cycles of $100$ nodes each, and is designed to test the ability to detect long-range dependencies.
Four of the cycle nodes are labeled by $1, 2, 3, 4$, and all the others by $0$.
The distance between each pair of the four nodes is exactly $25$ or $50$.
The label of the graph is $0$ if $d(1,2)=50$, $1$ if $d(1,3)=50$, and $2$ if $d(1,4)=50$.
\Cref{fig:ring-transfer-1-plot} shows an example of the dataset.
There are $400$ graphs per class.
The difficulty of the classification task lies in the fact that the information has to be propagated over a long distance.
For~\MyGNNs this is very easy because information can be propagated over arbitrary distances in a single encoder layer.

\paragraph{\textit{RingTransfer2}}
The dataset consists of $1200$ cycles of $16$ nodes each and is designed to test the ability
to detect long-range dependencies as well as the ability to add node labels and detect even and odd numbers.
The nodes in each graph are labeled from $0$ to $15$.
For all nodes and their opposite node in the circle (distance 8) the sum of their labels is computed.
If there are more even sums than odd sums, the graph is labeled $0$, otherwise it is labeled $1$.
\Cref{fig:ring-transfer-2-plot} shows an example of the dataset.
There are $600$ graphs per class.
We use the information that only distance $8$ is relevant, by only assigning weights to node pairs with distance $8$.

\paragraph{\textit{RingTransfer3}}
The dataset consists of the same graphs as \textit{RingTransfer2}.
However, the graphs are labeled differently.
The graph label is determined by the labels of the nodes at distances $8$ and $4$ from the node with the label $0$.
We denote the node at distance $8$ by $x$ and those at distance $4$ by $y, z$.
There are four cases: $x$ is even and $y+z$ is even, $x$ is even and $y+z$ is odd, $x$ is odd and $y+z$ is even, $x$ is odd and $y+z$ is odd.
Each case corresponds to a class label from $0$ to $3$.
\Cref{fig:ring-transfer-3-plot} shows an example of the dataset.
We construct $300$ graphs per class, that is, $100$ graphs for each of the four cases.
We use the information that information has to be collected from nodes at distances $8$ and $4$ only, by only assigning weights to node pairs with distances $4$ and $8$.

\paragraph{\textit{Snowflakes}} The dataset consists of graphs proposed by~\citet{naik2024iterative} that are indistinguishable
by the 1-WL test, see~\Cref{fig:snowflakes-plot} for an example.
The dataset consists of circles of length $3$ to $12$ and at each circle node a graph from $M_0, M_1, M_2$ or $M_3$ is attached, see~\Cref{fig:m-plot}.
$M_0, M_1, M_2$ and $M_3$ are non-isomorphic graphs that are not distinguishable by the 1-WL test.
We refer to~\citet{naik2024iterative} for the details.
One node in the circle is labeled by $1$ and all other graph nodes are labeled by $0$.
The label of the graph equals the index of the graph $M_0, M_1, M_2$ or $M_3$ that is attached to the circle node with label $1$.

Tables~\ref{tab:real-world-datasets} and~\ref{tab:synthetic-datasets} provide an overview of the real-world and synthetic datasets,
including the number of graphs, the number of nodes, the number of edges, the diameter, the number of node labels and the number of classes.

\begin{table*}[t]
	\centering
\begin{adjustbox}{width=\linewidth}
    \begin{tabular}{lr|rrr|rrr|rrr|rr}
        \toprule
        Dataset & \#Graphs &  \multicolumn{3}{c}{\#Nodes} & \multicolumn{3}{c}{\#Edges} & \multicolumn{3}{c}{Diameter} & \#Node Labels & \#Classes \\
         &  & max & avg & min & max & avg & min & max & avg & min & & \\
        \midrule

        NCI1 & 4\,110 & 111 & 29.9 & 3 & 119 & 32.3 & 2 & 45 & 11.5 & 0 & 37 & 2 \\
        NCI109 & 4\,127 & 111 & 29.7 & 4 & 119 & 32.1 & 3 & 61 & 11.3 & 0 & 38 & 2 \\
        Mutagenicity & 4\,337 & 417 & 30.3 & 4 & 112 & 30.8 & 3 & 41 & 6.3 & 0 & 14 & 2 \\
        DHFR & 756 & 71 & 42.4 & 20 & 73 & 44.5 & 21 & 22 & 14.6 & 8 & 9 & 2 \\
        \midrule
        IMDB-BINARY & 1\,000 & 136 & 19.8 & 12 & 1249 & 96.5 & 26 & 2 & 1.9 & 1 & 1 & 2 \\
        IMDB-MULTI & 1\,500 & 89 & 13.0 & 7 & 1467 & 65.9 & 12 & 2 & 1.5 & 1 & 1 & 3 \\
        \midrule
        ZINC ($12k$) & 12\,000 & 37 & 23.2 & 9 & 42 & 24.9 & 8 & 22 & 12.5 & 4 & 21 & - \\
        ZINC ($250k$) & 249\,456 & 38 & 23.2 & 6 & 45 & 24.9 & 5 & 23 & 12.5 & 3 & 28 & - \\
        \bottomrule
    \end{tabular}
\end{adjustbox}
    	    \caption{\label{tab:real-world-datasets}
    	Details of the real-world datasets used in the experiments.}

\end{table*}

\begin{table*}[t]
	\centering
	\begin{adjustbox}{width=\linewidth}
	\begin{tabular}{lr|rrr|rrr|rrr|rr}
		\toprule
		Dataset & \#Graphs &  \multicolumn{3}{c}{\#Nodes} & \multicolumn{3}{c}{\#Edges} & \multicolumn{3}{c}{Diameter} & \#Node Labels & \#Classes \\
		&  & max & avg & min & max & avg & min & max & avg & min & & \\
		\midrule
		RingTransfer1 & 1\,200 & 100 & 100.0 & 100 & 100 & 100.0 & 100 & 50 & 50.0 & 50 & 5 & 3 \\
		RingTransfer2 & 1\,200 & 16 & 16.0 & 16 & 16 & 16.0 & 16 & 8 & 8.0 & 8 & 16 & 2 \\
		RingTransfer3 & 1\,200 & 16 & 16.0 & 16 & 16 & 16.0 & 16 & 8 & 8.0 & 8 & 16 & 4 \\
		CSL & 150 & 41 & 41.0 & 41 & 82 & 82.0 & 82 & 10 & 6.0 & 4 & 1 & 10 \\
		Snowflakes & 1\,000 & 180 & 112.5 & 45 & 300 & 187.5 & 75 & 18 & 15.5 & 13 & 2 & 4 \\
		\midrule
		Substructure Conting & 5\,000 & 30 & 18.8 & 10 & 45 & 31.3 & 20 & 10 & 4.2 & 0 & 1 & - \\
		\bottomrule
	\end{tabular}
    \end{adjustbox}
	\caption{\label{tab:synthetic-datasets}
		Details of the synthetic datasets used in the experiments.}
\end{table*}

\begin{figure}[t]
	\centering
    {\includegraphics[width=0.48\textwidth]{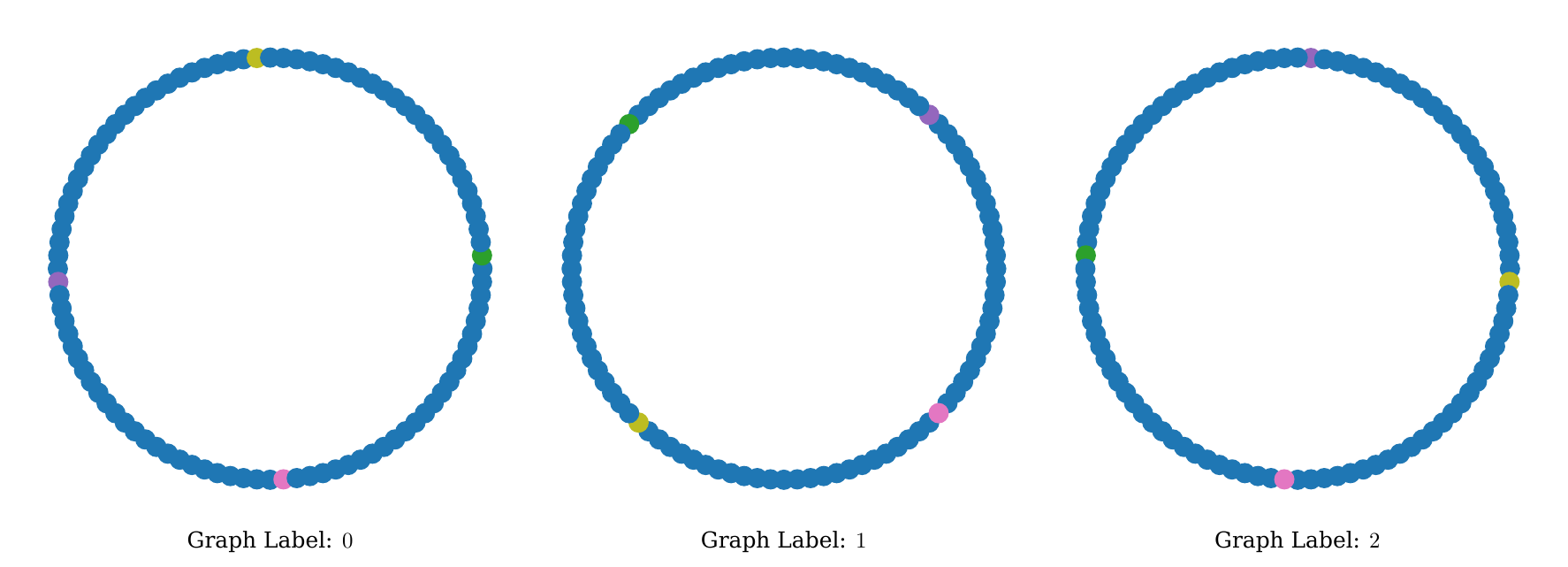}}
    \caption{\label{fig:ring-transfer-1-plot}
    Example graphs taken from the \textit{RingTransfer1} dataset.
    }
\end{figure}

\begin{figure}[t]
	\centering
    {\includegraphics[width=0.25\textwidth]{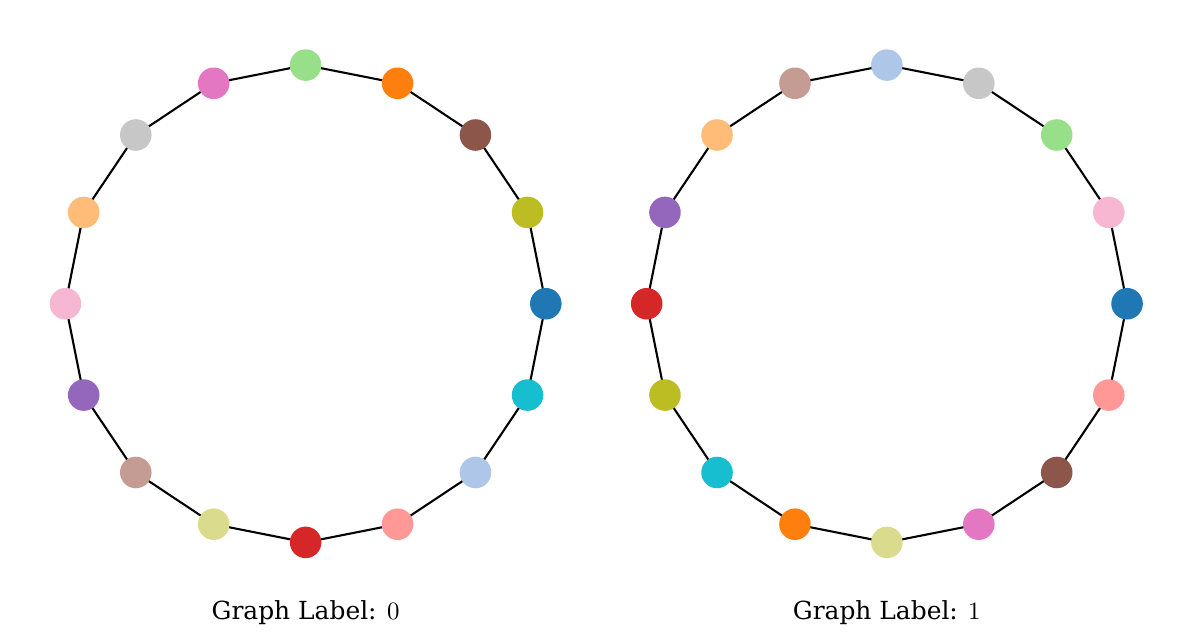}}
    \caption{\label{fig:ring-transfer-2-plot}
    Example graphs taken from the \textit{RingTransfer2} dataset.
    }
\end{figure}

\begin{figure}[t]
	\centering
    {\includegraphics[width=0.48\textwidth]{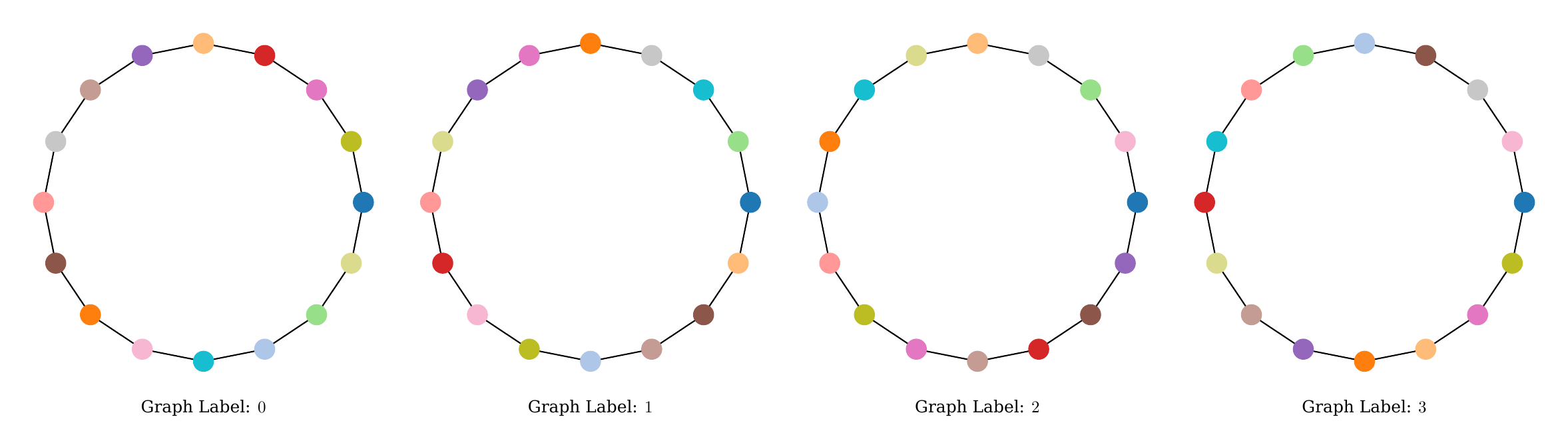}}
    \caption{\label{fig:ring-transfer-3-plot}
    Example graphs taken from the \textit{RingTransfer3} dataset.
    }
\end{figure}

\begin{figure}[t]
	\centering
    {\includegraphics[width=0.48\textwidth]{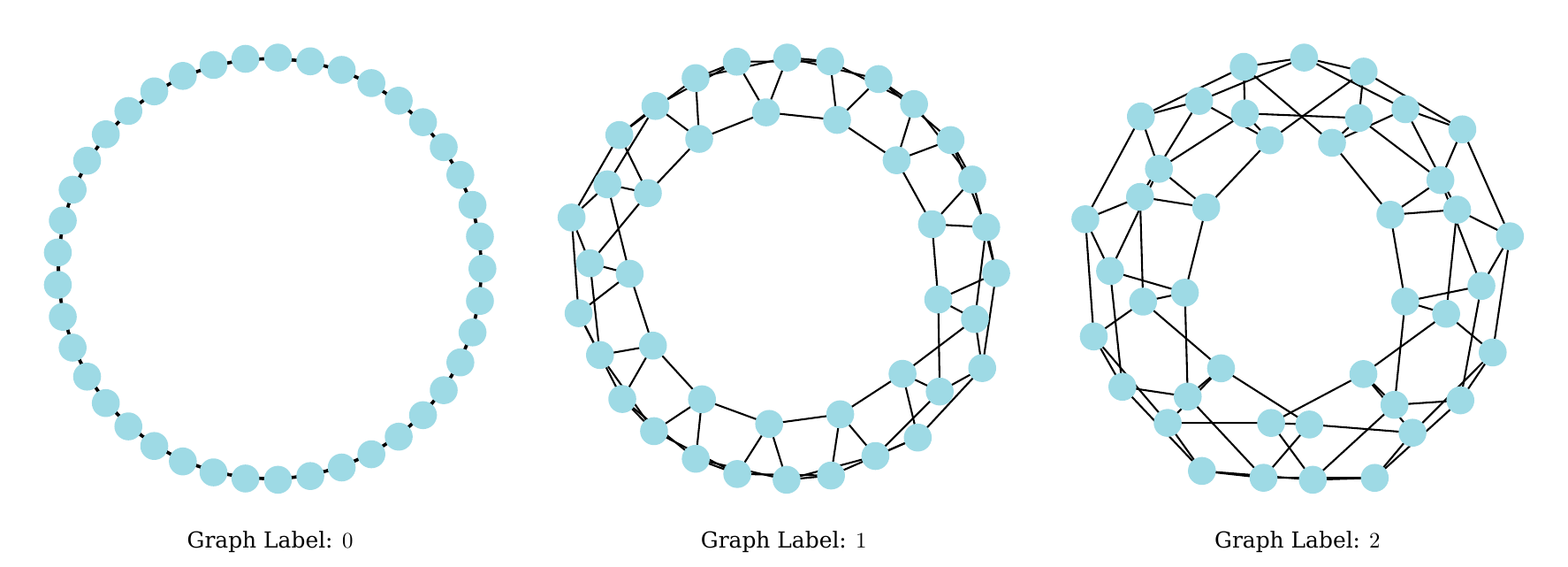}}
    \caption{\label{fig:csl-plot}
    Example graphs taken from the \textit{CLS} dataset.
    }
\end{figure}

\begin{figure*}[t]
	\centering
    {\includegraphics[width=\textwidth]{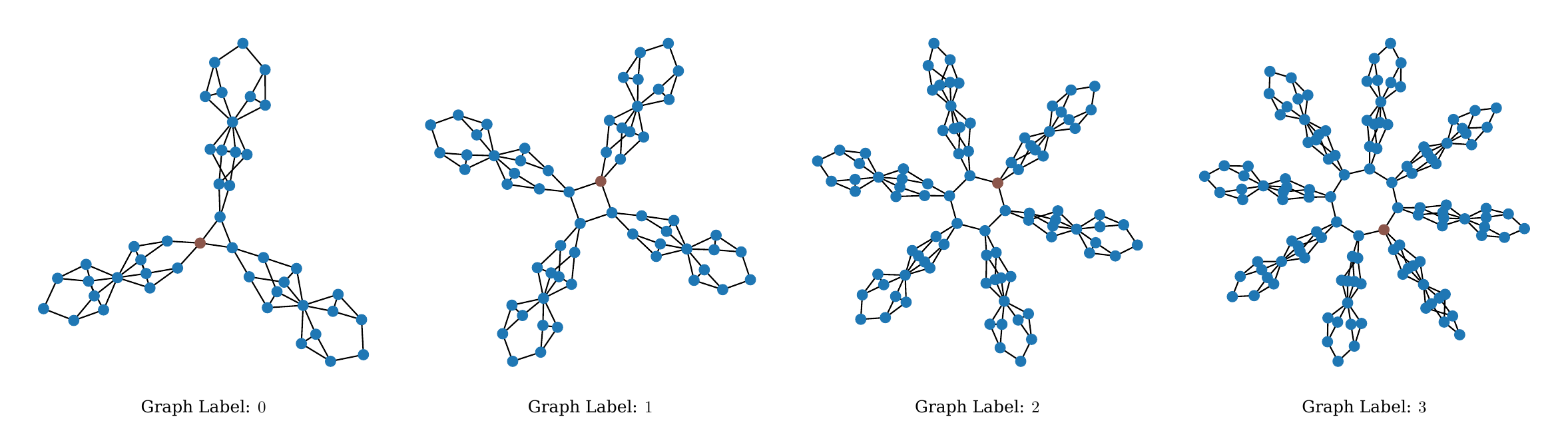}}
    \caption{\label{fig:snowflakes-plot}
    Example graphs from the \textit{Snowflakes} dataset. The brown node in the circle is labeled by $1$ and the other nodes by $0$.
    The label of the graph is determined by the subgraph attached to the brown node.
    }
\end{figure*}

\begin{figure}[t]\centering
    {\includegraphics[width=0.48\textwidth]{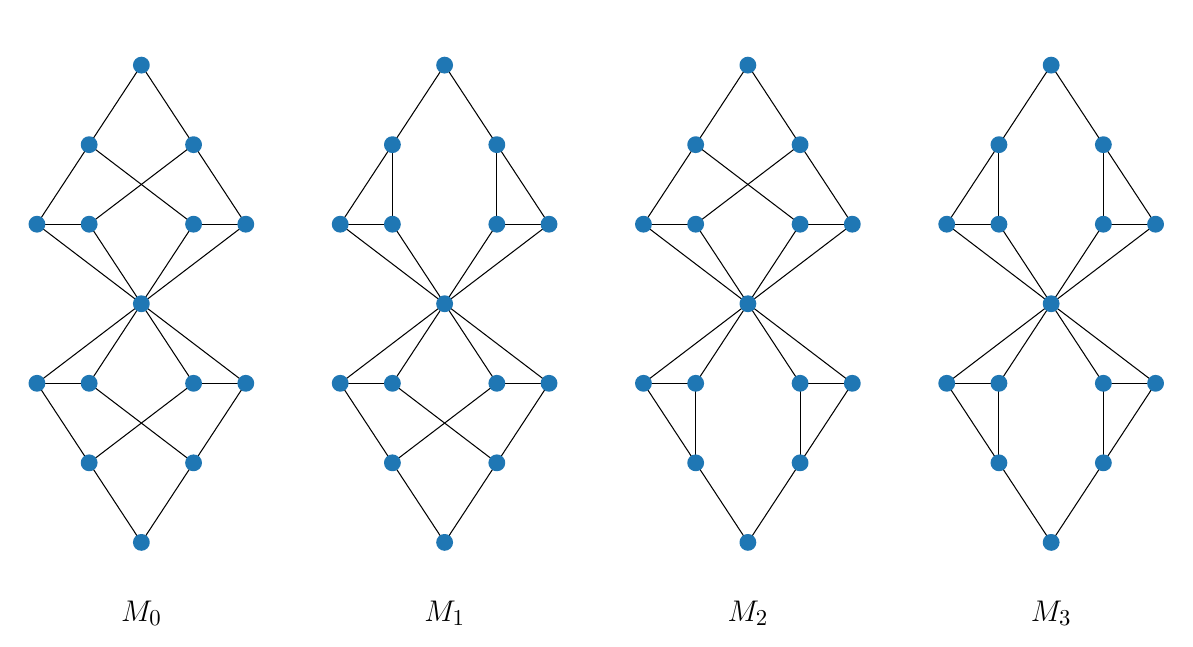}}
    \caption{\label{fig:m-plot}
    Graphs $M_0, M_1, M_2$ and $M_3$ \citep{naik2024iterative} that are not distinguishable by the 1-WL test.
    }
\end{figure}

\begin{table*}[t]
		\centering
\begin{adjustbox}{max width=\textwidth}
			\begin{tabular}[t]{ccc}
				\begin{tabular}[t]{l}
\toprule
\textbf{Molecules: Encoder/Decoder Labeling Function} \\
\midrule
Atomic Numbers \\
Node Degree \\
\midrule
WL, Iterations 1 \\
WL, Iterations 2 \\
WL, Iterations 3 \\
\midrule
Induced Cycles, Max.~Length 10 + Atomic Numbers \\
Induced Cycles, Max.~Length 20 + Atomic Numbers \\
Simple Cycles, Max.~Length 10 + Atomic Numbers \\
Simple Cycles, Max.~Length 20 + Atomic Numbers \\
Simple Cycles, Max.~Length 6 + Atomic Numbers \\
\bottomrule
\end{tabular}
				&
\begin{tabular}[t]{l}
\toprule
\textbf{Social: Encoder Labeling Function} \\
\midrule
Node Degree \\
\midrule
Induced Cycles, Max.~Length 10 \\
Induced Cycles, Max.~Length 20 \\
Induced Cycles, Max.~Length 4 \\
Induced Cycles, Max.~Length 5 \\
Simple Cycles, Max.~Length 3 \\
Simple Cycles, Max.~Length 4 \\
Simple Cycles, Max.~Length 5 \\
\midrule
Clique, Max.~Size 3 \\
Clique, Max.~Size 4 \\
Clique, Max.~Size 10 \\
Clique, Max.~Size 20 \\
Clique, Max.~Size 6 \\
Clique, Max.~Size 50 \\
\midrule
Clique Size 4 \\
Squares + Node Degree \\
Triangles + Node Degree \\
Triangles, Squares + Node Degree \\
\midrule
WL, Iterations 1 \\
WL, Iterations 2 \\
WL, Iterations 3 \\
WL, Iterations 4 \\
\bottomrule
\end{tabular}
&
				\begin{tabular}[t]{l}
\toprule
\textbf{Social: Decoder Labeling Function} \\
\midrule
Node Degree \\
\midrule
Induced Cycles, Max.~Length 10 \\
Induced Cycles, Max.~Length 20 \\
Induced Cycles, Max.~Length 4 \\
Induced Cycles, Max.~Length 5 \\
Simple Cycles, Max.~Length 3 \\
Simple Cycles, Max.~Length 4 \\
Simple Cycles, Max.~Length 5 \\
\midrule
Clique, Max.~Size 10 \\
Clique, Max.~Size 20 \\
Clique, Max.~Size 3 \\
Clique, Max.~Size 4 \\
Clique, Max.~Size 50 \\
Clique, Max.~Size 6 \\
\midrule
WL, Iterations 1 \\
WL, Iterations 2 \\
WL, Iterations 3 \\
WL, Iterations 4 \\
\bottomrule
\end{tabular}
\end{tabular}
\end{adjustbox}

\caption{\label{tab:invariants-social-encoder}
	Labeling functions used for the encoder and decoder layers for the molecules (left) and
	social datasets (middle and right).
	For the molecules are $10$ different options for the encoder and decoder layers, i.e., in total
	we tested $100$ different hyperparameter configurations.
	For the social datasets there are $22$ options for the encoder and $18$ for the decoder layers, i.e., in total
	we tested $396$ different hyperparameter configurations.
}
\end{table*}
\subsection{Substructure Counting Benchmark}\label{subsec:configuration-substructure-counting}
In this section, we provide additional details about the Substructure Counting Benchmark dataset and the configuration of the~\MyGNNs.
The Substructure Counting Benchmark dataset consists of $5\,000$ graphs (cf.~\Cref{tab:synthetic-datasets}) with predefined split $1\,500/1\,000/2\,500$ for training/validation/test.
As for the other datasets the \MyGNN consists of one encoder layer and one decoder layer.
Different from the classification benchmark configuration, we use multiple heads for both the encoder and decoder layers to capture different
invariants in one layer.
More precisely, for the encoder layer and decoder layer we use $5$ heads each.
The graph invariants respectively labeling functions used in the heads are simple cycles of length $3,4,5,6$ and the degree of the nodes.
In case of the encoder for the cycles we use only distance $0$ information, corresponding to message-passing on self-connections and for the node degree we use all distances from $0$ to $10$,
which is the maximum diameter of the graphs in the dataset.
We concatenate the heads of the encoder using a feature output dimension of $10$.
The output dimension of the decoder is $10$.
We add a fully connected layer to match the required output dimension.
All activations are LeakyReLU with the default slope of $0.01$.
The architecture definition can be found in the \textit{network\_substructure\_counting.yml} configuration file.
We use the Adam optimizer, MAE loss, batch size $64$, constant learning rate $0.001$ and $1000$ epochs.
\Cref{fig:substructure-counting-heads} shows the different activations per head.
We see that for the cycle invariants only self-connection weights are set while using the degree as invariant we get weights between all pairs of nodes.
Moreover, the visualization shows that for the different tasks (rows) the corresponding heads are activated while for the (all) model trained on all tasks simultaneously no specific head activation can be seen.

\begin{figure*}[t]
    \centering
    \includegraphics[scale=0.26]{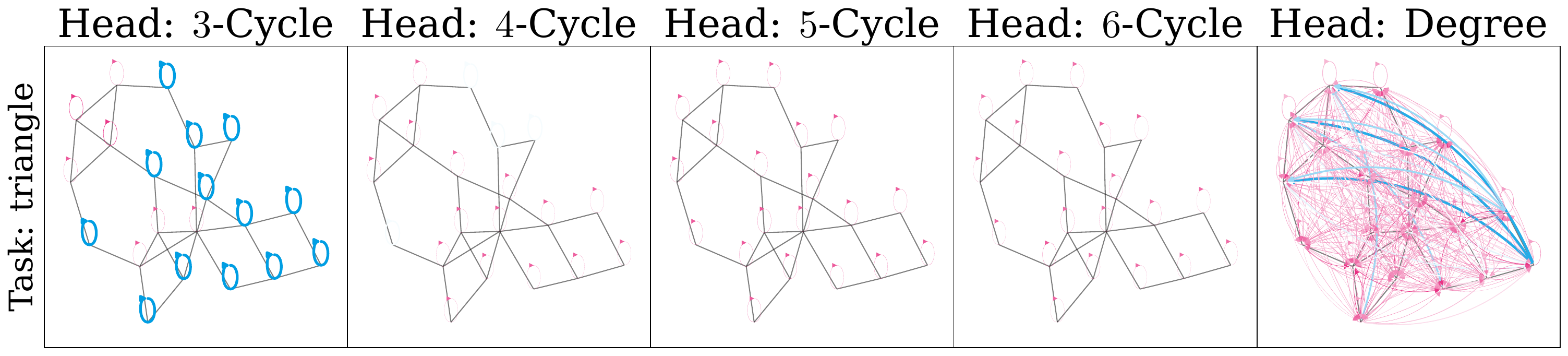}
    \includegraphics[scale=0.26]{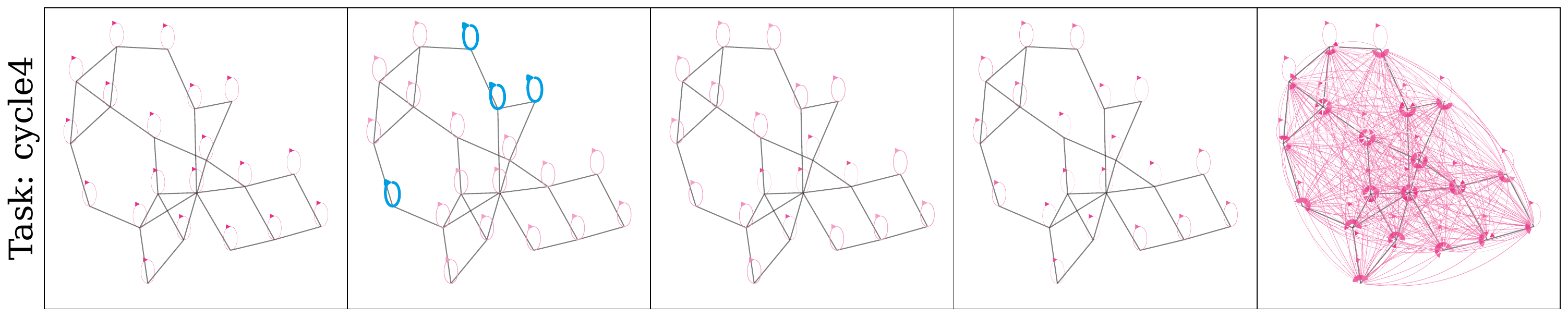}
    \includegraphics[scale=0.26]{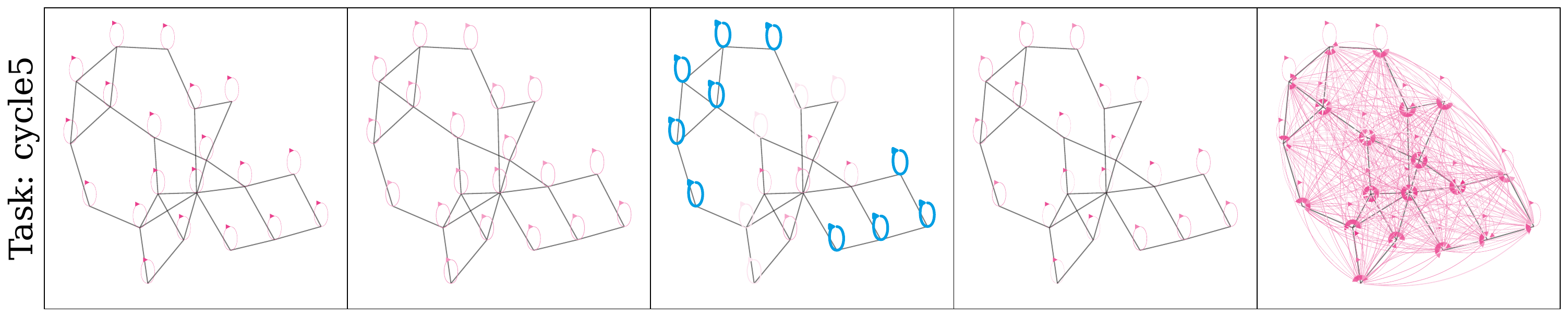}
    \includegraphics[scale=0.26]{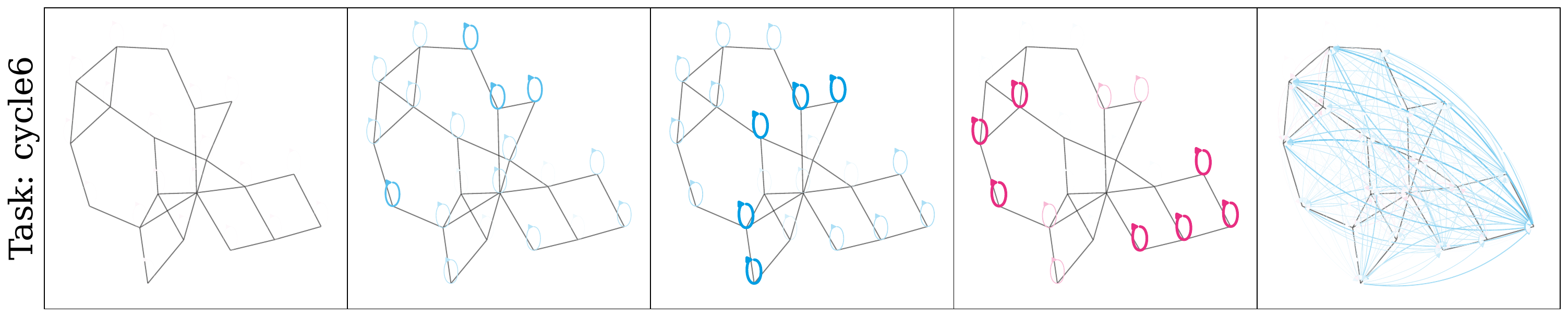}
    \includegraphics[scale=0.26]{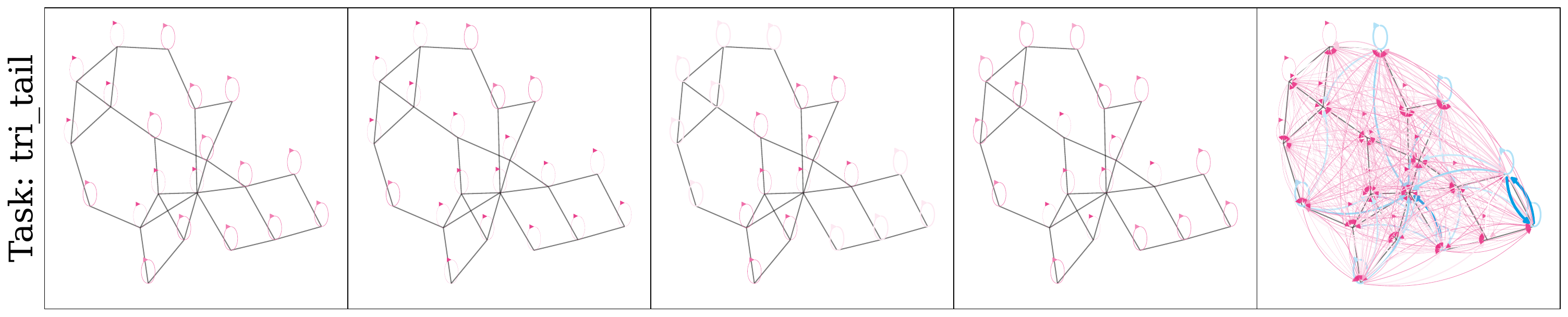}
    \includegraphics[scale=0.26]{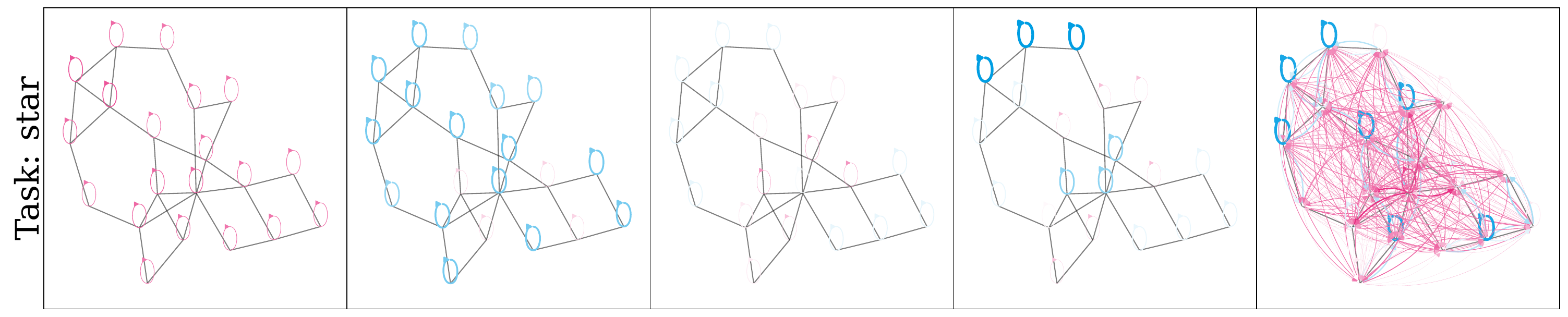}
    \includegraphics[scale=0.26]{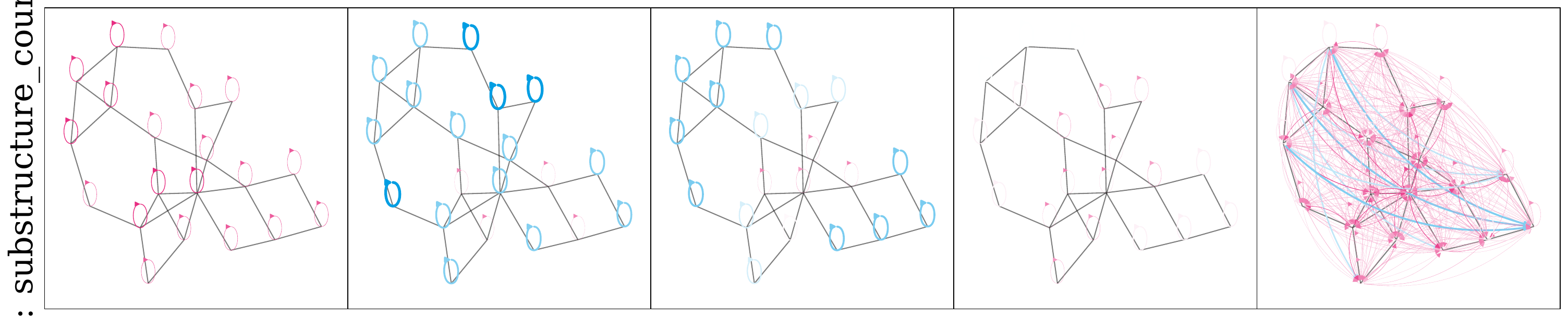}
    \caption{\label{fig:substructure-counting-heads}
    	Message passing within different heads of the ShareGNNs for the substructure counting task on one example graph from the test set. 
    	The tasks from top to bottom are \textit{triangle, cycle4, cycle5, cycle6, tailed triangle and star}.
    	The last row visualizes the ShareGNN (all) that is trained on all tasks in parallel.
    	Each column corresponds to one head of the invariant based encoder, where the heads base from left to right on the graph invariants \textit{3-cycles, 4-cycles, 5-cycles, 6-cycles and degrees}.
    	}
\end{figure*}

\subsection{Classification Benchmarks}\label{subsec:preprocessing-and-training-details}
In this section, we provide additional details about the preprocessing and training of the~\MyGNNs for the graph classification benchmark datasets.
\paragraph{Hyperparameter Search}
Each~\MyGNN is defined by the labeling functions for the encoder layers and the decoder.
As stated above we use no bias term for the encoder and only one encoder layer.
Thus, in fact, our hyperparameters are one labeling function for the encoder (for message passing)
and one for the decoder (for aggregation).
\Cref{tab:invariants-social-encoder} shows the hyperparameter grid used for our experiments on the real-world datasets, i.e.,
all labeling functions used for the encoder and the decoder.
We differentiate between the original atomic numbers, node degrees, the Weisfeiler-Leman algorithm with different
iteration depths, and subgraphs based on cycles and cliques of different lengths and sizes.
We test $10$ different options for the encoder and decoder layers for the molecules, and $22$ and $18$ options for the social datasets, respectively.
Thus, in total we test $100$ different hyperparameter configurations for the molecules and $396$ for the social datasets.
We did no hyperparameter search on the valid triples $\ruleset$ as we used a fixed $D$ to determine the maximum distance between node pairs
for which we compute shared weights.
For the molecules we used $D=6$, i.e., we compute shared message passing weights for all node pairs with distance at most $6$ and
for the social datasets we used $D=2$.
For the social datasets the maximum diameter is two, hence we define shared message passing weights for all node pairs
in the graph.
The best performing hyperparameter configuration for the \textit{fair} evaluation setup is shown in~\Cref{tab:details-best-runs}.
We note that our hyperparameter search is much faster than the one of the competitors in the \textit{fair} evaluation setup using
the implementation of~\citet{DBLP:conf/iclr/ErricaPBM20} and the same hardware, see also the note
by~\citet{DBLP:conf/iclr/ErricaPBM20} regarding the runtime of their experiments.

\paragraph{Preprocessing}
We need to precompute the node labels in advance for all labeling functions used in the hyperparameter search.
Moreover, we need the pairwise distances between the nodes of the graphs to compute all the valid triples.
\Cref{tab:preprocessing-times} shows the runtimes for the preprocessing of the datasets.
All pairwise distances can be computed in less than $10$ minutes, noting that we did not even pay attention to an efficient implementation (using networkx) as for these small datasets there was no need to do it.
Except for the IMDB datasets the preprocessing time to compute the labels
for all labeling functions for the hyperparameter search is negligible,
again noting that we did not pay attention to an efficient implementation.

\paragraph{The Best Hyperparameters}
\Cref{tab:details-best-runs} provides the details for the best hyperparameter configuration per dataset for the fair evaluation setup.
The first column shows the average epoch in which the best validation accuracy was achieved, i.e., the epoch for which the test accuracy was calculated.
It gives an indication of the model's convergence speed.
In particular, for all real-world datasets the best validation accuracy was achieved on average after less than $50$ epochs.
This shows that our approach is very efficient and converges quickly.
Recall, that the maximum number of epochs was set to $200$.
For NCI109, the best validation accuracy was already achieved on average after $3$ epochs.

The second column shows the average time per epoch.
The time per epoch depends on the number of parameters used in the model.
Indeed, for NCI109 the best hyperparameter configuration of our model
has approximately one million parameters and thus the time per epoch is higher than for the other datasets with
about one minute per epoch.
Particularly, except for the larger datasets NCI1, NCI109 and Mutagenicity the training time is less than $1$ second per epoch.
In general, runtime is not a problem for our approach as the preprocessing time is negligible and the training time is reasonable, compared to the competitors.
Note that our computations run in parallel, i.e., we are able to run all the three runs and $10$ folds in parallel on the same machine which produces some overhead but is more efficient than running the experiments sequentially.
Indeed, implementing a dynamic approach in PyTorch is also a very challenging task, see also the comments in~\citet{DBLP:journals/pami/HanHSYWW22}.
They mention that there is a gap between the theoretical and practical runtime of dynamic neural networks because
the implementation in PyTorch is not optimized for dynamic neural networks.
This is also the reason why our approach runs approximately three times faster on the CPU than on the GPU.

Regarding the \textit{Encoder Layers} in~\Cref{tab:details-best-runs} we see which encoder labeling functions performed best for the different datasets.
Notably, for the molecular datasets, Weisfeiler-Leman labeling with different iteration depths performed best.
For the social datasets, the best results were achieved with labeling functions based using a combination of degree labeling and counting of patterns (triangles, squares).

For the \textit{Decoder Layer} there is no uniform picture of which labeling functions performed best.
The total number of parameters depends strongly on the chosen labeling functions.
For the real-world datasets the number of parameters ranges from $10\,183$ in case of IMDB-BINARY to $925\,956$ in case of NCI109.

For the synthetic datasets we took only one hyperparameter configuration as
we wanted to demonstrate that our approach is able to
capture long-range dependencies and it is possible to restrict the message passing to a certain distance.
For the RingTransfer1 dataset we use the original node labels and define shared weights only for self-connections and node pairs at distance $50$.
For RingTransfer2 we use the original node labels and define shared weights only for node pairs with distance $8$.
For RingTransfer3 we use two encoder layers with the original node labels and define shared weights only for node pairs with distance $8$
in the first layer and for node pairs with distance $4$ in the second layer.
For CSL we use node labels induced by patterns consisting of
simple cycles up to length $10$ and define shared weights only for self-connections and node pairs with distance $1$.
For the Snowflakes dataset we use combined node labels induced by patterns consisting of
simple cycles of maximum length $10$ and the
original node labels.
We define shared weights only for self-connections and node pairs with distance $3$.
In this way the~\MyGNN is able to distinguish the graphs $M_0, M_1, M_2$ and $M_3$ that are not distinguishable by the 1-WL test.
For the output layer we used the Weisfeiler-Leman with $i=2$ iterations to collect the relevant information.

\begin{table}
	\centering
	\begin{tabular}{lrr}
\toprule
Dataset & Preprocessing Distances (s) & Preprocessing Labels (s) \\
\midrule
NCI1 & $33.6$ & $9.0$ \\
NCI109 & $33.2$ & $8.9$ \\
Mutagenicity & $36.0$ & $7.6$ \\
DHFR & $11.4$ & $1.8$ \\
\midrule
IMDB-BINARY & $4.9$ & $1007.7$ \\
IMDB-MULTI & $3.4$ & $1618.9$ \\
\midrule
ZINC ($12k$) &$94.4$&$26.1$ \\
ZINC ($250k$) &$1071.8$&$550.1$\\
\midrule
RingTransfer1 & $94.9$ & $0.0$ \\
RingTransfer2 & $2.6$ & $0.0$ \\
RingTransfer3 & $2.7$ & $0.0$ \\
CSL & $2.3$ & $5.0$ \\
Snowflakes & $125.8$ & $5.1$ \\
\midrule
Substructure Counting &$16.4$&$20.0$ \\
\bottomrule
\end{tabular}
    	\caption{\label{tab:preprocessing-times}
	Preprocessing times in seconds for the datasets used in the experiments.
	\textit{Preprocessing Distances} shows the time needed to compute all the pairwise distances between the nodes of the graphs,
	\textit{Preprocessing Labels} shows the time needed to compute the node labels for all node labeling functions used in the hyperparameter search.
	}
\end{table}

\begin{table*}
    \centering
	\begin{adjustbox}{width=1\linewidth}
\begin{tabular}{lcc|ccr|cr|r}
\toprule
Dataset & Best Epoch & Time per Epoch (s) & \multicolumn{3}{c}{Encoder Layers} & \multicolumn{2}{c}{Decoder Layer} & \# Total Weights \\
 & & & Labeling Function & Distances & \# Weights & Labeling Function & \# Weights & \\
\midrule
NCI1 & $6.4 \pm 5.1$ & $13.7 \pm 9.1$ & WL, Iterations 2 & 0, \ldots, 6 & 381\,145 & WL, Iterations 2 & 8\,118 & $389\,263$ \\
NCI109 & $3.0 \pm 5.0$ & $67.2 \pm 25.2$ & WL, Iterations 3 & 0, \ldots, 6 & 925\,878 & Atomic Numbers & 78 & $925\,956$ \\
Mutagenicity & $14.7 \pm 7.3$ & $2.1 \pm 0.5$ & WL, Iterations 1 & 0, \ldots, 6 & 28\,638 & Atomic Numbers & 30 & $28\,668$ \\
DHFR & $41.3 \pm 28.9$ & $0.4 \pm 0.1$ & WL, Iterations 2 & 0, \ldots, 6 & 53\,715 & Simple Cycles, Max.~Length 20 + Atomic Numbers & 292 & $54\,007$ \\
\midrule
IMDB-BINARY & $21.9 \pm 16.2$ & $0.4 \pm 0.1$ & Triangles, Squares + Node Degree & 0, 1, 2 & 19\,332 & Clique, Max.~Size 4 & 64 & $19\,396$ \\
IMDB-MULTI & $15.6 \pm 11.2$ & $0.7 \pm 0.2$ & Triangles, Squares + Node Degree & 0, 1, 2 & 10\,003 & Node Degree & 180 & $10\,183$ \\
\midrule
RingTransfer1 & $200.0 \pm 0.0$ & $0.6 \pm 0.1$ & Node Labels & 0, 50 & 18 & Node Labels & 18 & $36$ \\
RingTransfer2 & $200.0 \pm 0.0$ & $0.5 \pm 0.1$ & Node Labels & 8 & 240 & Node Labels & 34 & $274$ \\
\multirow{2}{*}{RingTransfer3} & \multirow{2}{*}{$152.5 \pm 56.8$} & \multirow{2}{*}{$0.7 \pm 0.2$} & \textit{Layer 1:} Node Labels & 8 & 240 & \multirow{2}{*}{Node Labels} & \multirow{2}{*}{68} & \multirow{2}{*}{$548$} \\
 &  & & \textit{Layer 2:} Node Labels & 4 & 240 &  &  &  \\
CSL & $200.0 \pm 0.0$ & $0.0 \pm 0.0$ & Simple Cycles, Max.~Length 10 & 0, 1 & 550 & Simple Cycles, Max.~Length 10 & 960 & $1\,510$ \\
Snowflakes & $71.8 \pm 44.7$ & $0.5 \pm 0.1$ & Simple Cycles, Max.~Length 10 + Node Labels & 0, 3 & 5\,000 & WL, Iterations 2 & 36 & $5\,036$ \\
\bottomrule
\end{tabular}

\end{adjustbox}
\caption{\label{tab:details-best-runs}
	Details of the best performing hyperparameter configurations for~\MyGNNs on the datasets used in the experiments.
	\textit{Best Epoch} denotes the average epoch where the highest validation accuracy is reached,
	\textit{Time per Epoch} is the average time per epoch in seconds,
	\textit{Encoder Layers} shows the \textit{Labeling Function} and \textit{Distances} between node pairs used to define the shared weights. \textit{\#Weights} is
	the resulting number of weights.
	\textit{Decoder Layer} shows the \textit{Labeling Function} used to define the shared weights and \textit{\#Weights} is the resulting number of weights.
	\textit{\#Total Weights} is the total number of weights in the model.
	}
\end{table*}

\subsection{Regression Benchmark (ZINC)}\label{subsec:regression}

For the ZINC benchmark, we conducted an extensive hyperparameter and architecture search over
different invariant labeling functions, number of heads, maximum message-passing distances,
and pooling strategies.
The final configuration was selected based on the best validation
performance.
\Cref{tab:zinc-arch} summarizes the resulting architecture, and
\Cref{tab:zinc-hparams} lists the corresponding training hyperparameters.

The model consists of a single linear projection layer followed by an
\emph{invariant-based convolution layer} with multiple parallel heads.
Each head corresponds to a specific structural invariant: induced/simple cycles
(lengths 6–10), atomic numbers, and Weisfeiler–Leman labels with and without considering edge labels, all combined with distance features up to 23 hops (cycles and atomic numbers 23 hops and WL 2 hops) 
These heads enable parameter sharing across nodes and node pairs with the same
invariant signature, allowing information to be propagated over the entire
molecule in a single message-passing layer.
The encoder heads are aggregated into a $100$ dimensional representation for each node.
A multi-head invariant pooling layer aggregates the resulting node embeddings
into a graph representation, which is processed by a compact MLP ($2$ layer) readout.
The invariant pooling layer uses the same invariants as the encoder layer but $10$ heads per invariant instead of $2$ heads (in the encoder).

All experiments are trained with MAE loss and the Adam optimizer, using learning
rate scheduling and weight initialization as specified in
\Cref{tab:zinc-hparams}.
The detailed layer and head definitions
are provided in the configuration files \texttt{network\_ZINC.yml} and \texttt{network\_ZINC\_full.yml}.

\begin{table*}[t]
\centering

\begin{tabular}{p{0.42\linewidth} p{0.5\linewidth}}
\toprule
\textbf{Layer} & \textbf{Configuration} \\
\midrule
Linear Layer & one-hot $\rightarrow$ $10$ output features \\
Layer Norm & \\
Invariant-Based Message-Passing Layer & (each head twice, biases: \textit{atomic numbers}) \\
 & \textit{induced cycles} (lengths: $6$ to $10$, $d=1,\ldots,22$)\\
 & \textit{simple cycles} (lengths: $6$ to $10$, $d=1,\ldots,22$)\\
 & \textit{atomic numbers} ($d=1,\ldots,22$)\\
 & \textit{atomic numbers} ($d=1$, edge labels)\\
 & \textit{degree + atomic numbers} ($d=1,\ldots,4$)\\
 & \textit{degree + atomic numbers + edge labels} ($d=1,\ldots,4$)\\
& \textit{WL, depth: 1 + edge labels} ($d=1,\ldots,2$)\\
Multi-Head Aggregation  & $100$ output features \\
Layer Norm & \\
Invariant-Based Pooling Layer & (same invariants but $10$ heads per invariant) \\
Readout MLP & 100 $\rightarrow$ 100 $\rightarrow$ 1 \\
\bottomrule
\end{tabular}
\caption{
\label{tab:zinc-arch}Best ShareGNN architecture for the ZINC dataset.
Invariants in italic. All layers except for the last one have LeakyRelu activations.}
\end{table*}

\begin{table}[t]
\centering
\begin{tabular}{p{0.42\linewidth} p{0.5\linewidth}}
\toprule
\textbf{Hyperparameter} & \textbf{Value} \\
\midrule
Loss & Mean Absolute Error (MAE) \\
Optimizer & Adam + ReduceLROnPlateau (factor 0.5, patience 5) \\
Learning rate & 0.001 (min: 0.0001)\\
Batch size & 128 \\
Epochs & 150 \\
Weight \& bias initialization & Uniform $[-0.001, 0.001]$ \\
Input features & One-hot node labels \\
Precision & double \\
\bottomrule
\end{tabular}

\caption{
\label{tab:zinc-hparams}Training hyperparameters for ZINC regression.}
\end{table}

\begin{figure*}
	\includegraphics[width=\linewidth]{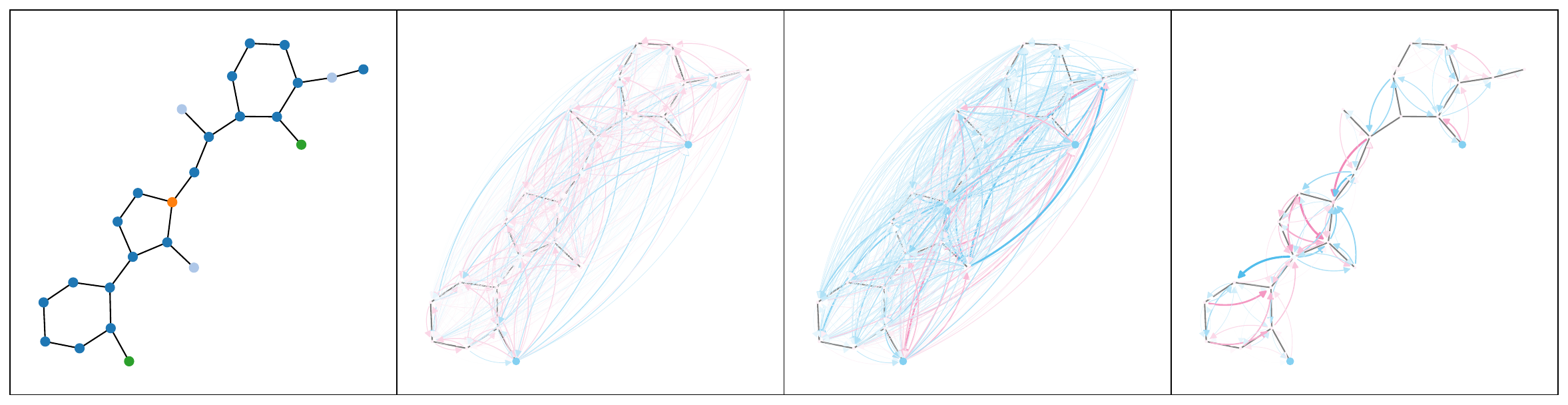}
    \caption{\label{fig:zinc-message-passing}
    Message passing within different heads of the ShareGNN for the ZINC dataset on one example graph.
    The leftmost column shows the original graph with atomic number node labels.
    The other columns show the message passing weights for different heads of the invariant-based encoder.
    The parameters per head are from left to right:
    \textit{(1) induced cycles (lengths 6), $d>1$},
    \textit{(2) atomic numbers, $d>1$},
    \textit{(3) $1$-WL labels at iteration $1$ (using edge and atomic numbers as initial label), $d\in\{1,2\}$}.
    }

\end{figure*}

\section{Ablation Study \& Additional Results}\label{sec:additional-results}
We provide additional results, including visualizations
of the message-passing weights.
In particular, we study the effect of restricting the number of weights as well as considering multiple encoder layers
and different maximum distances for the node pairs.

\subsection{Results for the Standard Evaluation}\label{subsec:standard-evaluation-results}
We already presented the results of our~\MyGNNs regarding the fair evaluation protocol in \Cref{tab:real-world}.
Here, we provide the results for the standard evaluation protocol in \Cref{tab:real-world-sota}.
We use the same splits as provided by~\citet{DBLP:conf/iclr/XuHLJ19}.
In the \textit{standard} setup, we compare to methods that explicitly enhance expressiveness through substructures or topological information, such as
CIN~\citep{DBLP:conf/nips/BodnarFOWLMB21},
SIN~\citep{DBLP:conf/icml/BodnarF0OMLB21},
GSN~\citep{DBLP:journals/pami/BouritsasFZB23},
PIN~\citep{DBLP:conf/aaai/Truong024}.

\paragraph{\textit{Standard} Evaluation}
For comparison with~\citet{DBLP:conf/iclr/XuHLJ19,DBLP:conf/nips/BodnarFOWLMB21,DBLP:journals/pami/BouritsasFZB23} we also adopt the widely used standard protocol.
Each dataset is split into $10$ predefined train/test folds.
Models are trained on the training folds with a grid of hyperparameter settings, and the configuration with the best average test accuracy across folds is used for final reporting (\Cref{tab:real-world-sota}).

\paragraph{Results}
In the standard evaluation (\Cref{tab:real-world-sota}), the performance matches or surpasses current state-of-the-art methods by noting that GSN, the best performing model on the social datasets, uses similar structural information.
The ability to perform well in both protocols highlights the model’s strengths and stability across different training setups.
A variant using Gaussian noise added to the constant node features often shows slightly improved generalization, emphasizing that the model relies entirely on structure rather than node features.

\begin{table}\centering
\begin{tabular}{lcc|cc}
		\toprule
        & NCI1 & NCI109  & IMDB-B& IMDB-M\\
		\midrule
		NoG  & $63.4 \pm 2.5$ & $61.6 \pm 2.7$ & $71.2 \pm 5.6$ & $47.1 \pm 3.3$\\
		WL & \SecondColor{85.4 \pm 2.3} & \SecondColor{85.5 \pm 1.6} & $74.3 \pm 3.5$ & $51.5 \pm 3.9$\\
		\hline
		CIN &$83.6\pm 1.4$&\ThirdColor{84.0\pm 1.6}&$75.6\pm 3.7$&\ThirdColor{52.7\pm 3.1}\\
		SIN & $82.7\pm 2.1$& - & $75.6\pm 3.2$ & $52.4\pm 2.9$\\
		GSN &$83.5\pm 2.3$&-&\FirstColor{77.8\pm 3.3}&\FirstColor{54.3\pm 3.3}\\
		PIN & \ThirdColor{85.1\pm 1.5}& \ThirdColor{84.0\pm 1.5} & \ThirdColor{76.6\pm 2.9} & -\\
		\midrule
		\textbf{ours} & \FirstColor{86.1 \pm 2.4} & \FirstColor{86.8 \pm 1.6} & \SecondColor{77.7 \pm 2.8} & \SecondColor{53.1 \pm 4.0}\\
		\bottomrule
	\end{tabular}
		\caption{\label{tab:real-world-sota}
	\textit{Standard} eval. (Acc. in \%).
		The best results are highlighted by \FirstColor{\textbf{First}}, \SecondColor{\textbf{Second}} and \ThirdColor{\textbf{Third}}.
    }
    \end{table}
    
    \paragraph{Runtime} We compared the runtime of~\MyGNN and that of classical MPNNs showing that~\MyGNN is competitive with GIN and faster than GAT/GATv2, while being slower than GraphSAGE but at the same time converging in fewer epochs.
    
• \begin{table}[t]
	\centering
	\begin{tabular}{lrrrrrrrr}
		\toprule
		& \multicolumn{2}{c}{NCI1} & \multicolumn{2}{c}{NCI109} & \multicolumn{2}{c}
		{Mutagenicity} & \multicolumn{2}{c}{DHFR} \\
		\midrule
  GraphSAGE & $\mathbf{1.1}$ & $195$ & $\mathbf{1.1}$ & $192$ & $\mathbf{1.3}$ &
$180$ & $\mathbf{0.3}$ & $184$ \\
GIN       & $2.9$ & $145$ & $2.8$ & $170$ & $2.9$ & $193$ & $0.8$ & $131$ \\
GAT       & $19.6$ & $447$ & $10.0$ & $263$ & $10.3$ & $295$ & $2.8$ & $183$ \\
GATv2     & $30.5$ & $183$ & $25.5$ & $177$ & $16.3$ & $198$ & $7.9$ & $94$ \\
ShareGNN  & $13.7$ & $\mathbf{6}$ & $67.2$ & $\mathbf{3}$ & $2.1$ & $
\mathbf{15}$ & $0.4$ & $\mathbf{41}$ \\

		\bottomrule
	\end{tabular}

    	\caption{Runtime comparison betwenn~\MyGNN and classical MPNNs showing the runtime per epoch in seconds (first column) and the convergence speed, i.e., the best epoch (second column).}
    \end{table}

\subsection{Influence of Information Encoding}\label{subsec:hyperparameter-study}
Our approach encodes additional structural information via invariance-based weight sharing, i.e.,
all information is encoded in the message-passing weights of the encoder layer and the weights of the decoder layer.
What should be analyzed is whether the improved performance of our~\MyGNNs is solely due to the additional information used
or whether the structural encoding of information is key to the performance.
The following experiment shows that the additional information is not solely responsible for our improved performance
but how the information is encoded matters.
To do so, we collect the node labels of the best run of~\MyGNNs and use them as additional input features for the competitors (\Cref{tab:results-with-features}).

Surprisingly, only GIN (on IMDB) and GraphSAGE (on DHFR) show a slight improvement
compared to their original results.
The overall performance is still worse than that of~\MyGNN.
We argue that the structural encoding of information is key to GNNs' performance and cannot be replaced by additional input features.
Second, we vary the number of encoder layers and the maximum message-passing distance $D$ between two nodes.
\Cref{fig:ablation-c} shows that with increasing distance $D$ the performance on NCI1 slightly improves from $82.9\%$ ($D=1$)
to $85.6\%$ $(D=12)$ for a single encoder layer.
In fact,~\textit{a single} message-passing layer is sufficient, as it already yields the best performance.
This shows that capturing long-range interactions in a single layer is better than capturing them in multiple layers.
Third, we analyze the effect of valid triples $\ruleset$ on the model performance.
In particular, we reduce the number of shared message passing weights as follows:
We consider only such shared weights for message passing that appear at least $x$-times in the graph dataset, i.e., we ignore rare relations.
\Cref{fig:ablation-a,fig:ablation-b} show the results for NCI1 and IMDB-B, respectively.
Each $x$-axis value represents the model where all shared weights for message passing that occur at most $(x-1)$-times in the respective graph dataset are removed.
The performance on NCI1 is relatively stable up to an $x$-value of $8$, i.e., the model
considers only shared weights that occur at least $9$ times in the NCI1 dataset.
For IMDB-BINARY we see no significant drop in performance up to an $x$-value of $20$.
Thus, for both datasets, we can reduce the number of weights by more than $80\%$ without any significant performance loss.
We provide more details on the above ablation study for all the datasets in~\Cref{sec:additional-results}.

\subsection{Weight Analysis of the Encoder}\label{subsec:encoder-layer-analysis}
\Cref{fig:occurences_per_rule} shows the number of occurrences of the shared weights of the encoder layer summed over all graphs for the respective dataset.
The underlying hyperparameter configuration, i.e., labeling function is the one with the best performance for the respective dataset, see \Cref{tab:details-best-runs}.
Each $x$-axis value corresponds to a weight of the encoder layer sorted by the number of occurrences.
The $y$-axis shows the number of occurrences of the respective weight.
The ten vertical dotted lines indicate the distribution of the number of occurrences.
All weights in the plots that are right of the $n$-th vertical line (counting from the right) are present less than $(n+1)$-times in the dataset,
e.g., weights left of the leftmost vertical line appear more than 10 times in the dataset.
The different colors denote the node pair distance the respective weight is associated with.
We observe that for NCI and NCI109, more than half of the weights only occur once in the whole dataset.
In fact, for all datasets only a small fraction of the weights occur more than 10 times.
\Cref{fig:weights_per_rule} shows the final values of the trained encoder layer weights of the respective datasets.
For all datasets we observe a gap in the distribution of the weights, which is best visible for the NCI109 dataset.
We explain this behavior as follows:
The message-passing weights are all initialized with the same constant value of $0.001$.
All weights are updated during training, except for the ones that are associated with node pairs that are only present in the test datasets.
These weights are not updated and remain at their initial value.
Weights that occur only a few times or even only once in the dataset are also less likely to be updated during training.
Surprisingly, it seems that if the weights are updated, they are all updated above or below a certain threshold which
results in the above-mentioned gap in the distribution of the weights.
What is also surprising is that the weight values seem to be distributed symmetrically around zero, except of the IMDB-BINARY dataset.

\subsection{Multi-Layer~\MyGNNs}\label{subsec:ablation-encoder-layers}
We study the effect of using multiple message-passing layers on the performance of the model as well as changing the maximum distance between two nodes for which
message-passing weights are learned.
\Cref{fig:ablation_distance} shows the mean accuracy
for~\MyGNN
for different numbers of encoder layers  ($x$-axis) and different maximum distances between two nodes for which message-passing weights are learned ($y$-axis) for the NCI109 and DHFR datasets.
Note that the models are initialized with the best hyperparameters for the respective dataset.
We observe that including more distant node pairs for which message-passing weights are learned does not lead to a decrease in performance.
Thus, long-range interactions in a single message-passing layer does not reduce performance.
In contrast, if capturing long-range interactions with more and more encoder layers, the performance decreases.
Indeed, the model is unable to fit the data with more than 4 encoder layers.
This shows that capturing long-range interactions in a single layer surpasses the performance of capturing them in multiple layers.

\subsection{Reduction of the Message-Passing Weights}\label{subsec:weight-restiction}
The above observations motivate the following ablation study.
The idea is to restrict the number of weights
by considering only the most frequent weights that appear in the graph dataset~\Cref{fig:ablation_threshold_lower}, and only
the most infrequent weights that appear in the graph dataset~\Cref{fig:ablation_threshold_upper}.
Recall, that each shared weight in the encoder layer is associated with a node pair at a specific distance.
Thus, we can pre-compute the number of occurrences of each weight and sort them accordingly for each dataset.
\Cref{fig:ablation_threshold_lower} shows the mean accuracies with standard deviation on the NCI109, DHFR and IMDB-MULTI datasets for models with different numbers of weights.
Each $x$-axis value corresponds to a model that uses only such shared weights for message passing that occur at least $x$ times in the dataset.
For DHFR and IMDB-MULTI, we observe that the performance only slightly decreases when deleting more and more of the
less frequent weights.
In this way we can reduce the number of weights by approximately 80\% without a significant loss in performance.
For the NCI109 dataset, the performance decreases more rapidly, likely because the number of parameters also drops quickly.
\Cref{fig:ablation_threshold_upper}, where
each $x$-axis value corresponds to a model that uses only such shared weights for message passing that occur at most $x$ times in the dataset,
shows the opposite case.
That is, we reduce the number of weights by removing the most frequent ones.
We observe that the performance increases when adding more and more of the less frequent weights.
However, the performance is still lower than the one of the models that use the most frequent weights despite the fact that the total number of weights is higher.
We conclude that infrequent weights are less important for the graph classification task than the frequent ones.

\Cref{fig:weights_per_rule_lower_10} visualizes the same as~\Cref{fig:weights_per_rule} but this time for the models that are trained with message-passing weights that occur at least 10 times in the dataset.
We see only one vertical dotted line in the plots.
Weights to the right of the vertical line appear exactly 10 times; those to the left appear more than 10 times.
Moreover, by comparing the $x$-axes of~\Cref{fig:weights_per_rule} and~\Cref{fig:weights_per_rule_lower_10} we observe
we train on a significantly smaller number of weights.
Interestingly, the gap in the distribution of the weights values observed in~\Cref{fig:weights_per_rule} has disappeared.
This suggests that all considered weights are often enough updated during training.
Another surprising observation is that in case of IMDB-BINARY, the weights seem to
take certain levels.
We can only speculate that these weights are associated with node pairs that are present in only one or a few graphs in the dataset, and thus are updated very infrequently during training.
In fact, this would mean that they might not be that important for the graph classification task.
An exact explanation for this behavior needs a deeper analysis of these weights and hence the corresponding node pairs.
Despite the fact that we only consider weights that occur at least 10 times in the dataset,
lots of weights in the IMDB-BINARY dataset seem to only occur once in the test dataset (bar around zero).

\subsection{Visualization of the Message-Passing Weights}\label{subsec:visualization}
Our method allows for a very natural visualization of the learned message-passing weights and hence we
expect our model to provide helpful insights regarding interpretability and explainability of graph classification tasks.
Here are a few examples to illustrate this point.
\Cref{fig:interpretability-1} shows the learned weights of the encoder layers of our~\MyGNN
that achieves $100\%$ accuracy on the RingTransfer2 and $99.4\%$ accuracy on the RingTransfer3 dataset, respectively.
Indeed, it is easy to observe that we restricted to such node pairs that are at distance $8$ for RingTransfer2~(\ref{fig:ring-transfer-2-messages}).
For RingTransfer3~(\ref{fig:ring-transfer-3-messages}), we see that the weights in the first layer are restricted to node pairs at distance $8$
and in the second layer to node pairs at distance $4$.
Moreover, for RingTransfer3, we have rotated the graphs such that the node with the label $0$ is at rightmost position for all three graphs.
Recall, that the node with label $0$ together with the labels of the nodes at distances $4$ and $8$ determine the label of the graph.
The columns \textit{Layer $1$ Top $3$} and \textit{Layer $2$ Top $3$} in~\Cref{fig:ring-transfer-3-messages} show the top three absolute weights
of the respective encoder layers.
In particular, we see how the model aggregates the necessary information from the nodes at distances $8$ in the first layer,
i.e., the largest weights are assigned to the node pairs containing the node with label $0$.
In the second layer, the model aggregates the information from the nodes at distances $4$ and assigns the largest weights to the node pairs containing the node with label $0$.
\Cref{fig:interpretability-2} shows the learned weights of the encoder layers of our~\MyGNN
for the CSL and Snowflakes datasets.
In case of Snowflakes, comparing the columns \textit{Node Labels} and \textit{Top $3$ Weights} in~\Cref{fig:snowflakes-messages}
we observe that the model learns to aggregate the relevant information from the subgraphs $M_0,M_1,M_2$ and $M_3$
attached to the node with label $1$.
This validates that the model learns to look at the relevant subgraphs to classify the Snowflake graphs.

\begin{figure*}[t]
	\begin{subfigure}{0.325\textwidth}\centering
		{\includegraphics[width=\textwidth]{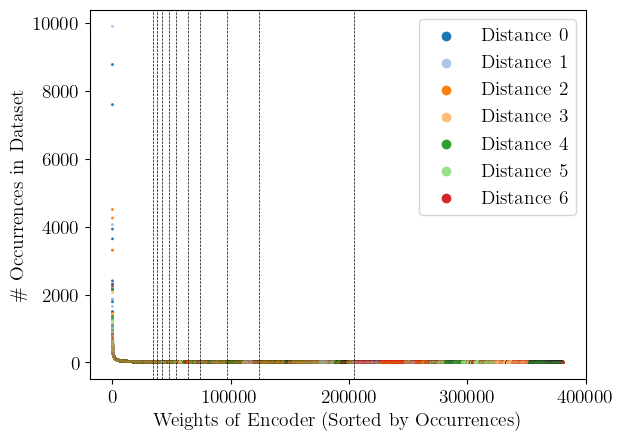}}
		\caption{NCI1\label{fig:occurences_per_rule-NCI1}}
	\end{subfigure}
	\begin{subfigure}{0.325\textwidth}\centering
		{\includegraphics[width=\textwidth]{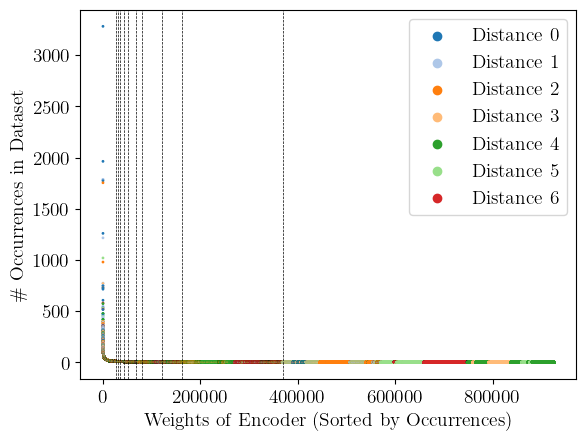}}
		\caption{NCI109\label{fig:occurences_per_rule-NCI109}}
	\end{subfigure}
	\begin{subfigure}{0.325\textwidth}\centering
		{\includegraphics[width=\textwidth]{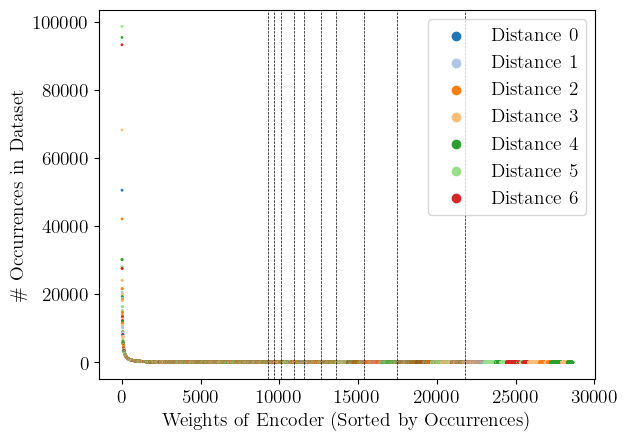}}
		\caption{Mutagenicity\label{fig:occurences_per_rule-Mutagenicity}}
	\end{subfigure}
	\begin{subfigure}{0.325\textwidth}\centering
		{\includegraphics[width=\textwidth]{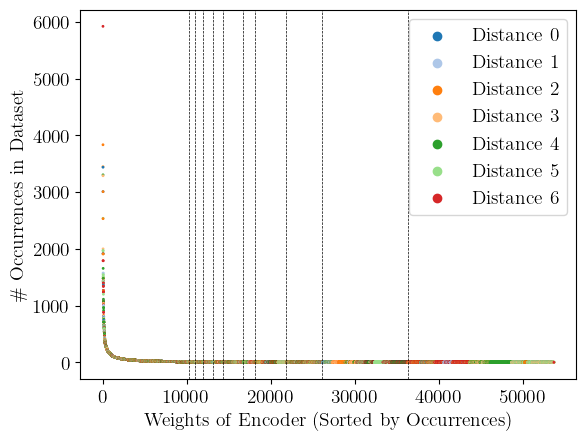}}
		\caption{DHFR\label{fig:occurences_per_rule-DHFR}}
	\end{subfigure}
	\begin{subfigure}{0.325\textwidth}\centering
		{\includegraphics[width=\textwidth]{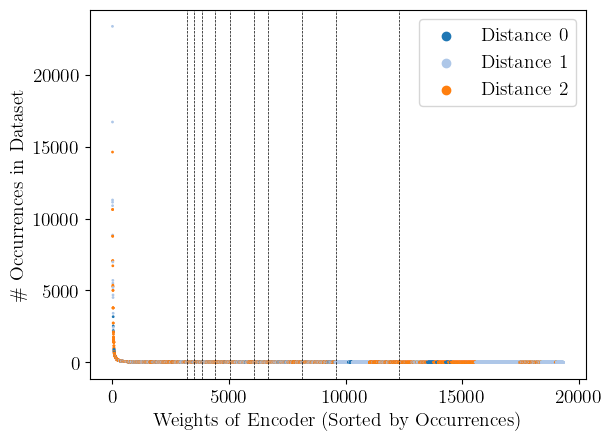}}
		\caption{IMDB-BINARY\label{fig:occurences_per_rule-IMDB-BINARY}}
	\end{subfigure}
	\begin{subfigure}{0.325\textwidth}\centering
		{\includegraphics[width=\textwidth]{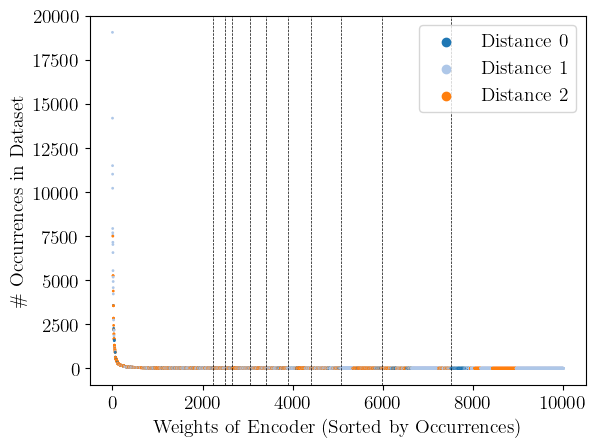}}
		\caption{IMDB-MULTI\label{fig:occurences_per_rule-IMDB-MULTI}}
	\end{subfigure}
	\caption{\label{fig:occurences_per_rule}
	Number of occurrences of the message-passing weights of the encoder layer summed over all graphs for the respective dataset.
	The model is initialized with the best hyperparameters for the respective dataset.
	The ten vertical dotted line indicate the distribution of the number of occurrences.
	All weights in the plots that are right of the $n$-th vertical line (counting from the right) are present less than $(n+1)$-times in the dataset,
	i.e., weights left of the leftmost vertical line appear more than 10 times in the dataset.
	The different colors denote the node pair distance the respective weight is associated with.
	}
\end{figure*}

\begin{figure*}[t]
		\begin{subfigure}{0.325\textwidth}\centering
		{\includegraphics[width=\textwidth]{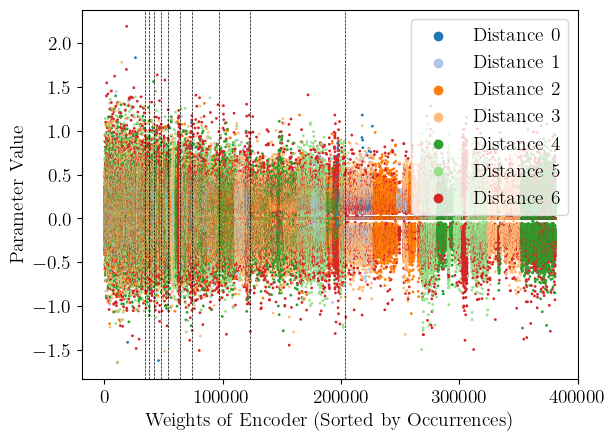}}
		\caption{NCI1\label{fig:weights_per_rule-NCI1}}
	\end{subfigure}
	\begin{subfigure}{0.325\textwidth}\centering
		{\includegraphics[width=\textwidth]{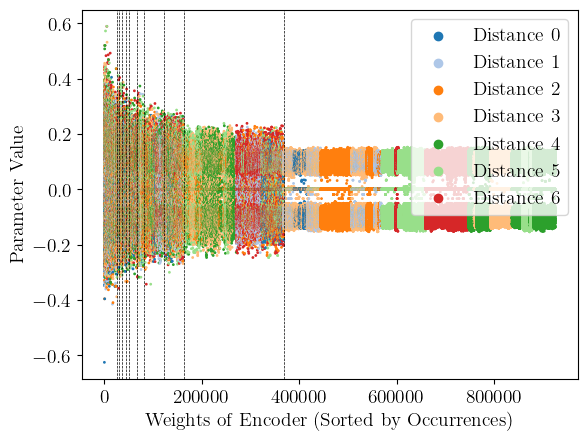}}
		\caption{NCI109\label{fig:weights_per_rule-NCI109}}
	\end{subfigure}
	\begin{subfigure}{0.325\textwidth}\centering
		{\includegraphics[width=\textwidth]{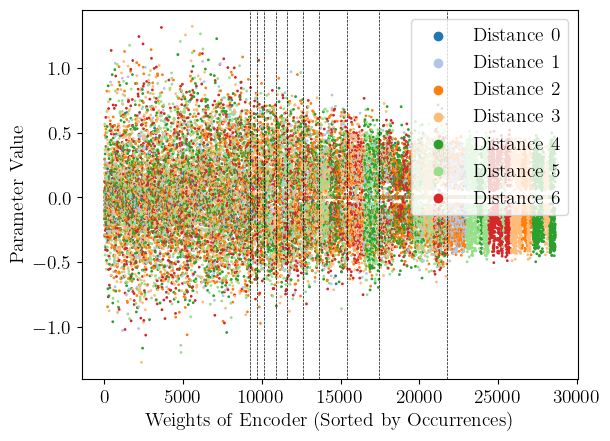}}
		\caption{Mutagenicity\label{fig:weights_per_rule-Mutagenicity}}
	\end{subfigure}
	\begin{subfigure}{0.325\textwidth}\centering
		{\includegraphics[width=\textwidth]{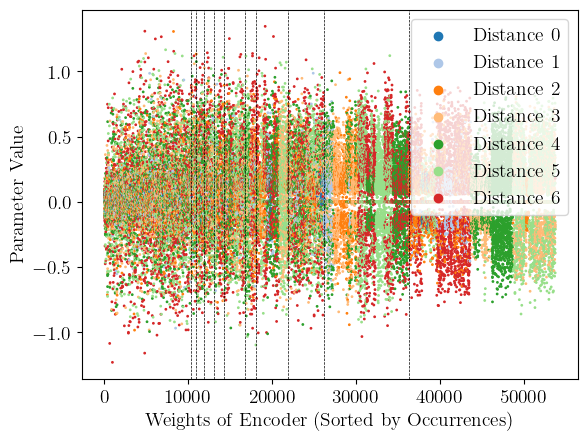}}
		\caption{DHFR\label{fig:weights_per_rule-DHFR}}
	\end{subfigure}
	\begin{subfigure}{0.325\textwidth}\centering
		{\includegraphics[width=\textwidth]{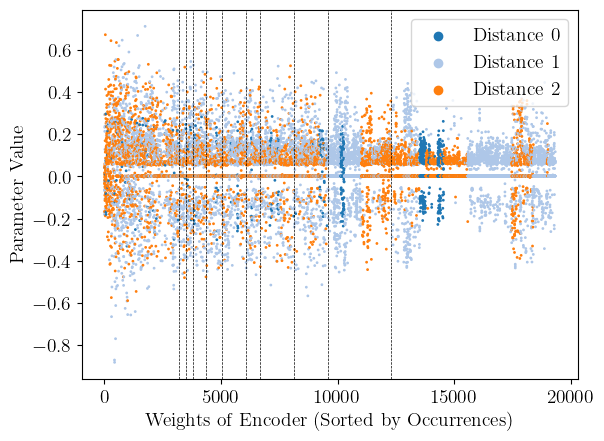}}
		\caption{IMDB-BINARY\label{fig:weights_per_rule-IMDB-BINARY}}
	\end{subfigure}
	\begin{subfigure}{0.325\textwidth}\centering
		{\includegraphics[width=\textwidth]{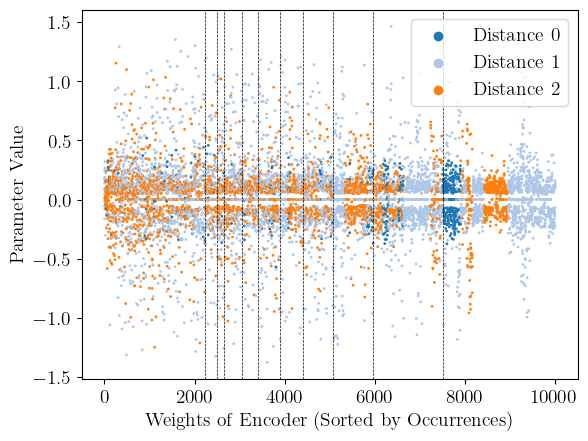}}
		\caption{IMDB-MULTI\label{fig:weights_per_rule-IMDB-MULTI}}
	\end{subfigure}
	\caption{\label{fig:weights_per_rule}
	Values of the trained message-passing weights of the respective datasets.
	The weights are sorted by the number of occurrences in the dataset.
	The model is initialized with the best hyperparameters for the respective dataset.
	The ten vertical dotted lines indicate the number of occurrences.
	All weights in the plots that are right of the $n$-th vertical line (counting from the right) are present less than $(n+1)$-times in the dataset,
	i.e., weights left of the leftmost vertical line appear more than 10 times in the dataset.
	The different colors denote the node pair distance the respective weight is associated with.
	}
\end{figure*}

\begin{figure*}[t]
		\begin{subfigure}{0.325\textwidth}\centering
		{\includegraphics[width=\textwidth]{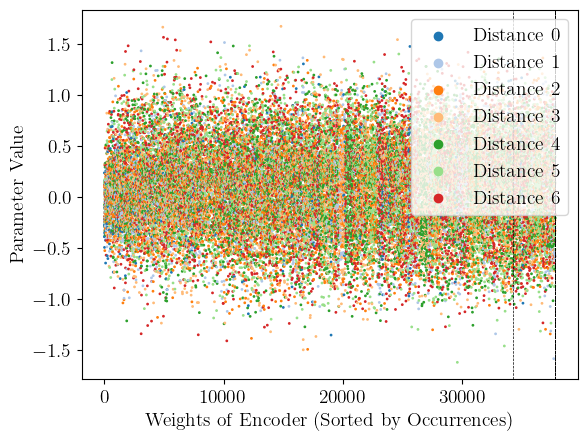}}
		\caption{NCI1\label{fig:weights_per_rule-NCI1_lower_10}}
	\end{subfigure}
	\begin{subfigure}{0.325\textwidth}\centering
		{\includegraphics[width=\textwidth]{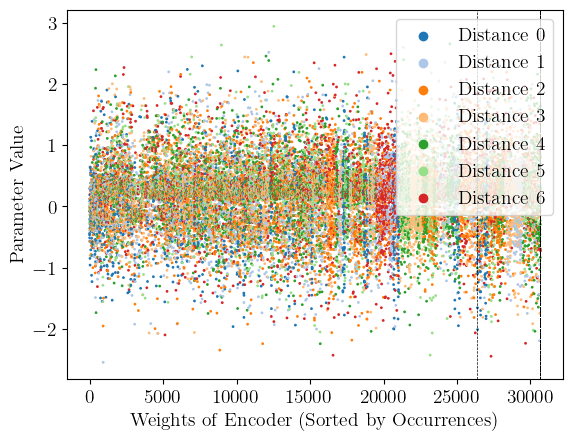}}
		\caption{NCI109\label{fig:weights_per_rule-NCI109_lower_10}}
	\end{subfigure}
	\begin{subfigure}{0.325\textwidth}\centering
		{\includegraphics[width=\textwidth]{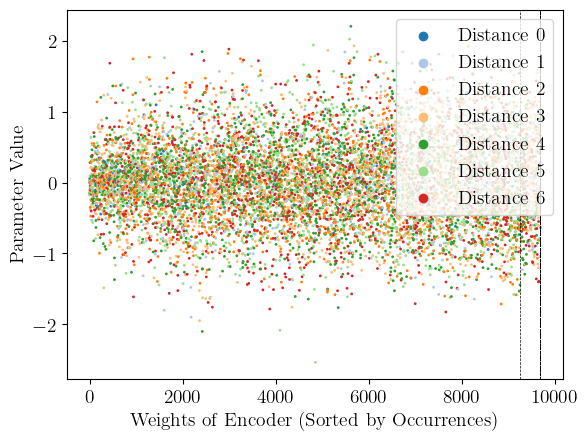}}
		\caption{Mutagenicity\label{fig:weights_per_rule-Mutagenicity_lower_10}}
	\end{subfigure}
	\begin{subfigure}{0.325\textwidth}\centering
		{\includegraphics[width=\textwidth]{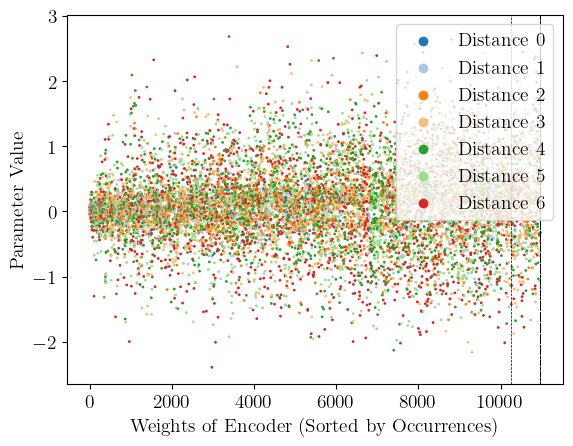}}
		\caption{DHFR\label{fig:weights_per_rule-DHFR_lower_10}}
	\end{subfigure}
	\begin{subfigure}{0.325\textwidth}\centering
		{\includegraphics[width=\textwidth]{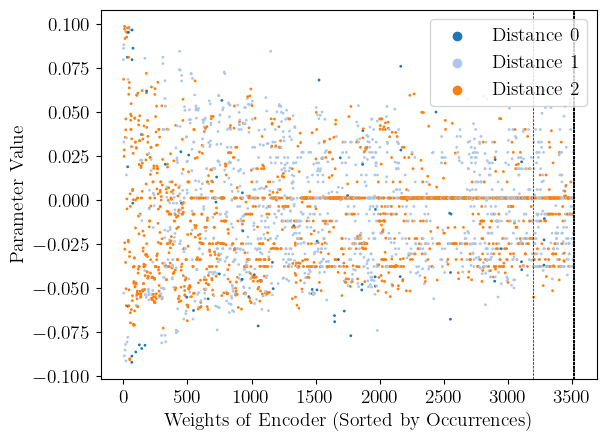}}
		\caption{IMDB-BINARY\label{fig:weights_per_rule-IMDB-BINARY_lower_10}}
	\end{subfigure}
	\begin{subfigure}{0.325\textwidth}\centering
		{\includegraphics[width=\textwidth]{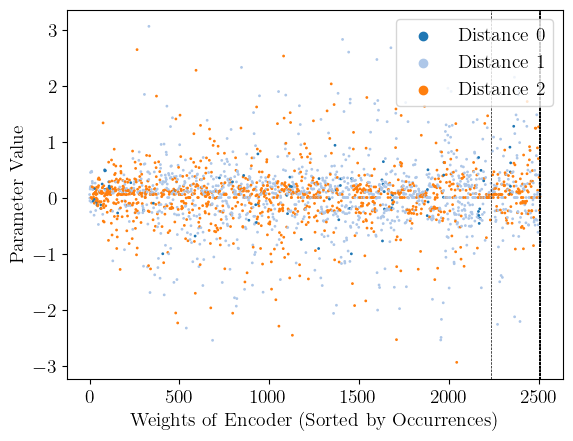}}
		\caption{IMDB-MULTI\label{fig:weights_per_rule-IMDB-MULTI_lower_10}}
	\end{subfigure}
	\caption{\label{fig:weights_per_rule_lower_10}
See~\Cref{fig:weights_per_rule}, but this time the model is trained only with message-passing weights that occur at least 10 times in the dataset.
	}
\end{figure*}

\begin{figure*}[t]
		\begin{subfigure}{0.49\textwidth}\centering
		{\includegraphics[width=\textwidth]{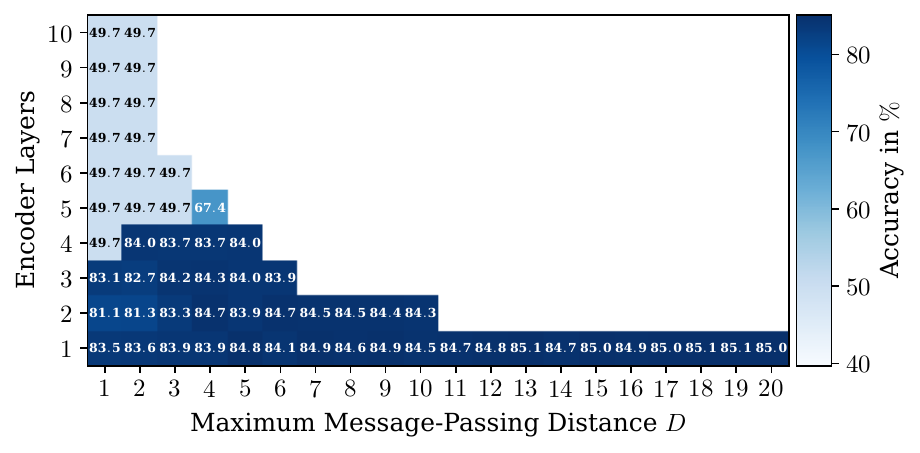}}
		\caption{NCI109\label{fig:ablation_distance-NCI109-20}}
	\end{subfigure}
	\begin{subfigure}{0.49\textwidth}\centering
		{\includegraphics[width=\textwidth]{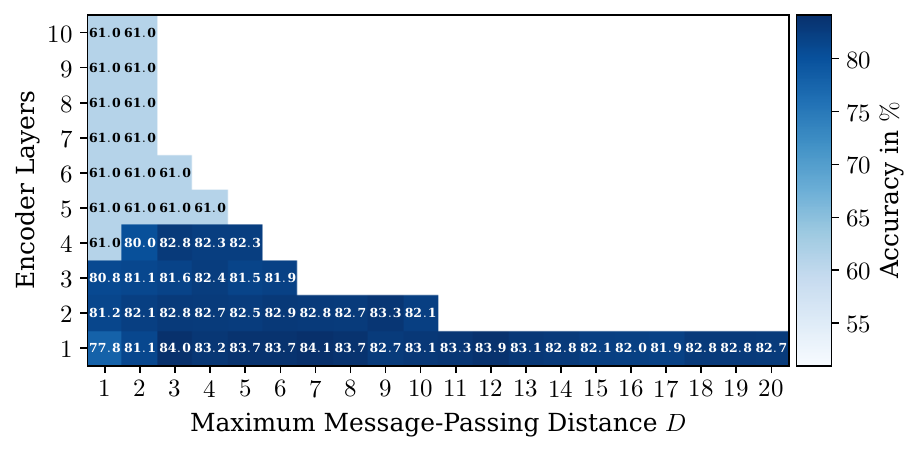}}
		\caption{DHFR\label{fig:ablation_distance-DHFR-20}}
	\end{subfigure}
		\caption{\label{fig:ablation_distance}
\MyGNN performance for different maximum message-passing distances $D$ ($x$-axis) and different numbers of encoder layers ($y$-axis).
		}
		\end{figure*}

\begin{figure*}[t]
	\begin{subfigure}{0.325\textwidth}\centering
		{\includegraphics[width=\textwidth]{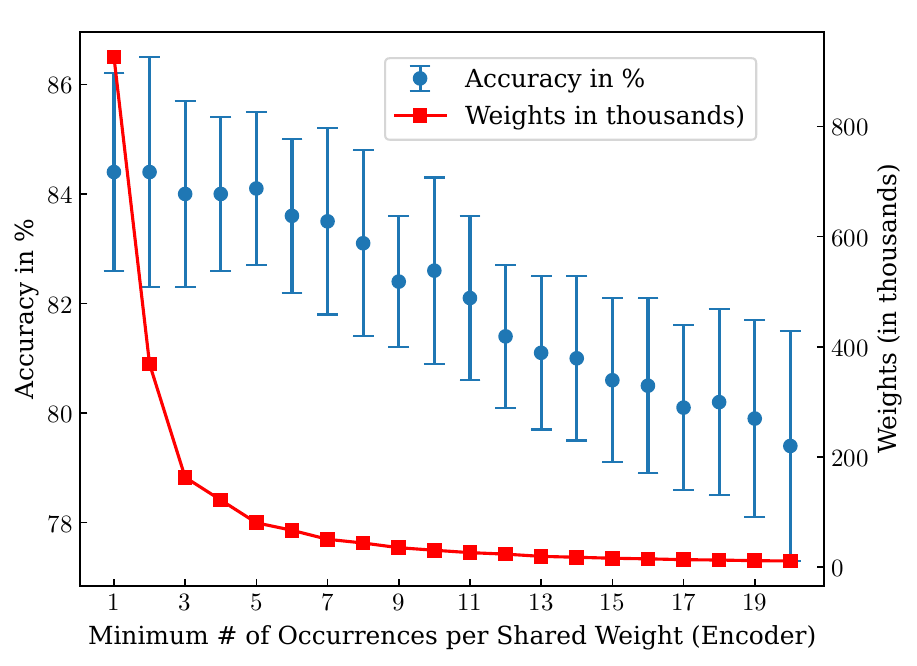}}
		\caption{NCI109\label{fig:ablation_threshold-NCI109-Lower}}
	\end{subfigure}
	\begin{subfigure}{0.325\textwidth}\centering
		{\includegraphics[width=\textwidth]{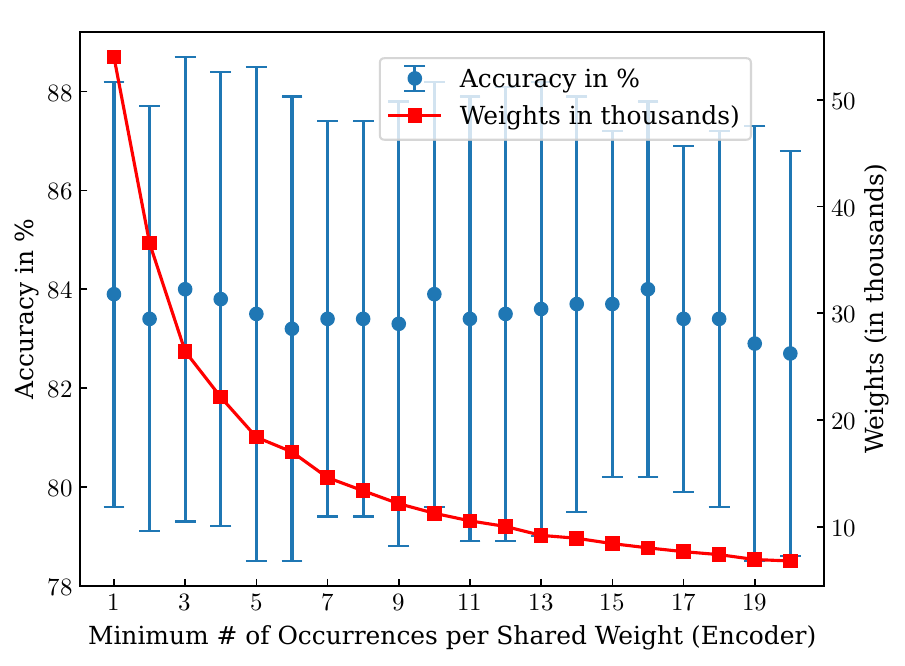}}
		\caption{DHFR\label{fig:ablation_threshold-DHFR-Lower}}
	\end{subfigure}
	\begin{subfigure}{0.3253\textwidth}\centering
		{\includegraphics[width=\textwidth]{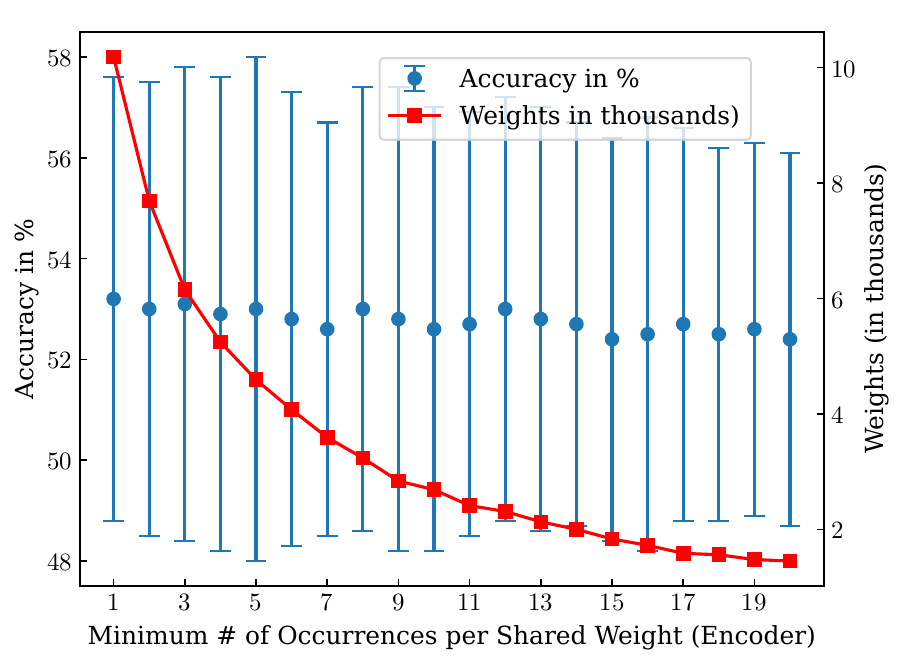}}
		\caption{IMDB-MULTI\label{fig:ablation_threshold-IMDB-MULTI-Lower}}
	\end{subfigure}
	\caption{\label{fig:ablation_threshold_lower} \MyGNN performance (mean accuracy and standard deviation) for different number of weights.
		The model corresponding to the $x$-axis value $i$ contains only those shared message passing weights that occur more than $i-1$-times in the respective graph dataset.
	}
\end{figure*}

\begin{figure*}[t]
		\begin{subfigure}{0.325\textwidth}\centering
		{\includegraphics[width=\textwidth]{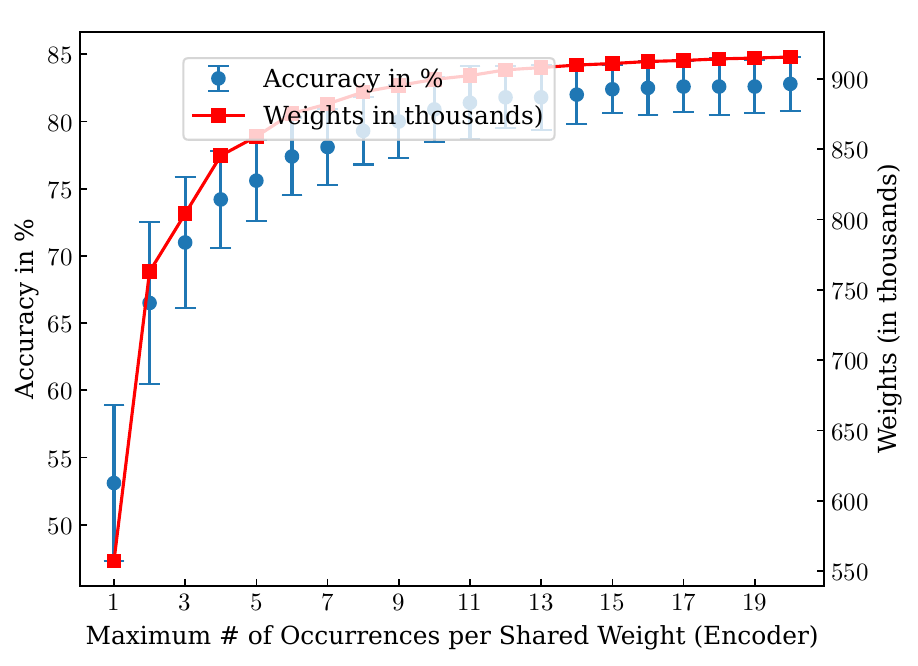}}
		\caption{NCI109\label{fig:ablation_threshold-NCI109-Upper}}
	\end{subfigure}
	\begin{subfigure}{0.325\textwidth}\centering
		{\includegraphics[width=\textwidth]{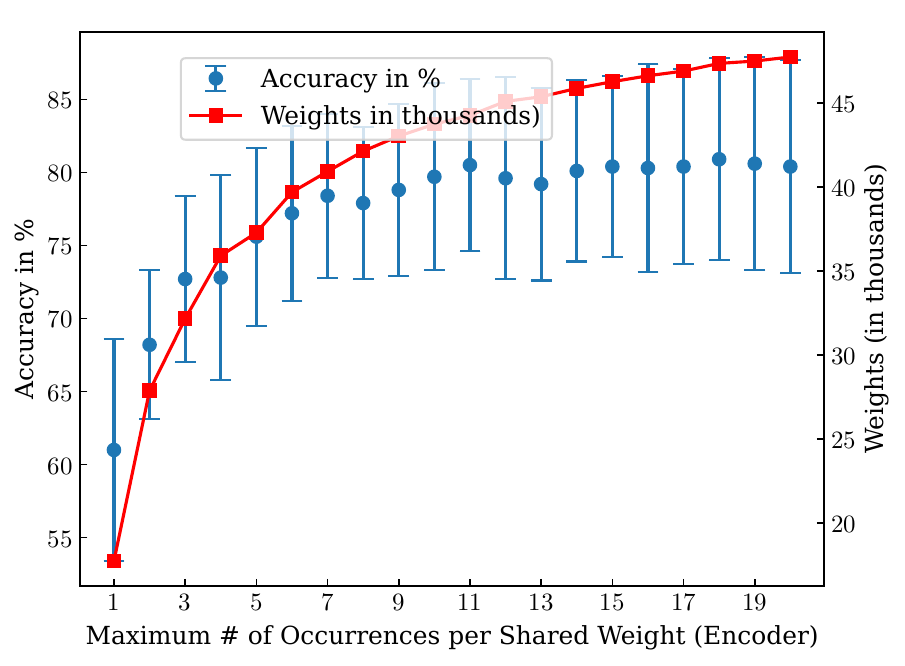}}
		\caption{DHFR\label{fig:ablation_threshold-DHFR-Upper}}
	\end{subfigure}
	\begin{subfigure}{0.325\textwidth}\centering
		{\includegraphics[width=\textwidth]{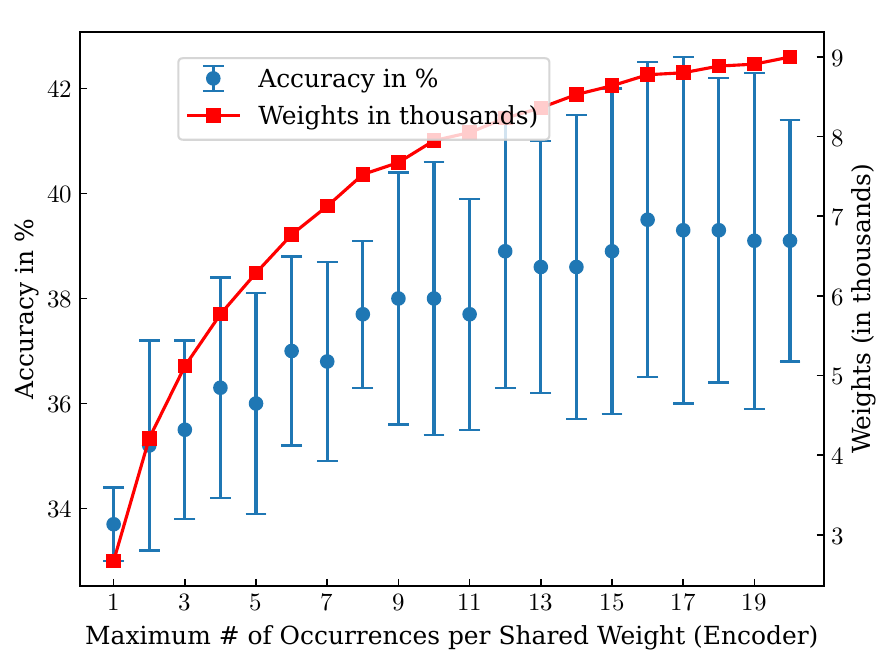}}
		\caption{IMDB-MULTI\label{fig:ablation_threshold-IMDB-MULTI-Upper}}
	\end{subfigure}
		\caption{\label{fig:ablation_threshold_upper}
		\MyGNN performance (mean accuracy and standard deviation) for different number of weights.
		The model corresponding to the $x$-axis value $i$ contains only those shared message passing weights that occur less than $i-1$-times in the respective graph dataset.
		}
		\end{figure*}

\begin{figure*}[t]
  \centering
  	\begin{subfigure}{0.29\textwidth}\centering
  		{\includegraphics[width=\textwidth]{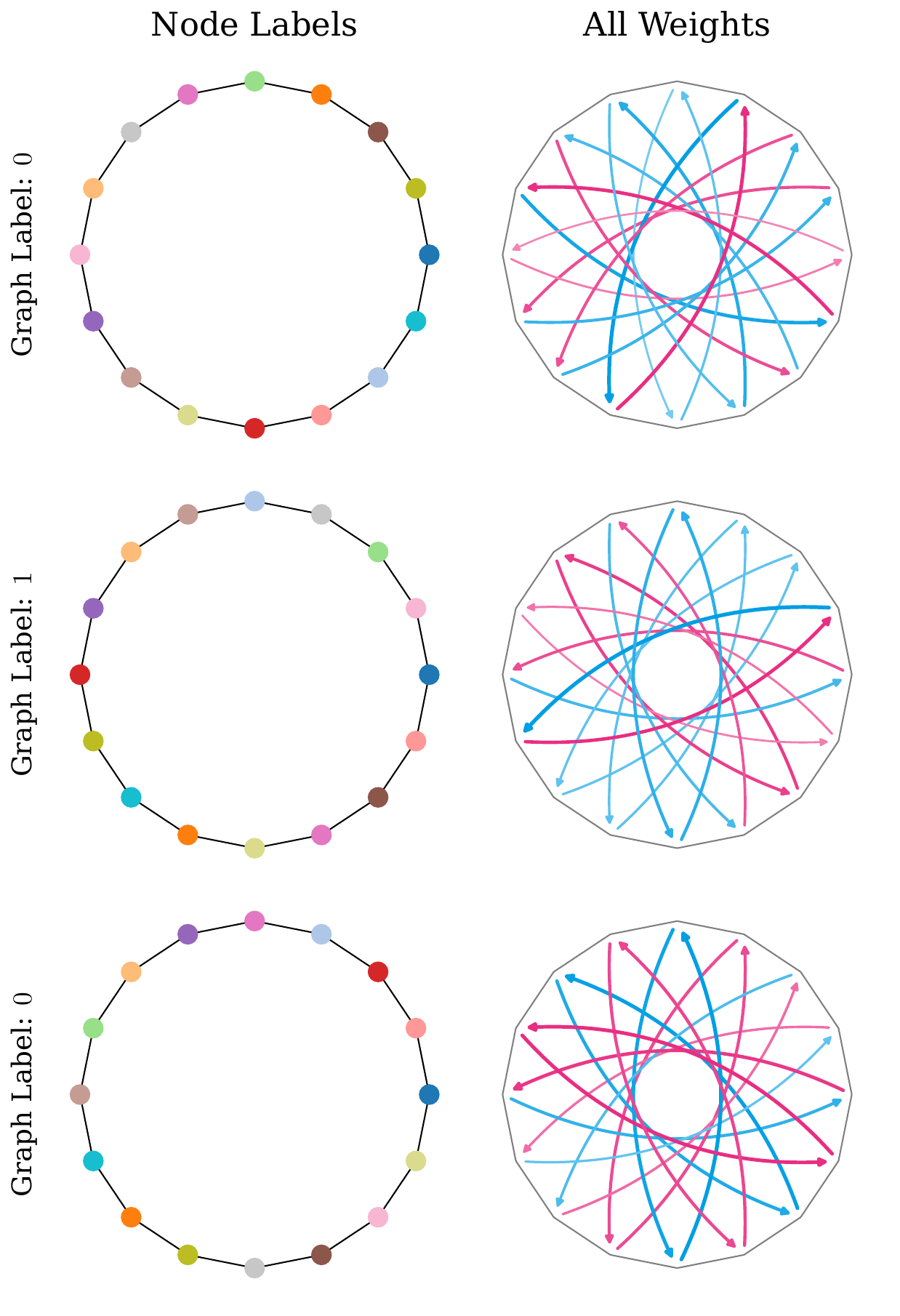}}
  		\caption{RingTransfer2\label{fig:ring-transfer-2-messages}}
		  	\end{subfigure}
  	\begin{subfigure}{0.69\textwidth}\centering
		  		{\includegraphics[width=\textwidth]{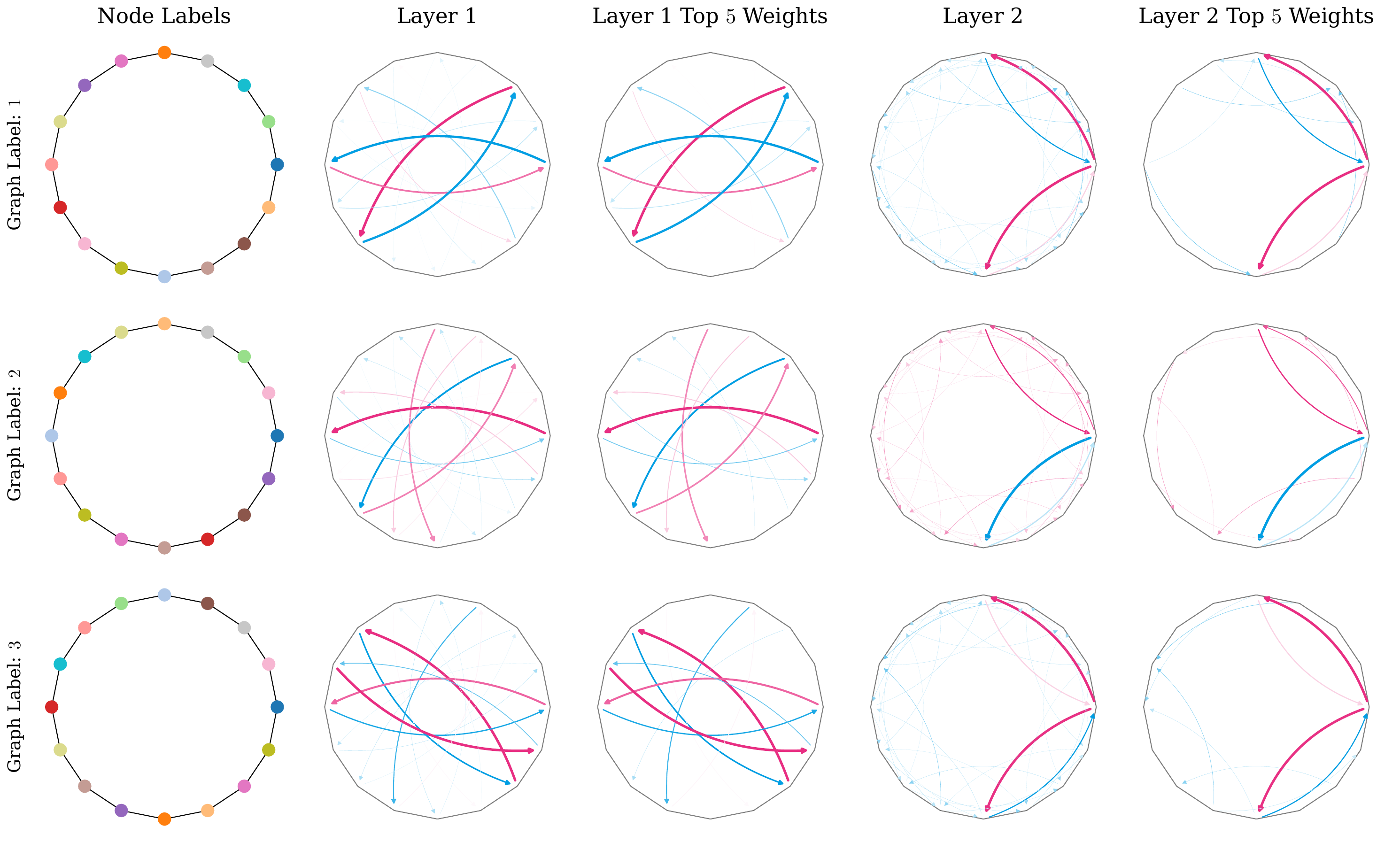}}
  		\caption{RingTransfer3\label{fig:ring-transfer-3-messages}}
\end{subfigure}
  \caption{\label{fig:interpretability-1}
  Visualization of the weights for the respective dataset.
	  The first column of each subfigure shows the graphs with the colors of the nodes representing the different node labels.
	  The other columns show all learned weights and the top five absolute weights, respectively, for the respective encoder layer.}
\end{figure*}

	\begin{figure*}[htbp]
  \centering
  	\begin{subfigure}{0.42\textwidth}\centering
  		{\includegraphics[width=\textwidth]{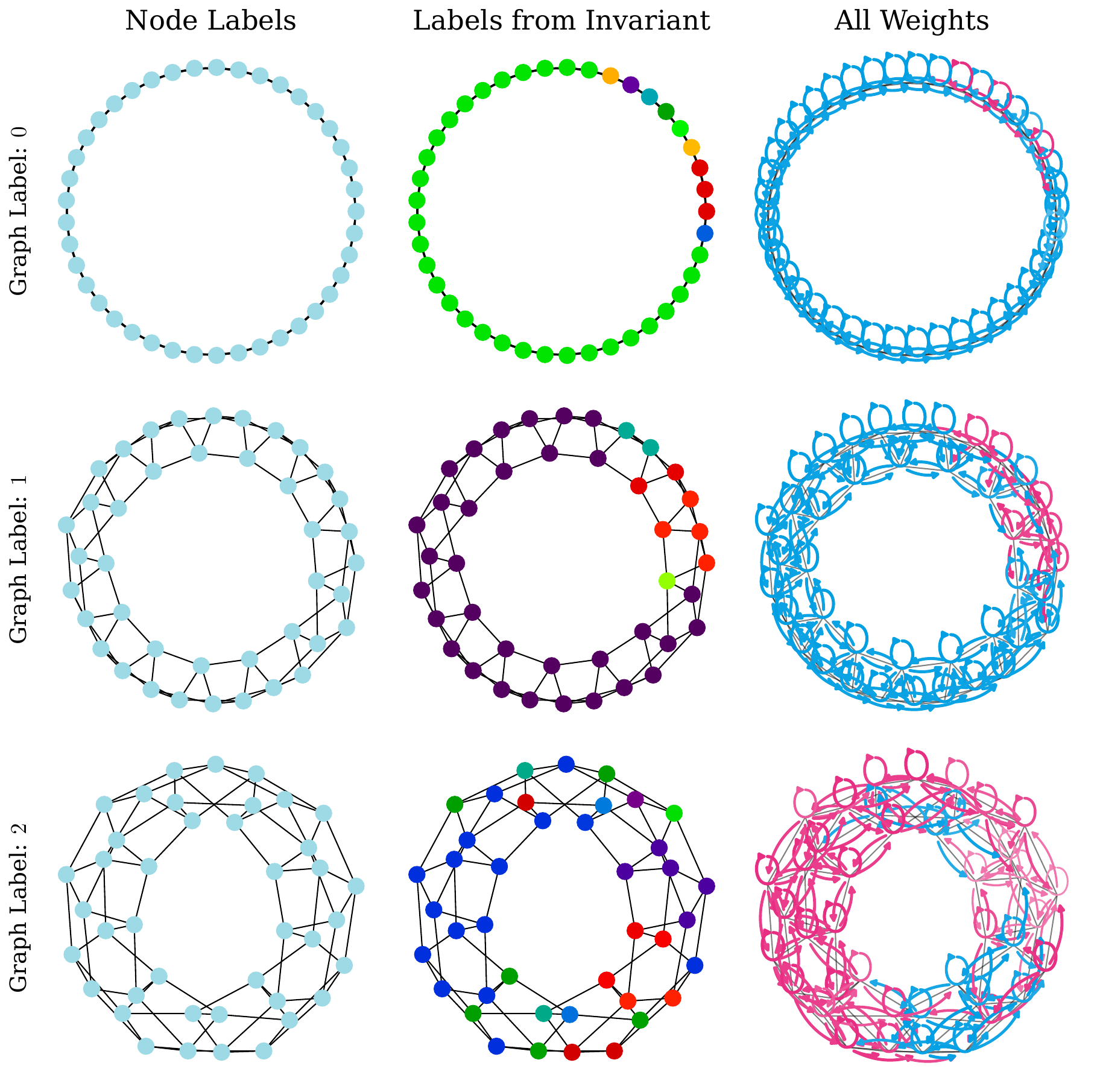}}
  		\caption{CSL\label{fig:csl-messages}}
		  	\end{subfigure}
  	\begin{subfigure}{0.56\textwidth}\centering
		  		{\includegraphics[width=\textwidth]{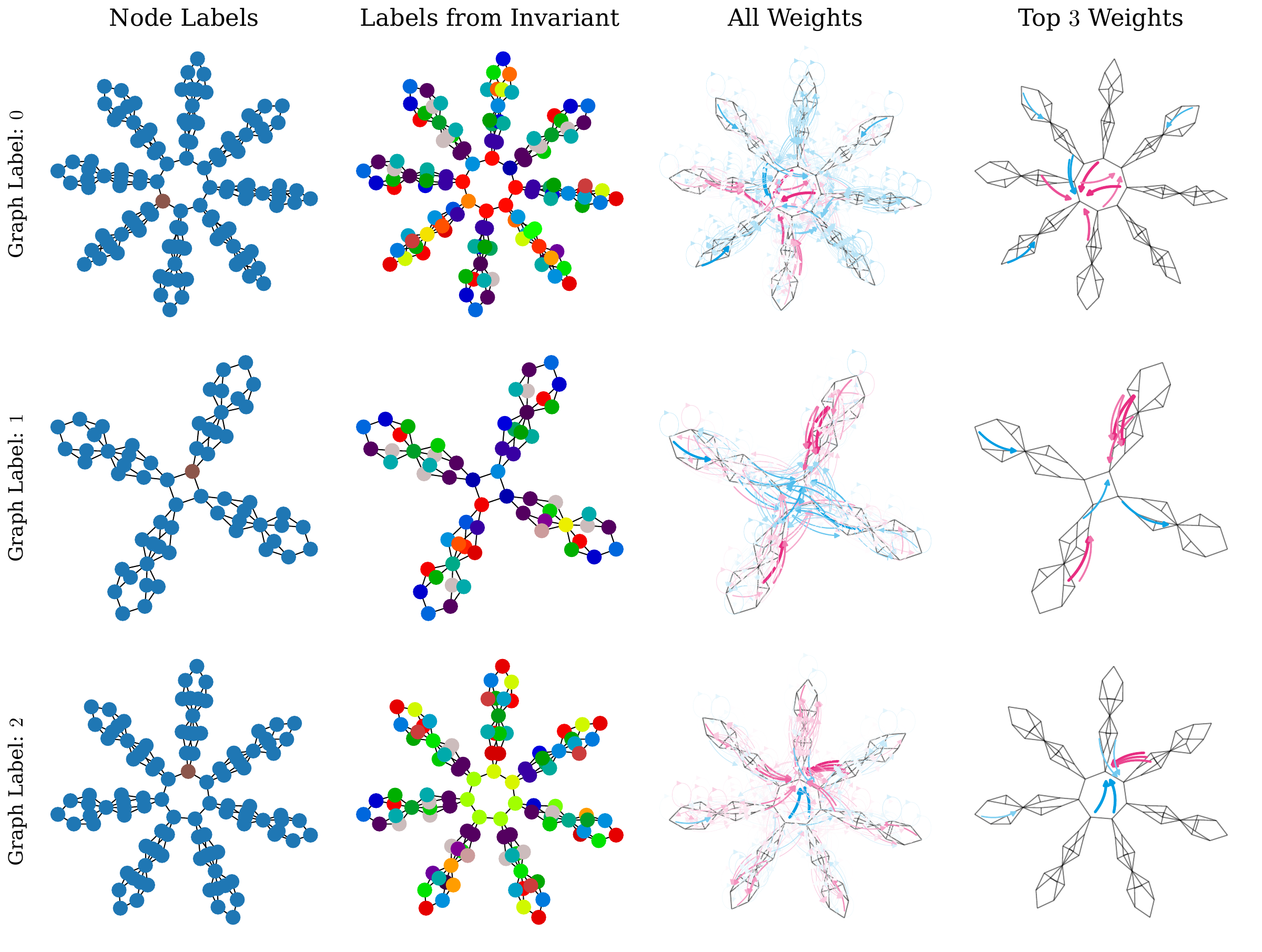}}
  		\caption{Snowflakes\label{fig:snowflakes-messages}}
\end{subfigure}

  \caption{\label{fig:interpretability-2}Visualization of the weights for the
 respective datasets.
  	  The first column of each subfigure shows the graphs with the colors of the nodes representing the different node labels.
  The second column shows the labels of the nodes with respect to the labeling function.
  The last two columns in (\ref{fig:snowflakes-messages}) show all and the top three absolute weights.
  }

\end{figure*}

\section{Extended Example: Edge Labels}\label{sec:example-molecule-graphs}
\begin{figure}[t]
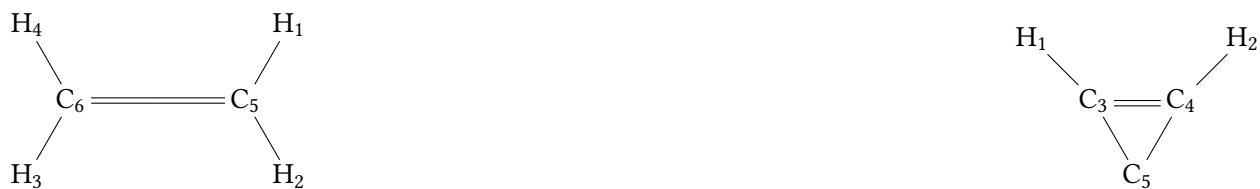
\centering
	\chemfig{{C_6}(==[:0]{C_5}(-[:60]{H_1})(-[:-60]{H_2}))(-[:-120]{H_3})(-[:120]{H_4})}
	\hfill
	\chemfig{[:-30]{C_3}(-[:135]{H_1})*3(-{C_5}-{C_4}(-[:45]{H_2})=)}
	\caption{\label{fig:MoleculeGraph}Molecular graphs of ethylene (left) and cyclopropenylidene (right). The numbers denote the order of the nodes.}
\end{figure}

	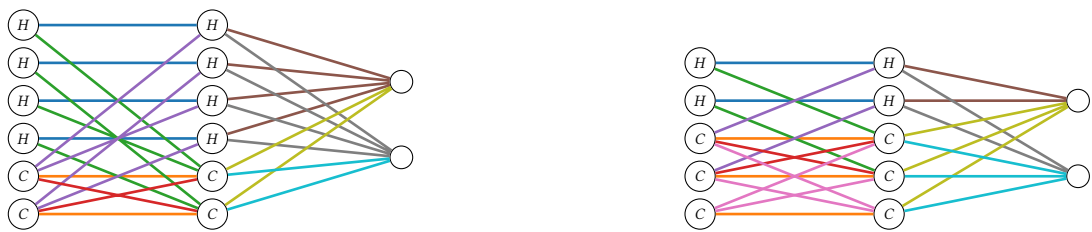
\begin{figure*}[t]
		\centering
				\scalebox{1}[1]{
		\begin{subfigure}{0.45\textwidth}\centering
	\begin{tikzpicture}[scale = 0.5]\small
	\node[circle,inner sep=0pt, minimum size=0.4cm, draw=black] (a) at (0,0) {\tiny $H$};
	\node[circle,inner sep=0pt, minimum size=0.4cm, draw=black] (b) at (0,-1) {\tiny $H$};
	\node[circle,inner sep=0pt, minimum size=0.4cm, draw=black] (c) at (0,-2) {\tiny $H$};
	\node[circle,inner sep=0pt, minimum size=0.4cm, draw=black] (d) at (0,-3) {\tiny $H$};
	\node[circle,inner sep=0pt, minimum size=0.4cm, draw=black] (e) at (0,-4) {\tiny $C$};
	\node[circle,inner sep=0pt, minimum size=0.4cm, draw=black] (f) at (0,-5) {\tiny $C$};

	\node[circle,inner sep=0pt, minimum size=0.4cm, draw=black] (g) at (5,0) {\tiny $H$};
	\node[circle,inner sep=0pt, minimum size=0.4cm, draw=black] (h) at (5,-1) {\tiny $H$};
	\node[circle,inner sep=0pt, minimum size=0.4cm, draw=black] (i) at (5,-2) {\tiny $H$};
	\node[circle,inner sep=0pt, minimum size=0.4cm, draw=black] (j) at (5,-3) {\tiny $H$};
	\node[circle,inner sep=0pt, minimum size=0.4cm, draw=black] (k) at (5,-4) {\tiny $C$};
	\node[circle,inner sep=0pt, minimum size=0.4cm, draw=black] (l) at (5,-5) {\tiny $C$};

	\draw[tab_blue, line width = 1pt] (a)--(g) (b)--(h) (c)--(i)  (d)--(j) ;
	\draw[tab_orange, line width = 1pt] (e)--(k) (f)--(l);
	\draw[tab_green, line width = 1pt] (a)--(k) (b)--(l) (c)--(k) (d)--(l);
	\draw[tab_purple, line width = 1pt]  (e)--(i) (e)--(g) (f)--(j) (f)--(h);
	\draw[tab_red, line width = 1pt]    (e)--(l) (f)--(k) ;

	\node[circle,inner sep=0pt, minimum size=0.3cm, draw=black] (y1) at (10,-1.5) {};
	\node[circle,inner sep=0pt, minimum size=0.3cm, draw=black] (y2) at (10,-3.5) {};

	\draw[tab_brown, line width = 1pt]  (y1)--(g) (y1)--(h) (y1)--(i) (y1)--(j);
	\draw[tab_gray, line width = 1pt]  (y2)--(g) (y2)--(h) (y2)--(i) (y2)--(j);
	\draw[tab_olive, line width = 1pt]    (y1)--(l) (y1)--(k);
	\draw[tab_cyan, line width = 1pt]    (y2)--(l) (y2)--(k);

	\end{tikzpicture}
	\end{subfigure}}\hfill
						\scalebox{1}[1]{
	\begin{subfigure}{0.45\textwidth}\centering
	\begin{tikzpicture}[scale = 0.5]\small
	\node[circle,inner sep=0pt, minimum size=0.4cm, draw=black] (a) at (0,0) {\tiny $H$};
	\node[circle,inner sep=0pt, minimum size=0.4cm, draw=black] (b) at (0,-1) {\tiny $H$};
	\node[circle,inner sep=0pt, minimum size=0.4cm, draw=black] (c) at (0,-2) {\tiny $C$};
	\node[circle,inner sep=0pt, minimum size=0.4cm, draw=black] (d) at (0,-3) {\tiny $C$};
	\node[circle,inner sep=0pt, minimum size=0.4cm, draw=black] (e) at (0,-4) {\tiny $C$};

	\node[circle,inner sep=0pt, minimum size=0.4cm, draw=black] (g) at (5,0) {\tiny $H$};
	\node[circle,inner sep=0pt, minimum size=0.4cm, draw=black] (h) at (5,-1) {\tiny $H$};
	\node[circle,inner sep=0pt, minimum size=0.4cm, draw=black] (i) at (5,-2) {\tiny $C$};
	\node[circle,inner sep=0pt, minimum size=0.4cm, draw=black] (j) at (5,-3) {\tiny $C$};
	\node[circle,inner sep=0pt, minimum size=0.4cm, draw=black] (k) at (5,-4) {\tiny $C$};

	\draw[tab_blue,solid, line width = 1pt] (a)--(g);
	\draw[tab_blue,solid, line width = 1pt] (b)--(h);
	\draw[tab_orange, line width = 1pt] (c)--(i);
	\draw[tab_orange, line width = 1pt] (d)--(j);
	\draw[tab_orange, line width = 1pt] (e)--(k);
	\draw[tab_green, line width = 1pt] (a)--(i);
	\draw[tab_green, line width = 1pt] (b)--(j);
	\draw[tab_purple, line width = 1pt]  (c)--(g);
	\draw[tab_purple, line width = 1pt] (d)--(h);
	\draw[tab_red, line width = 1pt] (c)--(j);
	\draw[tab_red, line width = 1pt] (d)--(i);

	\draw[tab_pink, line width = 1pt] (c)--(k);
	\draw[tab_pink, line width = 1pt] (d)--(k);
	\draw[tab_pink, line width = 1pt] (e)--(i);
	\draw[tab_pink, line width = 1pt] (e)--(j);

	\node[circle,inner sep=0pt, minimum size=0.3cm, draw=black] (y1) at (10,-1) {};
	\node[circle,inner sep=0pt, minimum size=0.3cm, draw=black] (y2) at (10,-3) {};

	\draw[tab_brown, line width = 1pt]  (y1)--(g) (y1)--(h);
	\draw[tab_gray, line width = 1pt]  (y2)--(g) (y2)--(h);
	\draw[tab_olive, line width = 1pt]   (y1)--(k) (y1)--(i) (y1)--(j);
	\draw[tab_cyan, line width = 1pt]   (y2)--(k) (y2)--(i) (y2)--(j);
	\end{tikzpicture}
\end{subfigure}}
\caption{\label{fig:MyGNNExample}
	Computational graphs of a simple~\MyGNN with one encoder layer for the molecular graphs of ethylene (left) and cyclopropenylidene (right).
	The input signal is propagated from left to right. The graph nodes represent the neurons of the neural network.
	Edges of the same color denote layer-wise shared weights.
}
\end{figure*}

In this example we show that also edge labels can be included into the~\MyGNN framework.

We consider the molecular graphs of ethylene and cyclopropenylidene, see Figure~\ref{fig:MoleculeGraph}.
The atoms of the molecules are denoted by $H$ for hydrogen  and $C$ for carbon.
The graph nodes correspond to the atoms and the graph edges to the atom bonds.
Ethylene is represented by the initial vector $x\in\mathbb{R}^6$  and cyclopropenylidene by $y\in\mathbb{R}^5$.
The bond types  (\textit{single} and \textit{double}) can be seen as edge labels.
Extending the considerations of our approach we show how to include edge labels into our model.
Instead of considering for each pair of nodes a triple consisting of the atom labels and the distance between the nodes
we consider triples consisting of the atom labels and the bond type between the atoms,
i.e., each pair of nodes $(v,w)$ is associated with a triple $(l(v), l(w), \text{bond}(v,w))$  where $l(v), l(w)$ is the atom label
and $\text{bond}(v,w)$ is either $\emptyset$ (no bond), $\odot$ (self-connection), $-$ (single bond) or $=$ (double bond).
Triples $\mathcal{T}$ are valid if they do not contain $\emptyset$ in the third entry, i.e., it they are not of the form $(\cdot,\cdot,\emptyset)$.
Moreover, if $\text{bond}(v,w)=\odot$, i.e., the triple corresponds to a self-connection, then the the first two entries have to be equal.
Finally, we assign each valid triple a learnable parameter from $\weightset_\mathcal{T}$.
For our example, we need the following six learnable parameters:

\begin{center}
\begin{tabular}{ccc}
	$\omega_1\coloneqq \omega_{(H, H, \odot)}$ & $\omega_2\coloneqq \omega_{(C, C, \odot)}$ & $\omega_3\coloneqq \omega_{(H, C, -)}$\\
	$\omega_4\coloneqq \omega_{(C, H, -)}$ & $\omega_5\coloneqq \omega_{(C, C, -)}$ & $\omega_6\coloneqq \omega_{(C, C, =)}$\\
\end{tabular}
\end{center}

Regarding the decoder,
we use the atomic numbers for $l$  and fix the output size to $m=2$.
The set of valid labels $\mathcal{L}$ are all possible atomic numbers.
Let $\weightset_{\labelset}$ be the set of learnable weight vectors for the decoder.
For our example we need the following learnable parameters $\omega_H, \omega_C\in\mathbb{R}^2$.

The order of the graph nodes is fixed as shown in~\Cref{fig:MoleculeGraph}.
Recall, that we need the fixed order to construct
the matrices for forward propagation matrices.
Note that in this example we omit all bias terms for simplicity.

Using the learnable parameters defined above
we get the following two matrices that define the forward propagation (message passing) of the encoder layer for the ethylene graph (left) and the
cyclopropenylidene graph (right).

\begin{center}
\begin{tabular}{cc}
	$W^{\operatorname{enc}}_x= \begin{psmallmatrix}
	\omega_1&0&0&0&\omega_3&0\\
	0&\omega_1&0&0&\omega_3&0\\
	0&0&\omega_1&0&0&\omega_3\\
	0&0&0&\omega_1&0&\omega_3\\
	\omega_4&\omega_4&0&0&\omega_2&\omega_5\\
	0&0&\omega_4&\omega_4&\omega_5&\omega_2\\
	\end{psmallmatrix}$
	&$W^{\operatorname{enc}}_y= \begin{psmallmatrix}
	\omega_1&0&\omega_3&0&0\\
	0&\omega_1&0&\omega_3&0\\
	\omega_4&0&\omega_2&\omega_6&\omega_5\\
	0&\omega_3&\omega_6&\omega_2&\omega_5\\
	0&0&\omega_5&\omega_5&\omega_2\\
	\end{psmallmatrix}
	$\\
\end{tabular}
\end{center}

We get the decoder weight matrices for the ethylene graph (left) and the cyclopropenylidene graph (right) as follows:

\begin{center}
\begin{tabular}{cc}
		$W^{\operatorname{dec}}_x= \begin{psmallmatrix}
	\omega_H&\omega_H&\omega_H&\omega_H&\omega_C&\omega_C\\
	\end{psmallmatrix}$
		&
	$W^{\operatorname{dec}}_y= \begin{psmallmatrix}
	\omega_H&\omega_C&\omega_C&\omega_H&\omega_H\\
	\end{psmallmatrix}$

\end{tabular}
\end{center}

Combining the encoder and decoder layers we obtain $$x' = \sigma^{\operatorname{dec}}(W^{\operatorname{dec}}_x\cdot \sigma^{\operatorname{enc}}(W^{\operatorname{enc}}_x \cdot x))$$
for the ethylene graph, and
$$y' = \sigma^{\operatorname{dec}}(W^{\operatorname{dec}}_y\cdot \sigma^{\operatorname{enc}}(W^{\operatorname{enc}}_y\cdot y))$$
for the cyclopropenylidene graph.

Note that the forward propagation of the encoder layer is essentially a multiplication with a weighted adjacency matrix of the graph, where the weights of the adjacency matrix are given by the learnable parameters (\Cref{fig:MyGNNExample}).
In contrast to adjacency matrices, the weight matrix is not necessary symmetric.
The computation graph induced by the weight matrices exactly represent the graph structure while the edge weights are shared across the network (\Cref{fig:MyGNNExample}).

This example uses edge labels only for message passing between adjacent nodes.
However, edge label information can also be included for node pairs that are not connected by an edge.
For example, by counting
the number of specific edge labels on all shortest paths between two nodes.

\section{Further Applications: Image and Text Classification}\label{sec:appendix-text-image}
	Our approach is not limited to the graph domain.
	Indeed, there are multiple options to extend it to other domains.
	We will give examples how image and text data can be represented as graphs and how our approach can be applied to them.

	\paragraph{Image Data} Each image can be seen as a grid graph with diagonal edges.
	Ordinary convolutional neural networks (CNNs) are based on fixed-size kernels that move along the image.
	With our novel approach applied to the above interpretation of images as grid graphs, message passing in
	our invariant-based encoder is nothing more than a convolutional kernel of arbitrary shape and size,
    see~\Cref{fig:example-convolution-image} for an illustration.

	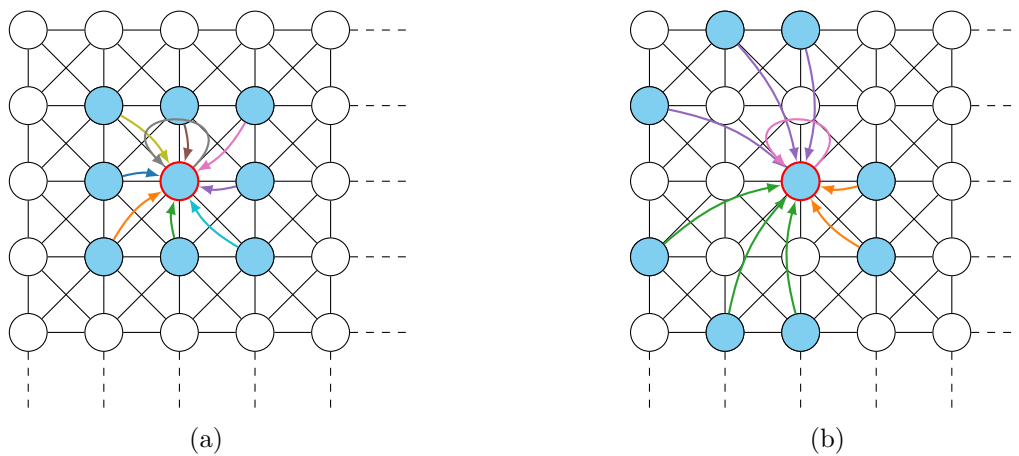
\begin{figure*}[t]
		\begin{subfigure}{0.49\textwidth}\centering
			\begin{tikzpicture}

				\foreach \x in {0,1,2,3,4} {
					\foreach \y in {0,1,2,3,4} {
						\node[draw,circle,minimum size=0.5cm] (n\x\y) at (\x,\y) {};
					}
				}

				\foreach \x in {0,1,2,3,4} {
					\foreach \y in {0,1,2,3,4} {
						\ifnum\x<4
						  \draw (n\x\y) -- (n\the\numexpr\x+1\relax\y);
						\fi
						\ifnum\y<4
						  \draw (n\x\y) -- (n\x\the\numexpr\y+1\relax);
						\fi
						\ifnum\x<4\ifnum\y<4
						  \draw (n\x\y) -- (n\the\numexpr\x+1\relax\the\numexpr\y+1\relax);
						\fi\fi
						\ifnum\x<4\ifnum\y>0
						  \draw (n\x\y) -- (n\the\numexpr\x+1\relax\the\numexpr\y-1\relax);
						\fi\fi
					}
				}
				\foreach \x in {0,1,2,3,4} {
					\draw[dashed] (n\x0) -- (\x,-1);
				}
				\foreach \y in {0,1,2,3,4} {
					\draw[dashed] (n4\y) -- (5,\y);
				}

				\node[thick, draw=red,circle,minimum size=0.5cm,fill=aqua!50] (r0) at (2,2) {};

				\node[draw,circle,minimum size=0.5cm,fill=aqua!50] (r1) at (1,2) {};
				\node[draw,circle,minimum size=0.5cm,fill=aqua!50] (r2) at (1,1) {};
				\node[draw,circle,minimum size=0.5cm,fill=aqua!50] (r3) at (2,1) {};
				\node[draw,circle,minimum size=0.5cm,fill=aqua!50] (r4) at (3,2) {};
				\node[draw,circle,minimum size=0.5cm,fill=aqua!50] (r5) at (2,3) {};
				\node[draw,circle,minimum size=0.5cm,fill=aqua!50] (r6) at (3,3) {};
				\node[draw,circle,minimum size=0.5cm,fill=aqua!50] (r7) at (1,3) {};
				\node[draw,circle,minimum size=0.5cm,fill=aqua!50] (r8) at (3,1) {};

				\draw[tab_blue,thick, -latex] (r1) to[bend left=15] (r0);
				\draw[tab_orange,thick, -latex] (r2) to[bend left=15] (r0);
				\draw[tab_green,thick, -latex] (r3) to[bend left=15] (r0);
				\draw[tab_purple,thick, -latex] (r4) to[bend left=15] (r0);
				\draw[tab_brown,thick, -latex] (r5) to[bend left=15] (r0);
				\draw[tab_pink,thick, -latex] (r6) to[bend left=15] (r0);
				\draw[tab_olive,thick, -latex] (r7) to[bend left=15] (r0);
				\draw[tab_cyan,thick, -latex] (r8) to[bend left=15] (r0);
				\draw[tab_gray,thick, -latex, loop] (r0) to (r0);
				\end{tikzpicture}
						\caption{\label{fig:example-convolution-regular}}%
		\end{subfigure}
		\begin{subfigure}{0.49\textwidth}\centering
			\begin{tikzpicture}
				\foreach \x in {0,1,2,3,4} {
					\foreach \y in {0,1,2,3,4} {
						\node[draw,circle,minimum size=0.5cm] (n\x\y) at (\x,\y) {};
					}
				}

				\foreach \x in {0,1,2,3,4} {
					\foreach \y in {0,1,2,3,4} {
						\ifnum\x<4
						  \draw (n\x\y) -- (n\the\numexpr\x+1\relax\y);
						\fi
						\ifnum\y<4
						  \draw (n\x\y) -- (n\x\the\numexpr\y+1\relax);
						\fi
						\ifnum\x<4\ifnum\y<4
						  \draw (n\x\y) -- (n\the\numexpr\x+1\relax\the\numexpr\y+1\relax);
						\fi\fi
						\ifnum\x<4\ifnum\y>0
						  \draw (n\x\y) -- (n\the\numexpr\x+1\relax\the\numexpr\y-1\relax);
						\fi\fi
					}
				}
				\foreach \x in {0,1,2,3,4} {
					\draw[dashed] (n\x0) -- (\x,-1);
				}
				\foreach \y in {0,1,2,3,4} {
					\draw[dashed] (n4\y) -- (5,\y);
				}

				\node[thick, draw=red,circle,minimum size=0.5cm,fill=aqua!50] (r0) at (2,2) {};
				\node[draw,circle,minimum size=0.5cm,fill=aqua!50] (r1) at (3,1) {};
				\node[draw,circle,minimum size=0.5cm,fill=aqua!50] (r5) at (3,2) {};
				\node[draw,circle,minimum size=0.5cm,fill=aqua!50] (r2) at (2,0) {};
				\node[draw,circle,minimum size=0.5cm,fill=aqua!50] (r3) at (1,0) {};
				\node[draw,circle,minimum size=0.5cm,fill=aqua!50] (r4) at (0,1) {};

				\node[draw,circle,minimum size=0.5cm,fill=aqua!50] (r6) at (0,3) {};
				\node[draw,circle,minimum size=0.5cm,fill=aqua!50] (r7) at (1,4) {};
				\node[draw,circle,minimum size=0.5cm,fill=aqua!50] (r8) at (2,4) {};

				\draw[tab_orange,thick, -latex] (r1) to[bend left=15] (r0);
				\draw[tab_green,thick, -latex] (r2) to[bend left=15] (r0);
				\draw[tab_green,thick, -latex] (r3) to[bend left=15] (r0);
				\draw[tab_green,thick, -latex] (r4) to[bend left=15] (r0);
				\draw[tab_orange,thick, -latex] (r5) to[bend left=15] (r0);
				\draw[tab_purple,thick, -latex] (r6) to[bend left=15] (r0);
				\draw[tab_purple,thick, -latex] (r7) to[bend left=15] (r0);
				\draw[tab_purple,thick, -latex] (r8) to[bend left=15] (r0);
				\draw[tab_pink,thick, -latex, loop] (r0) to (r0);
				\end{tikzpicture}
			\caption{\label{fig:example-convolution-irregular}}%
		\end{subfigure}
		\caption{\label{fig:example-convolution-image} Representation of a 5x5 partial image as a grid graph with diagonal edges.
		The left subfigure shows the ordinary convolution for images with a regular kernel (3x3 square shape).
		The right subfigure shows an example of an irregular kernel (G-shape)
		that can be represented by our invariant-based encoder layer.
		The node with the red border corresponds to the center of the kernel, the light blue ones denote the receptive field.
		Different edge colors denote different (message-passing) weights.
		}
	\end{figure*}

	\paragraph{Text Data} Each text can be seen as a path graph where each node corresponds to a word in the text, see~\Cref{fig:overview,fig:text-example} for an illustration.
	Such representations of text are also called~\textit{local word consecutive}~\citep{DBLP:journals/air/WangDH24}.
    The label of a node in this case can be the id of the corresponding word given by a tokenizer.
	This allows to apply our approach directly to text classification by using the path graph representation.
	Learning the weights of the invariant-based encoder is nothing more than learning
	relations between certain words at a certain distance in the text.
	The application of our approach to text classification can be seen as a generalization of the approach of~\cite{DBLP:conf/emnlp/HuangMLZW19}.

	Notably, interpretability and transferability of our approach also applies to these domains.
	Specifically, we can re-use the learned weights between two words in a text to learn relations between two words in another text.
	
	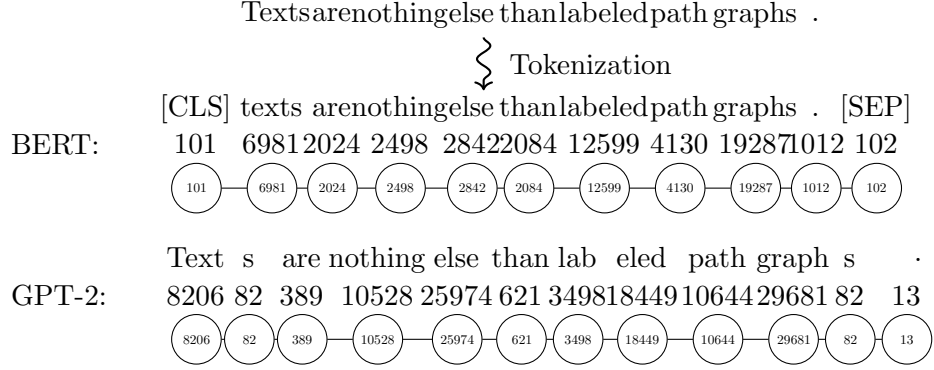
\begin{figure*}[t]
		\centering
		\begin{tikzpicture}[on grid,graph node/.style={draw,fill=white,circle,text width={width(20000)},align=center,font=\small, scale=0.5}]
			\node (t1) {Texts\phantom{q}};
			\node[xshift=0.8cm] (t2) {are\phantom{qj}};
			\node[xshift=1.7cm] (t3) {nothing\phantom{qj}};
			\node[xshift=2.65cm] (t4) {else\phantom{qj}};
			\node[xshift=3.4cm] (t5) {than\phantom{qj}};
			\node[xshift=4.4cm] (t6) {labeled\phantom{qj}};
			\node[xshift=5.4cm] (t7) {path\phantom{qj}};
			\node[xshift=6.4cm] (t8) {graphs\phantom{qj}};
			\node[xshift=7.2cm] (t9) {.\phantom{qj}};
			
			\draw[->, snake it, thick] (t4.south) -- ++(0,-0.7) node[midway, right, xshift=0.2cm] {Tokenization};

			\node[below of=t1, yshift=-0.25cm] (token1) {texts\phantom{qj}};
			\node[below of=t2, yshift=-0.25cm] (token2) {are\phantom{qj}};
			\node[below of=t3, yshift=-0.25cm] (token3) {nothing\phantom{qj}};
			\node[below of=t4, yshift=-0.25cm] (token4) {else\phantom{qj}};
			\node[below of=t5, yshift=-0.25cm] (token5) {than\phantom{qj}};
			\node[below of=t6, yshift=-0.25cm] (token6) {labeled\phantom{qj}};
			\node[below of=t7, yshift=-0.25cm] (token7) {path\phantom{qj}};
			\node[below of=t8, yshift=-0.25cm] (token8) {graphs\phantom{qj}};
			\node[below of=t9, yshift=-0.25cm] (token9) {.\phantom{qj}};
			
			\node[left of=token1] (cls) {[CLS]\phantom{qj}};
			\node[right of=token9, xshift=-0.2cm] (sep) {[SEP]\phantom{qj}};

			\node[below of=cls, yshift=0.5cm] (id0) {101\phantom{qj}};
			\node[below of=token1, yshift=0.5cm] (id1) {6981\phantom{qj}};
			\node[below of=token2, yshift=0.5cm] (id2) {2024\phantom{qj}};
			\node[below of=token3, yshift=0.5cm] (id3) {2498\phantom{qj}};
			\node[below of=token4, yshift=0.5cm] (id4) {2842\phantom{qj}};
			\node[below of=token5, yshift=0.5cm] (id5) {2084\phantom{qj}};
			\node[below of=token6, yshift=0.5cm] (id6) {12599\phantom{qj}};
			\node[below of=token7, yshift=0.5cm] (id7) {4130\phantom{qj}};
			\node[below of=token8, yshift=0.5cm] (id8) {19287\phantom{qj}};
			\node[below of=token9, yshift=0.5cm] (id9) {1012\phantom{qj}};
			\node[below of=sep, yshift=0.5cm] (id10) {102\phantom{qj}};
			
			\node[left of=id0, xshift=-0.75cm] (bert) {BERT:\phantom{qj[]}};
			
			\foreach \i\j in {0/101, 1/6981, 2/2024, 3/2498, 4/2842, 5/2084, 6/12599, 7/4130, 8/19287, 9/1012, 10/102} {
				\node [graph node, below of=id\i, xshift=-0.3cm, yshift=-0.1cm] (graphnode\i) {\j};
			}
			\foreach \i/\j in {0/1, 1/2, 2/3, 3/4, 4/5, 5/6, 6/7, 7/8, 8/9, 9/10} {
				\draw (graphnode\i) -- (graphnode\j);
			}
			
			\node[below of=id1, yshift=-0.5cm, xshift=-1cm] (token1gpt) {Text\phantom{qj}};
			\node[below of=id1, yshift=-0.5cm, xshift=-0.3cm] (token2gpt) {s\phantom{qj}};
			\node[below of=id2, yshift=-0.5cm, xshift=-0.4cm] (token3gpt) {are\phantom{qj}};
			\node[below of=id3, yshift=-0.5cm, xshift=-0.3cm] (token4gpt) {nothing\phantom{qj}};
			\node[below of=id4, yshift=-0.5cm, xshift=-0.2cm] (token5gpt) {else\phantom{qj}};
			\node[below of=id5, yshift=-0.5cm, xshift=-0.1cm] (token6gpt) {than\phantom{qj}};
			\node[below of=id6, yshift=-0.5cm, xshift=-0.35cm] (token7gpt) {lab\phantom{qj}};
			\node[below of=id6, yshift=-0.5cm, xshift=0.5cm] (token8gpt) {eled\phantom{qj}};
			\node[below of=id7, yshift=-0.5cm, xshift=0.5cm] (token9gpt) {path\phantom{qj}};
			\node[below of=id8, yshift=-0.5cm, xshift=0.5cm] (token10gpt) {graph\phantom{qj}};
			\node[below of=id8, yshift=-0.5cm, xshift=1.25cm] (token11gpt) {s\phantom{qj}};
			\node[below of=id9, yshift=-0.5cm, xshift=1.2cm] (token12gpt) {.};
			
			\node[below of=token1gpt, yshift=0.5cm] (id1gpt) {8206\phantom{qj}};
			\node[below of=token2gpt, yshift=0.5cm] (id2gpt) {82\phantom{qj}};
			\node[below of=token3gpt, yshift=0.5cm] (id3gpt) {389\phantom{qj}};
			\node[below of=token4gpt, yshift=0.5cm] (id4gpt) {10528\phantom{qj}};
			\node[below of=token5gpt, yshift=0.5cm] (id5gpt) {25974\phantom{qj}};
			\node[below of=token6gpt, yshift=0.5cm] (id6gpt) {621\phantom{qj}};
			\node[below of=token7gpt, yshift=0.5cm] (id7gpt) {3498\phantom{qj}};
			\node[below of=token8gpt, yshift=0.5cm] (id8gpt) {18449\phantom{qj}};
			\node[below of=token9gpt, yshift=0.5cm] (id9gpt) {10644\phantom{qj}};
			\node[below of=token10gpt, yshift=0.5cm] (id10gpt) {29681\phantom{qj}};
			\node[below of=token11gpt, yshift=0.5cm] (id11gpt) {82\phantom{qj}};
			\node[below of=token12gpt, yshift=0.5cm] (id12gpt) {13\phantom{qj}};
			
			\node[left of=id1gpt, xshift=-0.7cm] (gpt) {GPT-2:\phantom{qj[]}};
			
			\foreach \i\j in {1/8206, 2/82, 3/389, 4/10528, 5/25974, 6/621, 7/3498, 8/18449, 9/10644, 10/29681, 11/82, 12/13} {
				\node [graph node, below of=id\i gpt, xshift=-0.3cm, yshift=-0.1cm] (graphnodegpt\i) {\j};
			}
			\foreach \i/\j in {1/2, 2/3, 3/4, 4/5, 5/6, 6/7, 7/8, 8/9, 9/10, 10/11, 11/12} {
				\draw (graphnodegpt\i) -- (graphnodegpt\j);
			}

		\end{tikzpicture}
		\caption{Representation of a sample text as labeled path graph.
			First, the text is parsed through a tokenizer, that assigns each token a unique identifier.
			Second, each token is mapped to a node in a path graph.
			That is, the edges of the graph connect consecutive tokens, forming a path.
		}
		\label{fig:overview}
	\end{figure*}
	
	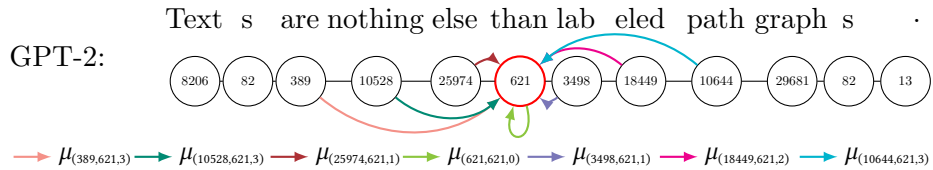
\begin{figure*}[t]\centering
		\begin{tikzpicture}[on grid,graph node/.style={draw,fill=white,circle,text width={width(20000)},align=center,font=\small, scale=0.5}]
			\node[below of=id1, yshift=-0.5cm, xshift=-1cm] (token1gpt) {Text\phantom{qj}};
			\node[below of=id1, yshift=-0.5cm, xshift=-0.3cm] (token2gpt) {s\phantom{qj}};
			\node[below of=id2, yshift=-0.5cm, xshift=-0.4cm] (token3gpt) {are\phantom{qj}};
			\node[below of=id3, yshift=-0.5cm, xshift=-0.3cm] (token4gpt) {nothing\phantom{qj}};
			\node[below of=id4, yshift=-0.5cm, xshift=-0.2cm] (token5gpt) {else\phantom{qj}};
			\node[below of=id5, yshift=-0.5cm, xshift=-0.1cm] (token6gpt) {than\phantom{qj}};
			\node[below of=id6, yshift=-0.5cm, xshift=-0.35cm] (token7gpt) {lab\phantom{qj}};
			\node[below of=id6, yshift=-0.5cm, xshift=0.5cm] (token8gpt) {eled\phantom{qj}};
			\node[below of=id7, yshift=-0.5cm, xshift=0.5cm] (token9gpt) {path\phantom{qj}};
			\node[below of=id8, yshift=-0.5cm, xshift=0.5cm] (token10gpt) {graph\phantom{qj}};
			\node[below of=id8, yshift=-0.5cm, xshift=1.25cm] (token11gpt) {s\phantom{qj}};
			\node[below of=id9, yshift=-0.5cm, xshift=1.2cm] (token12gpt) {.};
			
			\node[left of=id1gpt, xshift=-0.7cm] (gpt) {GPT-2:\phantom{qj[]}};
			
			\foreach \i\j in {1/8206, 2/82, 3/389, 4/10528, 5/25974, 6/621, 7/3498, 8/18449, 9/10644, 10/29681, 11/82, 12/13} {
				\ifthenelse{\equal{\i}{6}}{
					\node [graph node, below of=token\i gpt, xshift=-0.3cm, yshift=-0.6cm, draw=red, thick] (graphnodegpt\i) {\j};
				}{
					\node [graph node, below of=token\i gpt, xshift=-0.3cm, yshift=-0.6cm] (graphnodegpt\i) {\j};
				}
			}
			\foreach \i/\j in {1/2, 2/3, 3/4, 4/5, 5/6, 6/7, 7/8, 8/9, 9/10, 10/11, 11/12} {
				\draw (graphnodegpt\i) -- (graphnodegpt\j);
			}
			
			\foreach \i\color in {3/Salmon, 4/PineGreen, 8/Magenta, 9/Turquoise} {
				\draw[thick, -latex, bend right=40,draw=\color] (graphnodegpt\i) to (graphnodegpt6);
				
			}

			\foreach \i\color in {5/Maroon, 7/Periwinkle} {
				\draw[draw=\color, thick, -latex, bend left=40] (graphnodegpt\i) to (graphnodegpt6);
				
			}
			\draw[draw=LimeGreen,loop below, thick, -latex] (graphnodegpt6) to (graphnodegpt6);

			\draw[draw=Salmon, thick, -latex] (graphnodegpt1)++ (-2.4,-1) --++ (0.5,0) node[right, xshift=-0.05cm] {\small{$\mu_{\scriptscriptstyle{(389, 621, 3)}}$}};
			\draw[draw=PineGreen, thick, -latex] (graphnodegpt1)++ (-0.8,-1) --++ (0.5,0) node[right, xshift=-0.05cm] {\small{$\mu_{\scriptscriptstyle{(10528, 621, 3)}}$}};
			\draw[draw=Maroon, thick, -latex] (graphnodegpt1)++ (1,-1) --++ (0.5,0) node[right, xshift=-0.05cm] {\small{$\mu_{\scriptscriptstyle{(25974, 621, 1)}}$}};
			\draw[draw=LimeGreen, thick, -latex] (graphnodegpt1)++ (2.75,-1) --++ (0.5,0) node[right, xshift=-0.05cm] {\small{$\mu_{\scriptscriptstyle{(621, 621, 0)}}$}};
			\draw[draw=Periwinkle, thick, -latex] (graphnodegpt1)++ (4.4,-1) --++ (0.5,0) node[right, xshift=-0.05cm] {\small{$\mu_{\scriptscriptstyle{(3498, 621, 1)}}$}};
			\draw[draw=Magenta, thick, -latex] (graphnodegpt1)++ (6.15,-1) --++ (0.5,0) node[right, xshift=-0.05cm] {\small{$\mu_{\scriptscriptstyle{(18449, 621, 2)}}$}};
			\draw[draw=Turquoise, thick, -latex] (graphnodegpt1)++ (8,-1) --++ (0.5,0) node[right, xshift=-0.05cm] {\small{$\mu_{\scriptscriptstyle{(10644, 621, 3)}}$}};

		\end{tikzpicture}
		\caption{\label{fig:text-example}Encoder Layer: Message passing in the labeled path graph generated by the GPT-2 tokenizer.
			The arrows indicate different message passing weights $\mu$ from nodes up to distance 3 to the node corresponding to the token \textit{than} (621).
			Different colors indicate different weights.
		}
	\end{figure*}

\end{document}